\documentclass[letterpaper]{article} % DO NOT CHANGE THIS
\usepackage{aaai25}  % DO NOT CHANGE THIS
\usepackage{times}  % DO NOT CHANGE THIS
\usepackage{helvet}  % DO NOT CHANGE THIS
\usepackage{courier}  % DO NOT CHANGE THIS
\usepackage[hyphens]{url}  % DO NOT CHANGE THIS
\usepackage{graphicx} % DO NOT CHANGE THIS
\usepackage{natbib}  % DO NOT CHANGE THIS AND DO NOT ADD ANY OPTIONS TO IT
\usepackage{caption} % DO NOT CHANGE THIS AND DO NOT ADD ANY OPTIONS TO IT
\usepackage{algorithm}
\usepackage{newfloat}
\usepackage{listings}

\DeclareCaptionStyle{ruled}{labelfont=normalfont,labelsep=colon,strut=off} % DO NOT CHANGE THIS
\floatstyle{ruled}
\newfloat{listing}{tb}{lst}{}
\floatname{listing}{Listing}
\usepackage[utf8]{inputenc} % allow utf-8 input
\usepackage[T1]{fontenc}    % use 8-bit T1 fonts
\usepackage{url}            % simple URL typesetting
\usepackage{booktabs}       % professional-quality tables
\usepackage{amsfonts}       % blackboard math symbols
\usepackage{nicefrac}       % compact symbols for 1/2, etc.
\usepackage{microtype}      % microtypography
\usepackage{xcolor}         % colors

\usepackage{amsmath}         % colors
\usepackage{amsthm}
\usepackage{amssymb}         % colors
\usepackage{comment}         % colors

\usepackage{algorithm}
\usepackage{algpseudocode}
\usepackage{makecell}
\usepackage{graphicx}
\usepackage{subcaption}
\usepackage{mathtools}

\newtheorem{theorem}{Theorem}[]
\newtheorem{lemma}{Lemma}
\newtheorem{proposition}{Proposition}

\newtheorem{definition}{Definition}

\newtheorem{claim}{Claim}
\newtheorem{assumption}{Assumption}

\title{Unsupervised Anomaly Detection for Tabular Data\\ Using Noise Evaluation}
\author{
    Wei Dai, ~Kai Hwang, ~Jicong Fan\thanks{Corresponding author.}\\
}
\affiliations{
    \textsuperscript{}The Chinese University of Hong Kong, Shenzhen, China\\
    weidai@link.cuhk.edu.com, hwangkai@cuhk.edu.cn, fanjicong@cuhk.edu.cn
}

\begin{document}

\maketitle

\begin{abstract}
 Unsupervised anomaly detection (UAD) plays an important role in modern data analytics and it is crucial to provide simple yet effective and guaranteed UAD algorithms for real applications. In this paper, we present a novel UAD method for tabular data by evaluating how much noise is in the data. Specifically, we propose to learn a deep neural network from the clean (normal) training dataset and a noisy dataset, where the latter is generated by adding highly diverse noises to the clean data. The neural network can learn a reliable decision boundary between normal data and anomalous data when the diversity of the generated noisy data is sufficiently high so that the hard abnormal samples lie in the noisy region. Importantly, we provide theoretical guarantees, proving that the proposed method can detect anomalous data successfully, although the method does not utilize any real anomalous data in the training stage.
Extensive experiments through more than 60 benchmark datasets demonstrate the effectiveness of the proposed method in comparison to 12 baselines of UAD. Our method obtains a 92.27\% AUC score and a 1.68 ranking score on average. Moreover, compared to the state-of-the-art UAD methods, our method is easier to implement.
\end{abstract}

\section{Introduction}
In the realm of data analysis, anomaly detection (AD) stands as a pivotal challenge with far-reaching implications across various domains, including cybersecurity \citep{siddiqui2019detecting, saeed2023anomaly}, healthcare \citep{yang2023deep, abououf2023explainable}, finance \citep{hilal2022financial}, and industrial processes \citep{fan2021kernel, roth2022towards}. 
Existing deep learning-based unsupervised AD methods often rely on an auxiliary learning objective such as auto-encoder, generative model, and contrastive learning. These methods indirectly detect anomalous data using other metrics such as reconstruction error, which lack generalizability and reliability guarantees \citep{hussain2023reliable}. Explicitly learning a one-class decision boundary may resolve this issue. Many well-known unsupervised AD methods assume the normal training data has a special structure in their data space or embedding space \citep{scholkopf2001estimating,tax2004support,ruff2018deep, goyal2020drocc,zhang2024deep,XIAO2025121435}. Such assumptions may not hold or be guaranteed in practice and sometimes place a burden on the model training \citep{cai2022perturbation,fu2024dense}. For instance, in deep SVDD \citep{ruff2018deep}, the optimal decision boundary in the embedding space may be very different from the learned hypersphere, leading to unsatisfactory detection performance \citep{zhang2024deep}.

Given that tabular data is probably the most common data type and other types of data such as images can be converted to tabular data using feature encoders or pretrained models, in this work, we focus on the tabular data only. We propose a novel unsupervised AD method for tabular data without making any assumption about the distribution of normal data. Since hard anomalies are often close to normal data, it is reasonable to hypothesize that hard anomalies are special cases of perturbed samples of normal data. Therefore, if the diversity of the perturbations or added noises on normal data is sufficiently high, we can obtain hard anomalies. Consequently, if a model can recognize highly diverse perturbations or noises, it can detect hard anomalies as well as easy anomalies successfully. By directly learning from the diverse noise patterns and the clean data patterns, we can learn an effective decision boundary around the normal data, generalizing well to unseen data.

Our contributions are highlighted as follows.
\begin{itemize}
    \item We propose a novel AD method for tabular data using noise evaluation. Our scheme generates highly diverse noise-augmented instances for the normal samples. By evaluating the noise magnitude, our method can accurately identify anomalies.
    \item The proposed method provides a simple yet effective scheme that does not make assumptions about the normal training data. In addition, the noise generation is straightforward without any extra training. Compared with \citep{wang2021auto, goyal2020drocc, yan2021learning, cai2022perturbation}, our method is more lightweight for training (requires less module for training.)
    % has fewer hyper-parameters.
    \item We theoretically prove the generalizability and reliability of the proposed method.
    \item We conduct extensive empirical experiments on 47 real datasets in an unsupervised anomaly detection setting and 25 real-world tabular datasets in a one-class classification setting to demonstrate the performance of the proposed schemes. The results show that our method achieved superior performance compared with 12 baseline methods including the state-of-the-art.
    % \item We provide a 
\end{itemize}

\section{Related Work}

Unsupervised anomaly detection (UAD) is also known as the one-class classification (OCC) problem, in which all or most training samples are assumed to be normal. The learning objective is to learn a decision boundary that distinguishes whether a sample belongs to the same distribution of the normal training data or not. There is another similar problem setting known as outlier detection on contaminated dataset \cite{huang2024entropystop}. The goal is to detect noised samples or outliers within the training data \cite{ding2022hyperparameter}. This line of work is orthogonal to our UAD problem.

In the past decades, many UAD methods have been studied \cite{liu2008isolation, chang2023data}. Traditional methods like proximity-based \citep{breunig2000lof, angiulli2002fast, papadimitriou2003loci, he2003discovering}, probabilty-based \citep{yang2009outlier, zong2018deep, li2020copod}, and one-class support vector machine \citep{scholkopf2001estimating, tax2004support} approaches struggle with high dimensionality and complex data structures. Deep neural network-based methods have been proposed to address these issues. For instance, auto-encoder methods identify outliers by detecting high reconstruction errors, as outlier samples do not conform to historical data patterns \citep{aggarwal2016outlier, chen2017outlier, wang2021auto}. Generative model methods compare latent features or generated samples to spot anomalies \citep{schlegl2017unsupervised, liu2019generative, zhang2023unsupervised, tur2023exploring,XIAO2025121435}. For example, \cite{XIAO2025121435} proposed an inverse generative adversarial network that converts the data distribution to a compact Gaussian distribution, based on which the density of test data can be calculated for anomaly detection. Contrastive learning \citep{sohn2020learning, jin2021anemone, shenkar2022anomaly} leverages feature representation differences to detect anomalies. Unlike autoencoder-based methods that focus on reducing reconstruction error and reducing dimensionality to remove noise, our method aims to evaluate noise level, which is similar to the denoising diffusion model \cite{ho2020denoising}.

Some works explicitly build an anomaly detection or OCC objective \citep{ruff2018deep, goyal2020drocc, yan2021learning, chen2022deep, cai2022perturbation}. For instance, Deep SVDD \citep{ruff2018deep} trains a neural network to construct a hypersphere in the output space to enclose the normal training data. DROCC \citep{goyal2020drocc} assumes normal samples lie on low-dimensional manifolds and treats identifying the ideal hypersphere as an adversarial optimization problem. PLAD \citep{cai2022perturbation} outputs an anomaly score by learning a small perturbation of normal data as the negative sample with a classifier. Unlike PLAD, which uses extra additive and multiplicative perturbations requiring a perturbator, our method generates negative samples without extra parameters, making the training efficient.

It is worth mentioning that there are vision UAD methods utilizing synthetic anomalous data. However, the normality and abnormality can be visualized directly and there naturally exists prior knowledge about anomalies (e.g., visible spots or blurs) in visual data. Regarding tabular data, we do not have such prior knowledge. Vision AD methods require prior pre-trained models or external reference datasets to obtain the negative samples, such as DREAM \cite{zavrtanik2021draem} with the Describable Textures Dataset or AnomalyDiffusion \cite{hu2024anomalydiffusion} with a stable diffusion model. Such resources are costly and violate the principle of unsupervised learning. Tabular data, however, spans diverse domains (e.g., medical, industrial), making external datasets and pretrained model infeasible.

\section{Proposed Method}

\subsection{Problem Formulation and Notations}
Given a data set $\mathcal{X}=\{\boldsymbol{x}_1, \boldsymbol{x}_2, \ldots, \boldsymbol{x}_N\}$, where the element $\boldsymbol{x}_i$ are drawn from an unknown distribution $\mathcal{D}$ in $\mathbb{R}^d$ (deemed as a normal data distribution). In the context of UAD, the primary objective is to develop a function $f: \mathbb{R}^d \rightarrow \{0, 1\}$, which effectively discriminates between in-distribution (normal) and out-of-distribution (anomalous) instances. This discriminative function is formulated to assign a binary label, where $ f(x) = 1 $ indicates $\boldsymbol{x}$ does not belong to $\mathcal{D}$ and $f(x) = 0$ corresponds to $\boldsymbol{x}$ coming from $\mathcal{D}$.
The main notations used in this paper are shown in Table \ref{tab_notation}.

\begin{figure}
% {r}{0.5\textwidth}
% \vspace{-120pt}
        \centering
        \includegraphics[width=0.75\linewidth]{ 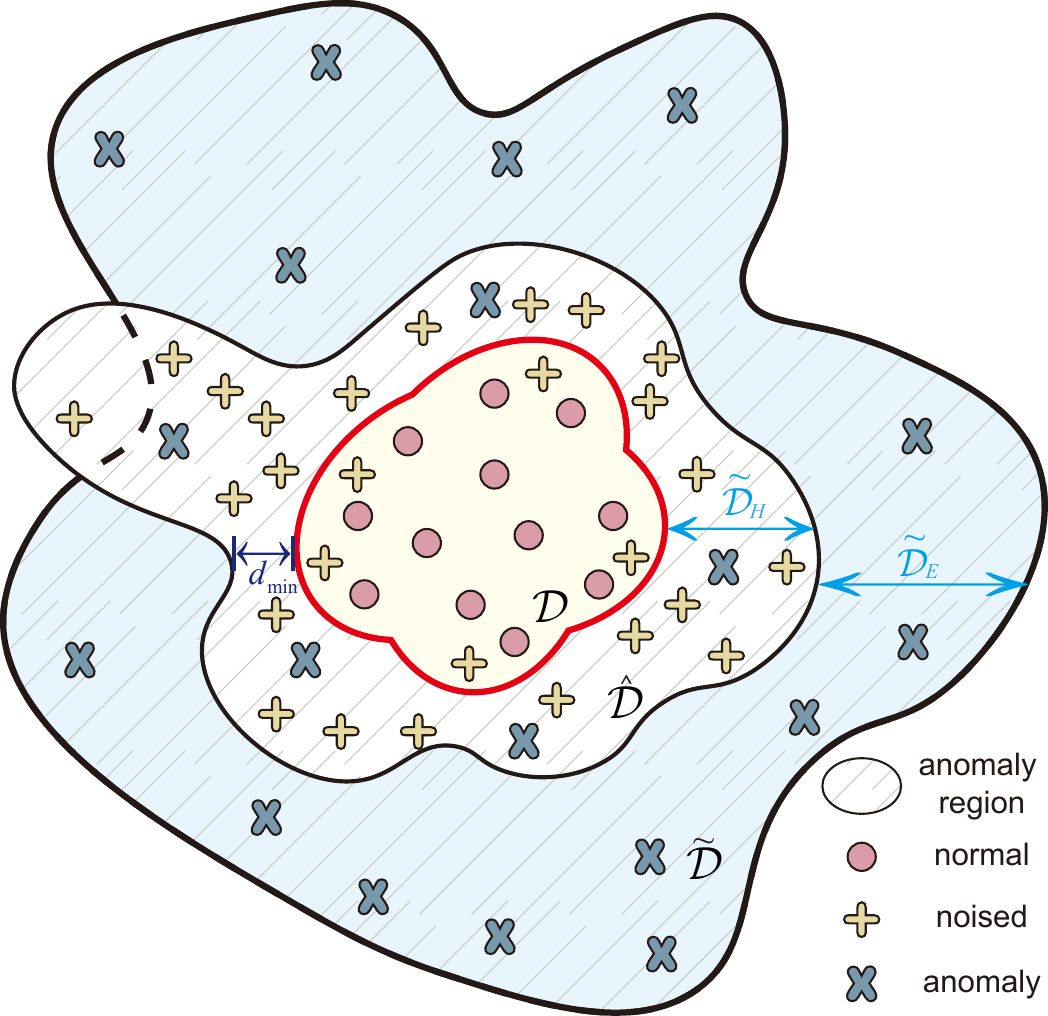}
        \caption{An illustration of the allocation of normal, noised, and true anomalous samples. $\mathcal{D}$, $\hat{\mathcal{D}}$, and $\tilde{\mathcal{D}}$ are the normal, noised, and anomalous distributions respectively. $\tilde{\mathcal{D}}$ is composed of a hard part $\tilde{\mathcal{D}}_H$ and an easy part $\tilde{\mathcal{D}}_E$. Theorem \ref{theorem: gap} and Theorem \ref{theorem_dmin} are for $\tilde{\mathcal{D}}_H$ and $\tilde{\mathcal{D}}_E$ respectively.}\label{fig: egg}
\end{figure}

 \begin{table}
 % {r}{0.5\textwidth}
 % \vspace{-5pt}
    \centering
    \setlength\tabcolsep{0.9pt}
\resizebox{0.8\linewidth}{!}{
    \begin{tabular}{c|c|c|c} \hline
    Symbol & Description & Symbol & Description \\ \hline
        $x$ & a scalar & $\boldsymbol{x}_i$ & a vector with index $i$ \\
        $\mathcal{X}_i$ & a set with index $i$ & $[i]$ & the set $\{1,2,\ldots,i\}$\\ 
        $\mathcal{D}$ & a distribution & $\|\cdot\|_2$ & $\ell_2$ norm of vector \\
        $\|\cdot\|_1$& $\ell_1$ norm of vector & $|\boldsymbol{x}|$ & \small{element-wise absolute} \\
        $\cup$ & union of sets & $\subset$ & subset \\
        $H$ & entropy & $h_{\boldsymbol{\theta}}$ & \small{model parameterized by $\boldsymbol{\theta}$} \\ \hline
    \end{tabular}}
    \caption{Notations}
    \label{tab_notation}
    % \vspace{-10pt}
\end{table}
\subsection{Anomalous Data Decomposition} \label{sec: method}

Let $\tilde{\mathcal{X}}$ be the set consisting of all anomalous data drawn from some unknown distribution $\tilde{\mathcal{D}}$ deemed as an anomalous distribution.
For any $\tilde{\boldsymbol{x}}\in \tilde{\mathcal{X}}$, we decompose it as 
\begin{equation}
    \tilde{\boldsymbol{x}}=\boldsymbol{x}+\boldsymbol{\epsilon},
\end{equation}
where $\boldsymbol{x}\in\mathcal{D}$ is a normal counterpart of $\tilde{\boldsymbol{x}}$ and $\boldsymbol{\epsilon}$ denotes the derivation of $\tilde{\boldsymbol{x}}$ from ${\boldsymbol{x}}$. The magnitude of $\boldsymbol{\epsilon}$, denoted as $\|\boldsymbol{\epsilon}\|_1$, measures how anomalous $\tilde{\boldsymbol{x}}$ is. Note that this decomposition is not unique and hence one may seek the one with the smallest $\|\boldsymbol{\epsilon}\|_1$. If we can learn a model $h$ to predict $\boldsymbol{\epsilon}$ for $\tilde{\boldsymbol{x}}$, i.e.,
\begin{equation}
    \boldsymbol{\epsilon}=h(\tilde{\boldsymbol{x}}),
\end{equation}
we will be able to determine whether $\tilde{\boldsymbol{x}}$ is normal or not according to $\|\boldsymbol{\epsilon}\|_1$. The challenge is that there is no available information about $\tilde{\mathcal{X}}$ in the training stage and we can only utilize $\mathcal{X}$.

Although $\tilde{\mathcal{X}}$ is unknown, we further theoretically partition $\tilde{\mathcal{X}}$ into two subsets without overlap, i.e.,
\begin{equation}
    \tilde{\mathcal{X}}\triangleq\tilde{\mathcal{X}}_{E}\cup \tilde{\mathcal{X}}_{H},~ \tilde{\mathcal{X}}_{E}\cap \tilde{\mathcal{X}}_{H}=\emptyset, {\tilde{\mathcal{X}}_{E}\sim\tilde{\mathcal{D}}_E,\tilde{\mathcal{X}}_{H}\sim\tilde{\mathcal{D}}_H}.
\end{equation}
$\tilde{\mathcal{X}}_E$ denotes an easy set, drawn from the easy part $\tilde{\mathcal{D}}_E$ of $\tilde{\mathcal{D}}$,
in which $\|\boldsymbol{\epsilon}\|_1$ for each sample is sufficiently large, 
while $\tilde{\mathcal{X}}_{H}$ denotes a hard set, drawn from the hard part $\tilde{\mathcal{D}}_H$ of $\tilde{\mathcal{D}}$,
in which $\|\boldsymbol{\epsilon}\|_1$ for each sample is small. 
% \weiadd{where $\|\boldsymbol{\epsilon}\|_1$ }
After the partition, we can assert that $\tilde{\mathcal{X}}_{H}$ is closer to the normal data. Consequently, it is easier for a model to recognize samples in $\tilde{\mathcal{X}}_{E}$ than those in $\tilde{\mathcal{X}}_{H}$, as the $\|\boldsymbol{\epsilon}\|_1$ values of samples in $\tilde{\mathcal{D}}_E$ are significantly larger than those in $\tilde{\mathcal{D}}_H$. Figure \ref{fig: egg} provides an intuitive example. The ultimate goal is to learn the decision boundary around the normal data.

Here, we focus on how to detect the samples in $\tilde{\mathcal{X}}_{H}$ or drawn from $\tilde{\mathcal{D}}_{H}$. Since $\tilde{\mathcal{X}}_{H}$ is very close to the normal data, it is reasonable to hypothesize that hard anomalies are special cases of perturbed samples of normal data.
We propose to generate a noisy dataset $\hat{\mathcal{X}}\subset\mathbb{R}^d$ from $\mathcal{X}$ by adding various noise to $\mathcal{X}$, and assume $\hat{\mathcal{X}}$ is drawn from certain perturbed distribution $\hat{{\mathcal{D}}}$, i.e.,
\begin{equation}\label{eq_M_noise}
    \hat{\mathcal{X}}\leftarrow \text{Gen}({\mathcal{X}})\sim \hat{{\mathcal{D}}},
\end{equation}
where $\text{Gen}$ denotes the noisy data generator and $|\hat{\mathcal{X}}|\gg N$.  
Let the diversity of added noise is sufficiently large, such that
% one can ensure that
\begin{equation}\label{eq_XHX}
    \tilde{\mathcal{X}}_{H}\subset \hat{\mathcal{X}},
\end{equation}
i.e., anomaly patterns of the hard set $\tilde{\mathcal{X}}_{H}$ are included in $\hat{\mathcal{X}}$. Even if \eqref{eq_XHX} does not hold, as shown by Theorem \ref{theorem: gap}, it is still possible to obtain correct detection, provided that $\hat{{\mathcal{D}}}$ is not too far from $\tilde{\mathcal{D}}_H$, i.e.,
\begin{equation}
    \text{dist}(\hat{{\mathcal{D}}},\tilde{\mathcal{D}}_H)\leq \gamma,
\end{equation}
where $\text{dist}(\cdot,\cdot)$ is some distance or divergence measure between two distributions and $\gamma$ is not too large. Therefore, a model $h$ learned from $\hat{\mathcal{X}}$ is able to generalize to $\tilde{\mathcal{X}}_{H}$ and then detect anomaly.

\subsection{Noise Evaluation Model}
To generate $\hat{\mathcal{X}}$, we add random noise to the elements of each sample $\boldsymbol{x}\in{\mathcal{X}}$. This operation will reduce the quality of the data, that is, the quality of $\hat{\mathcal{X}}$ is lower than that of ${\mathcal{X}}$, supported by
\begin{proposition}\label{prop_1}
    Adding random noises independently to the entries of $\mathcal{X}$ makes the data more disordered (higher entropy).
\end{proposition}
\begin{figure}
% {r}{0.5\textwidth}
        \vspace{-0pt}
        \includegraphics[width=\linewidth]{ 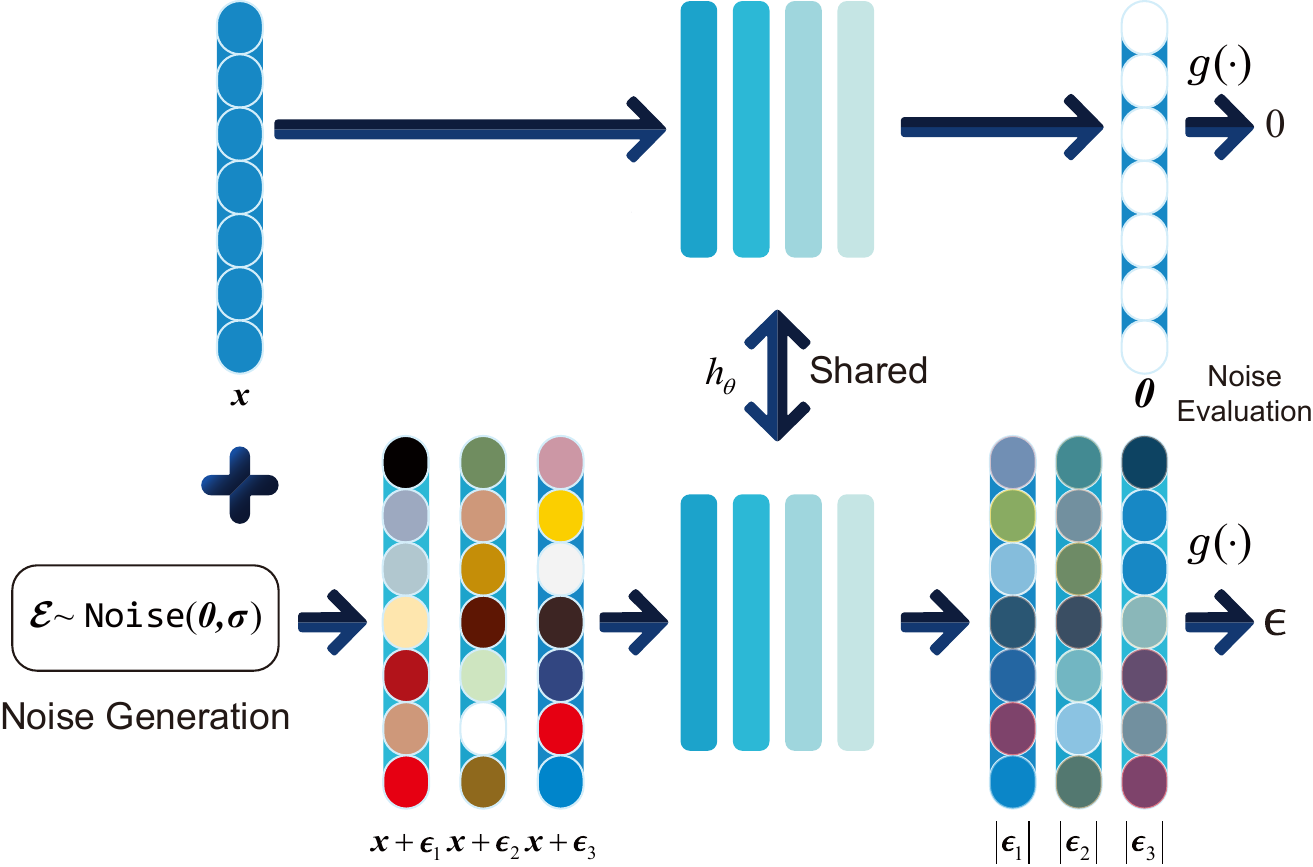}
        \caption{
        The training process of noise evaluation model. Noise with 0 mean and $\sigma$ standard deviation is added to the original data $\boldsymbol{x}$ to create noised versions $\hat{\boldsymbol{x}}= \boldsymbol{x}+\boldsymbol{\epsilon}$. The model $h_{\boldsymbol{\theta}}$ is trained to discern the zero vector for the original data and identify the noise vector $|\boldsymbol{\epsilon}|$ for the noised data. The final anomaly decision is made using an aggregation function $g(\cdot)$, where high-magnitude noise indicates abnormality.
        % The process involves adding noise with 0 mean and $\sigma$ standard deviation to the original data $\boldsymbol{x}$ to create one or multiple noised versions $\hat{\boldsymbol{x}}= \boldsymbol{x}+\boldsymbol{\epsilon}$. The model, $h_{\boldsymbol{\theta}}$, is then trained to discern the zero vector for the original input data and to identify the specific noise vector $|\boldsymbol{\epsilon}|$ for the noised data. The final anomaly decision is decided using an aggregation function $g(\cdot)$ If there is noise with high magnitude, the sample is abnormal.
        }\label{fig: overview}
        \vspace{-10pt}
\end{figure}
We would like to learn a deep neural network $h(\cdot): \mathbb{R}^d \rightarrow \mathbb{R}^d$ parameterized by $\boldsymbol{\theta}$ to quantify the quality of the input data. Its output is a vector with the same size as the input while each entry tells whether the input feature is noised or not. We generate the noised dataset by 
\begin{equation}\label{eq_xxxx}
    \hat{\mathcal{X}}=\hat{\mathcal{X}}_1\cup \hat{\mathcal{X}}_2\cdots\cup\hat{\mathcal{X}}_K,
\end{equation}
where each $\hat{\mathcal{X}}_k$ is composed of the samples generated by
\begin{equation}\label{eq_xxe}
  \hat{\boldsymbol{x}}=\boldsymbol{x}+\boldsymbol{\epsilon}, ~~\boldsymbol{\epsilon}\sim{\text{Noise}_{k}}(\boldsymbol{0}, \boldsymbol{\sigma}_{k}),~~\boldsymbol{x}\in\mathcal{X}.
\end{equation}
In \eqref{eq_xxe}, $\text{Noise}_k(\boldsymbol{0}, \boldsymbol{\sigma}_{k})$ (to be detailed in the next section) is a multivariate noise distribution with $\boldsymbol{0}$ mean value and $\boldsymbol{\sigma}_{k}$ standard deviation in $\mathbb{R}^d$, abbreviated as $\text{Noise}(\boldsymbol{\sigma}_{k})$. For instance, $\text{Noise}_k$ can be a Gaussian distribution. We also denote the standard deviation of the noise as \textit{noise level} \footnote{We use noise level and standard deviation of noise inter-changeably in this paper.}.
According to \eqref{eq_xxxx} and \eqref{eq_xxe}, we see that $\hat{\mathcal{X}}$ can contain different types of noise distributions with different standard deviations, and the diversity is controlled by $\sum_{k=1}^K{\|\sigma_k\|}$. Refer to Table \ref{tab: noise type} for the choices of the noise generator. Then, we write the learning objective as
\begin{align} \label{eq: final objective}
    \min_{\boldsymbol{\theta}} \sum_{\boldsymbol{x}_i\in\mathcal{X}}\left \| h_{\boldsymbol{\theta}}(\boldsymbol{x}_i) - \boldsymbol{0} \right \|_{2}^2 + \sum_{\hat{\boldsymbol{x}}_i\in\hat{\mathcal{X}}}\left \| h_{\boldsymbol{\theta}}(\hat{\boldsymbol{x}}_i) - \left|\boldsymbol{\epsilon}_i\right| \right \|_{2}^2,
\end{align}
where $\boldsymbol{0}$ is a $d$-dimension vector with all zero values, $\boldsymbol{\epsilon}$ is drawn from some $\text{Noise}(0, \boldsymbol{\sigma})$, and $\hat{\boldsymbol{x}}=\boldsymbol{x} + \boldsymbol{\epsilon}$ for some $\boldsymbol{x}\in\mathcal{X}$. Note that in \eqref{eq: final objective}, there is no hyperparameter to determine except the network structure, while many previous methods such as \citep{ruff2018deep, goyal2020drocc,cai2022perturbation,zhang2024deep} have at least one more crucial hyperparameter to tune in the learning objective.

In \eqref{eq: final objective}, the absolute value is used to let the model output a positive signal indicating an anomaly. Here, we only require the model to predict how much the noise is instead of the exact noise value, which also reduces the difficulty of the learning process. 
In addition, our method is closely related to the denoising score matching \citep{song2019generative, vincent2011connection}, which estimates the gradient of the probability density of the data distribution. We prove \eqref{eq: final objective} is a lower bound of denoising score matching learning objective in Appendix A. Hence, our method also learns data distribution $\mathcal{D}$ by proxy. 
Predicting the noise magnitude has the following advantages.

\begin{claim}
    % Rather than merely determining if the data conforms to the distribution $\mathcal{D}$, e
    Estimating the element-wise magnitude of noise enables the model to quantify the deviation of a sample from the normal data distribution $\mathcal{D}$. It improves the ability to identify specific features with abnormalities.
\end{claim}
Since the model output is a $d$-dimension vector, we introduce an aggregation operation after the training, $g(\cdot)$, to map the output vector to a scalar for anomaly detection. The operation in our design can be maximum, minimum, mean, median, or a combination. The final anomaly score is determined by 
\begin{equation}\label{eq_score}
   \text{score}(\boldsymbol{x}):=g\left(h_{\boldsymbol{\theta}}(\boldsymbol{x})\right). 
\end{equation}
Taking the maximum as an example, we have $\text{score}(\boldsymbol{x})=\max_{i\in[d]}[h_{\boldsymbol{\theta}}(\boldsymbol{x})]_i$.
We need to determine a threshold $\tau>0$ for the anomaly score. If $\text{score}(\boldsymbol{x})>\tau$, $\boldsymbol{x}$ is anomalous. An overview of our model is shown in Figure \ref{fig: overview}. 
% Hence, the final discriminative function for the one-class classification task is $f(\boldsymbol{x}) = g\left(h_{\boldsymbol{\theta}}(\boldsymbol{x})\right)$.\fnote{Need to revise} 

\subsection{Noise Generation and Model Training}\label{sec:noise_gen}
The noise generation in our design is not arbitrary. To cover the sampling space of the noised distributions as much as possible, we randomly generate the noise for each training sample in terms of noise level and position. An intuitive example is as follows:  
% \text{Noise}(0, \sigma_1), \text{Noise}(0, \sigma_1), ..., \text{Noise}(0, \sigma_2), ..., \text{Noise}(0, \sigma_m)
\begin{equation}
    \boldsymbol{\epsilon} = \left[\underbrace{\epsilon_1, \epsilon_2,}_{\mathclap{\quad\ \sim \text{Noise}(\sigma_1)}}
 \overbrace{\epsilon_3, \epsilon_4,}^{\mathclap{\sim\text{Noise}(\sigma_2)}} 
    % \underbrace{\epsilon_4, \epsilon_5}_{\text{Noise}(0, \sigma_1)},
    ..., \underbrace{\epsilon_i, \epsilon_{i+1},}_{\mathclap{\sim\text{Noise}(\sigma_3)}} \overbrace{\epsilon_{i+2},}^{\mathclap{\sim\text{Noise}(\sigma_1)}} ..., \underbrace{\epsilon_{d-2},}_{\mathclap{\sim\text{Noise}(\sigma_m)}} \overbrace{\epsilon_{d-1}, \epsilon_d}^{\mathclap{\sim\text{Noise}(\sigma_1)\quad}}\right].
\end{equation}
% \begin{equation}
%     \boldsymbol{\epsilon} = \left[\underbrace{\epsilon_1, \epsilon_2,}_{\substack{\sim\\\text{Noise}(\sigma_1)}
%     }
%  \underbrace{\epsilon_3, \epsilon_4,}_{\substack{\sim\\\text{Noise}(\sigma_2)} }
%     ..., \underbrace{\epsilon_i, \epsilon_{i+1},}_{\substack{\sim\\\text{Noise}(\sigma_3)}}\underbrace{\epsilon_{i+2},}_{\substack{\sim\\\text{Noise}(\sigma_1)}} ..., \underbrace{\epsilon_{d-2},}_{\substack{\sim\\\text{Noise}(\sigma_m)}} \underbrace{\epsilon_{d-1}, \epsilon_d}_{\substack{\sim \\ \text{Noise}(\sigma_1)}}\right]
% \end{equation}

\renewcommand{\algorithmicrequire}{\textbf{Input:}}
\renewcommand{\algorithmicensure}{\textbf{Output:}}

% \begin{wrapfigure}{R}{0.5\textwidth}
% \vspace{-20pt}
%     \begin{minipage}{0.5\textwidth}
\begin{algorithm}
\caption{Noise Generation} \label{alg: noise generation}
\begin{algorithmic}[1] 
\Require maximum noise level $\sigma_{max}$, number of noise distributions $m$, 
batch size $b$, the dimensionality of data $d$.
% a batch of $d$-dimension normal instance $\boldsymbol{x}$.
\State $\Delta \gets [0, \frac{1}{m}\sigma_{max}, \frac{2}{m}\sigma_{max}, ...,\frac{m}{m}\sigma_{max}$] \Comment{make $m$ intervals}
\State Initialize $\boldsymbol{\mathcal{E}}$ with empty
\For{$j \in \{1,...,b\}$}
    \For{$i \in \{1,...,m\}$}
        \State \textit{//random noise level}
        \State $\hat{\sigma} \gets$ \texttt{Uniform}$(\Delta[i], \Delta[i+1])$ 
        % \State \textit{//generate random noise from $m$ noise levels for $m$ parts.}
        \State $\boldsymbol{\epsilon}[\frac{i-1}{m}d: \frac{i}{m}d] \gets$ \texttt{Noise}$(0, \hat{\sigma})$ ~~ \textit{//generate noise from $m$ noise levels for $m$ parts}
    \EndFor 
    \State \textit{//shuffle position of noise elements}
    \State $\boldsymbol{\epsilon} \gets$ \texttt{Shuffle}$(\boldsymbol{\epsilon})$ 
    \State $\boldsymbol{\mathcal{E}}[j] \gets \boldsymbol{\epsilon}$
\EndFor
% \State $\boldsymbol{\Sigma}[i] \gets$ \texttt{Uniform}$(\Delta[i], \Delta[i+1]), \forall i\in\{1,...,m\}$ % \Comment{Smooth the counts}
% \State $\boldsymbol{\epsilon}[\frac{d}{m}(i-1): \frac{d}{m}i] \gets$ \texttt{Noise}$(0, \boldsymbol{\Sigma}[i]), \forall i\in\{1,...,m\}$ 
% \State \Comment{Randomly generate noise from $m$ noise levels for $m$ parts.}
% \State $\boldsymbol{\epsilon} \gets$ \texttt{Shuffle}$(\boldsymbol{\epsilon})$ \Comment{shuffle of position of noise elements.}
%% \State $\tilde{\boldsymbol{x}} \gets \boldsymbol{x} + \boldsymbol{\epsilon}$
% \State \Return $\boldsymbol{\mathcal{E}}$;
\Ensure Generated noise $\boldsymbol{\mathcal{E}}$ for a batch of input data
\end{algorithmic}
\end{algorithm}
% \end{minipage}
% \vspace{-10pt}
% \end{wrapfigure}
For one type of distribution (e.g. Gaussian), there are often two hyper-parameters to control the noise generation process. The first one is $\sigma_{max}$, denoting the maximum noise level. The other is $m$, describing how many parts of the feature vector are added by the same level of noise. Therefore, for an $\boldsymbol{x}$, different elements may be corrupted by different levels of noise, and, if necessary, one $\boldsymbol{x}$ can produce multiple $\hat{\boldsymbol{x}}$ with different types of distribution. A detailed process for generating noised samples on a batch of data is in Algorithm \ref{alg: noise generation}. According to Algorithm \ref{alg: noise generation}, the noise generation time complexity is $\mathcal{O}(bd)$. In contrast, other methods involving perturbation \cite{cai2022perturbation, qiu2021neural} and adversarial sample \cite{goyal2020drocc} have time complexity with $\mathcal{O}(bdW)$, where $W$ is workload related to a neural network module. Compared with them, our noise generation is more efficient, where no learnable parameter is required. A comparative study on time cost is shown in Appendix I.

In Figure \ref{fig: overview}, the lower pathway illustrates the noise synthesis mechanism, where a noise vector $\boldsymbol{\epsilon}$ (mean 0, standard deviation $\boldsymbol{\sigma}$) is randomly generated and added to the input $\boldsymbol{x}$, producing a noise-augmented variant $\hat{\boldsymbol{x}} = \boldsymbol{x} + \boldsymbol{\epsilon}$. Multiple noised samples can be generated from a single input. Both $\boldsymbol{x}$ and $\hat{\boldsymbol{x}}$ are processed by the noise evaluation network \( h_\theta \), optimized to regress towards zero for $\boldsymbol{x}$ and estimate the noise vector $|\boldsymbol{\epsilon}|$ for $\hat{\boldsymbol{x}}$. Training details are in Appendix B\ref{app: training alg}, Algorithm \ref{alg: training}.
Notice that for each training epoch, we randomly generate a new noised instance for each training sample. This helps us enlarge the sampling number from the noise distribution, and avoid over-fitting on some ineffective noise. There are some optional noise generation schemes such as using different noise types, different noise levels, and different noise ratios, which are studied later. These options add several hyper-parameters to our methods. However, it is still simpler than many recent methods. We compare them in Appendix J.

\section{Theoretical Analysis} 
In the proposed method, we learn a one-class classification decision boundary closely around the normal samples. In Figure \ref{fig: egg}, we divide the anomaly region into two non-overlap distributions, easy $\tilde{\mathcal{D}}_E$, and hard $\tilde{\mathcal{D}}_H$, respectively. In this section, we theoretically show the anomaly detection ability of the model meeting easy anomaly samples $\tilde{\mathcal{X}}_E$ from $\tilde{\mathcal{D}}_E$ and hard anomaly samples $\tilde{\mathcal{X}}_H$ from $\tilde{\mathcal{D}}_H$ in Theorem \ref{theorem: gap} and Theorem \ref{theorem_dmin}, respectively.

Without loss of generality, we let $h_{\boldsymbol{\theta}}\in\mathcal{H}$, where $\mathcal{H}$ is the hypothesis space of ReLU-activated neural network with $L$ layers and the number of neurons at each layer is in the order of $p$. 
We give the following definition.
\begin{definition}\label{def_1}
    For the noise evaluation hypothesis $h_{\theta}$, the risk of the hypothesis on a distribution $\mathcal{D}$ is defined by the probability according to $\mathcal{D}$ that a processed hypothesis $I(g(h_{\boldsymbol{\theta}}(\boldsymbol{x})) >\tau)$ disagrees with a labeling function $\rho: \mathbb{R}^d \rightarrow \{0, 1\}$. Mathematically,
$$
\varepsilon_{\mathcal{D}}(h) := \mathbb{E}_{\boldsymbol{x}\sim\mathcal{D}}\left[\left|I\left(g\left(h_{\boldsymbol{\theta}}(\boldsymbol{x})\right) >\tau\right) - \rho(\boldsymbol{x})\right|\right],
$$
where $I(\cdot)$ is the indicator function, $\tau > 0$ is some threshold, and $\rho(\cdot)$ is the true labeling function. 
\end{definition}

Definition \ref{def_1} is a form of the disagreement metric \cite{hanneke2014theory}. Based on Definition \ref{def_1} and results of \citep{ben2010theory, bartlett2019nearly}, we can provide the following theoretical guarantee (proved in Appendix C\ref{app: proof_eed}) for the hard anomalous data $\tilde{\mathcal{X}}_H$.
\begin{theorem} \label{theorem: gap} (Generalization Error Bound)
Let $\varepsilon_{\tilde{\mathcal{D}}_H}(h)$ and $\varepsilon_{\hat{\mathcal{D}}}(h)$ be the risks of hypothesis $h$ on $\tilde{\mathcal{D}}_H$ and $\hat{\mathcal{D}}$, respectively.
If $\hat{\mathcal{X}}$ and $\tilde{\mathcal{X}}_H$ are unlabeled samples of size $N$ each, drawn from $\hat{\mathcal{D}}$ and $\tilde{\mathcal{D}}_H$ respectively, then for any $\delta \in (0, 1)$, with probability at least $1 - \delta$, for every $h \in \mathcal{H}$:
$$
\varepsilon_{\tilde{\mathcal{D}}_H}(h) \leq \varepsilon_{\hat{\mathcal{D}}}(h) + \tfrac{1}{2} \hat{d}_{\mathcal{H}\Delta\mathcal{H}}(\tilde{\mathcal{X}}_H, \hat{\mathcal{X}})+ 4 \sqrt{\tfrac{2d_{vc} \log(2N) + \log(\frac{2}{\delta})}{N}} + \lambda,
$$
where $d_{vc}=\mathcal{O}(p L \log (p L))$. $\lambda=\varepsilon_{\hat{\mathcal{D}}}(h^*) + \varepsilon_{\tilde{\mathcal{D}}_{H}}(h^*)$ where $h^*$ is the ideal joint hypothesis that minimize $\varepsilon_{\hat{\mathcal{D}}}(h) + \varepsilon_{\tilde{\mathcal{D}}_{H}}(h)$.
\end{theorem}
The definition of the divergence $\hat{d}_{\mathcal{H}\Delta\mathcal{H}}$ is shown in Appendix C\ref{app_H_div} for simplicity. As shown in Figure \ref{fig: egg}, Theorem \ref{theorem: gap} describes that if the divergence $\hat{d}_{\mathcal{H}\Delta\mathcal{H}}$ between the perturbed training data $\hat{\mathcal{X}}$ and the hard test data $\tilde{\mathcal{X}}_H$ is small, the model can correctly identify the anomaly. In other words, the generated data $\hat{\mathcal{X}}$ can be very useful for learning a reasonable detection model $h_{\boldsymbol{\theta}}$. Since there is no available information about $\tilde{\mathcal{X}}_H$ during the training, we cannot obtain the divergence $\hat{d}_{\mathcal{H}\Delta\mathcal{H}}$ in real practice. Hence, we always standardize the normal training data and make the added noise level relatively small. In the ablation study, we show that a large noise level is harmful to the training. Note that this guarantee does not apply to $\tilde{\mathcal{X}}_E$ because $\tilde{\mathcal{D}}_E$ may be very far from $\hat{\mathcal{D}}$. The following context will provide a guarantee for detecting $\tilde{\mathcal{X}}_E$.
We make the following assumption and present the theoretical result (proved in Appendix C\ref{app: proof_dmin}).
\begin{assumption} \label{ass: gap}
    For any ${\boldsymbol{x}}$ and $\tilde{\boldsymbol{x}}$ drawn from ${\mathcal{D}}$ and $\tilde{\mathcal{D}}_E$ respectively, there exists a constant $c>0$ such that $c\|{\boldsymbol{x}} - \tilde{\boldsymbol{x}} \| \le |g(h_{\boldsymbol{\theta}}({\boldsymbol{x}})) - g(h_{\boldsymbol{\theta}}(\tilde{\boldsymbol{x}}))|$.
\end{assumption}

\begin{theorem}\label{theorem_dmin} 
Let $d_{\text{min}}=\inf_{{\boldsymbol{x}}\sim{\mathcal{D}},\tilde{\boldsymbol{x}}\sim\tilde{\mathcal{D}}_E}\|{\boldsymbol{x}}-\tilde{\boldsymbol{x}}\|$ (shown by Figure \ref{fig: egg}). Suppose $g(h_{\boldsymbol{\theta}}(\boldsymbol{x}))<\epsilon$ for any $\boldsymbol{x}\in\mathcal{D}$, where $\epsilon\geq 0$. Then, under Assumption \ref{ass: gap}, any anomalous samples drawn from $\tilde{\mathcal{D}}_E$ can be successfully detected if $d_{\min}>\max\{\epsilon/c,\tau/c\}$.
\end{theorem}
This theorem provides a theoretical guarantee for our method to detect any anomalous sample in $\tilde{\mathcal{X}}_E$ or drawn from $
\tilde{\mathcal{D}}_E$ (depicted by the blue double-headed arrow in Figure \ref{fig: egg}) as anomaly successfully. When $c$ is larger, the detection is easier. 
We can derive a bound for $c$ with more assumptions: if the learned model $h$ is bijective and $L$-Lipschitz, then $h^{-1}$ is $1/L$-Lipschitz, meaning that $c=\alpha L$, where $\alpha=\inf_ {\boldsymbol{z} \in {dom}(g)}\Vert\nabla g(z)\Vert$. Using an invertible neural network \cite{behrmann2019invertible} could achieve this. For an $h$ with $q$ layers, spectral norm $\Vert\mathbf{W}_ i\Vert_ {\sigma}$ for layer $i$, and $\rho$-Lipschitz activation, we have $c=\alpha\prod_ {i=1}^q\rho^q\Vert\mathbf{W}_ i\Vert_ {\sigma}^q$. 
In our experiments, we used both MLP and ResMLP without dimensionality reduction. According to Theorem 1 of \cite{behrmann2019invertible}, ResMLP is invertible if each layer’s Lipschitz constant is under 1, which is easy to ensure.

% Note that computing $c$ remains a challenge in the field of neural networks and deep learning, which is out of the scope of this paper, though one may take advantage of random matrix theories.

To sum up, Theorem \ref{theorem: gap} and Theorem \ref{theorem_dmin} provide guarantees for detecting $\mathcal{X}_H$ and $\mathcal{X}_E$ respectively. Therefore, our method is theoretically guaranteed to detect $\tilde{\mathcal{X}}$, although we never use any real anomalous data in the training stage.

\section{Experiments}

\subsection{Experimental Settings} \label{sec: exp}
\begin{comment}
\begin{table}[h]
\centering
\begin{tabular}{l|cccc} 
% \hline
\toprule
\textbf{Dataset} & \textbf{\# Sample} & \textbf{Dims.} & \textbf{\# Class} & \textbf{\% Anomaly} \\ 
\midrule
abalone & 4177 & 8 & 28 & 94.19 \\ 
arrhythmia & 452 & 280 & 2 & 45.80 \\ 
breastw & 699 & 10 & 2 & 34.48 \\ 
cardio & 2126 & 22 & 2 & 22.15 \\ 
ecoli & 336 & 8 & 8 & 77.16 \\ 
glass & 214 & 10 & 6 & 72.74 \\ 
ionosphere & 351 & 35 & 2 & 35.90 \\ 
kdd & 5209460 & 39 & 2 & 80.52 \\ 
letter & 20000 & 17 & 26 & 96.15 \\ 
lympho & 142 & 19 & 2 & 42.96 \\ 
mammography & 11183 & 7 & 2 & 2.32 \\ 
mulcross & 262144 & 5 & 2 & 10.00 \\ 
musk & 7074 & 167 & 2 & 17.30 \\
optdigits & 5620 & 65 & 10 & 90.00 \\ 
pendigits & 10992 & 17 & 10 & 90.00 \\ 
pima & 768 & 9 & 2 & 34.90 \\ 
satimage & 6435 & 37 & 6 & 83.33 \\ 
seismic & 2584 & 15 & 2 & 6.58 \\ 
shuttle & 58000 & 10 & 7 & 75.03 \\ 
speech & 3686 & 401 & 2 & 1.65 \\ 
thyroid & 7200 & 7 & 2 & 7.29 \\ 
vertebral & 310 & 7 & 2 & 67.74 \\ 
vowels & 1456 & 13 & 2 & 3.43 \\ 
wbc & 569 & 31 & 2 & 37.26 \\ 
wine & 178 & 14 & 3 & 66.67 \\
\bottomrule
\end{tabular}
\caption{25 real-world tabular datasets tested in this paper. For datasets with more than 2 labeled classes, an average anomaly ratio is recorded. Kdd refers to KDD-CUP99 \citep{kdd}.}
\label{tab: dataset}
\end{table}
\end{comment}

\paragraph{\textbf{Datasets}} 
We evaluate our method in two common settings: unsupervised anomaly detection and one-class classification. In the anomaly detection setting, where anomalous samples are few, we use 47 real-world tabular datasets\footnote{https://github.com/Minqi824/ADBench/} from \citep{han2022adbench}, covering domains like healthcare, image processing, and finance. For one-class classification setting, we collected 25 benchmark tabular datasets used in previous works \citep{pang2021deep, shenkar2022anomaly}. The raw data was sourced from the UCI Machine Learning Repository \citep{UCI} and their official websites. For categoric value, we use a one-hot encoding. We test on all classes in multi-class datasets, reporting the average performance score per class, which is similar to one-class classification on image dataset \citep{cai2022perturbation}. 
For datasets with validation/testing sets, we train on all normal samples. If only a training set is available, we randomly split 50\% of normal samples for training and use the rest with anomalous data for testing. Data is standardized using the training set's mean and standard deviation. Dataset details are in Table \ref{tab: dataset} of Appendix D.

\paragraph{\textbf{Baseline Methods}} We select 12 baseline methods for comparative analysis, including probabilistic-based, proximity-based, deep neural network-based, ensemble-based methods, and recent UAD methods that can be applied to tabular data. They are Isolated Forest (\textbf{IForest}) \citep{liu2008isolation}, \textbf{COPOD} \citep{li2020copod}, auto-encoder (\textbf{AE}) \citep{aggarwal2016outlier}, \textbf{KNN} \citep{angiulli2002fast}, Local Outlier Factor (\textbf{LOF}) \citep{breunig2000lof}, \textbf{DeepSVDD} \citep{ruff2018deep}, \textbf{AnoGAN} \citep{schlegl2017unsupervised}, \textbf{ECOD} \citep{li2022ecod}, \textbf{SCAD} \citep{shenkar2022anomaly}, \textbf{NeuTraLAD} \cite{qiu2021neural}, \textbf{PLAD} \cite{cai2022perturbation}, and \textbf{DPAD} \cite{fu2024dense}. For DPAD, PLAD, SCAD, and NeuTraLAD, we use the code provided by the authors.
For the other baseline methods, we utilize PyOD, a Python library developed by \citep{zhao2019pyod}. The default settings are adopted. We repeat 10 times for each baseline.

\begin{figure}
% {r}{0.5\textwidth}
\centering
        \includegraphics[width=0.70\linewidth,trim={10 10 -15 0},clip]{ 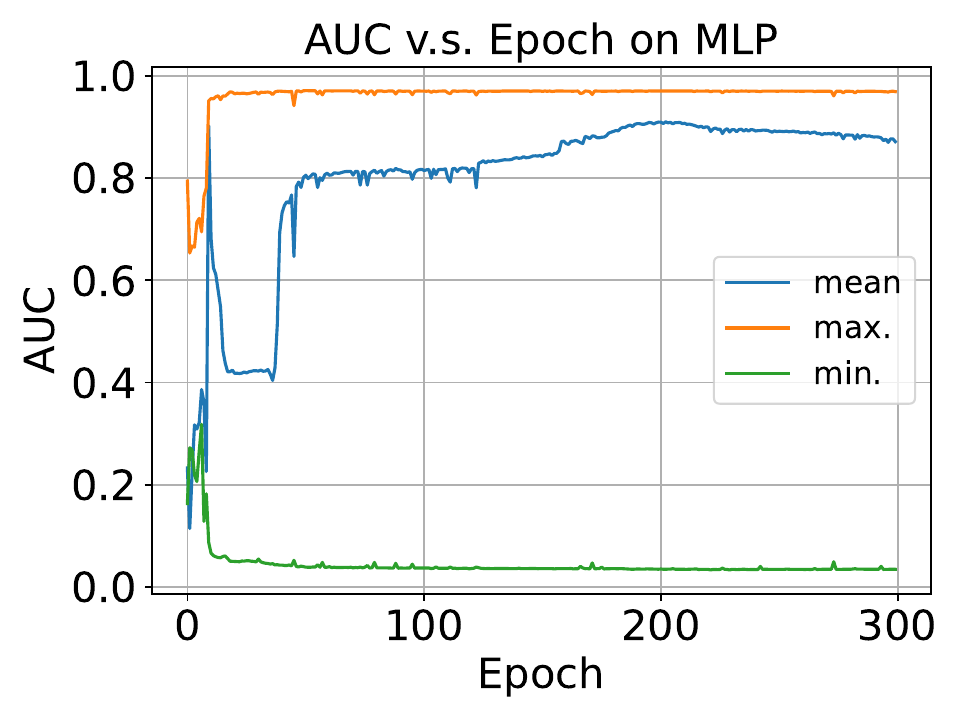}
        \caption{Comparison of different $g(\cdot)$, i.e. mean, maximum, and minimum, on KDD-CUP99, at each optimization epoch. 
        % The maximum $g$ reaches the pinnacle of performance metrics, achieving the highest AUC score among the depicted measures. This rapid attainment of the optimal score reflects the model's expedited convergence capabilities.
        }\label{fig: metric}
        \vspace{-10pt}
\end{figure}

\begin{figure*}
\begin{subfigure}{0.49\linewidth}
  \centering
  \includegraphics[width=\linewidth]{ 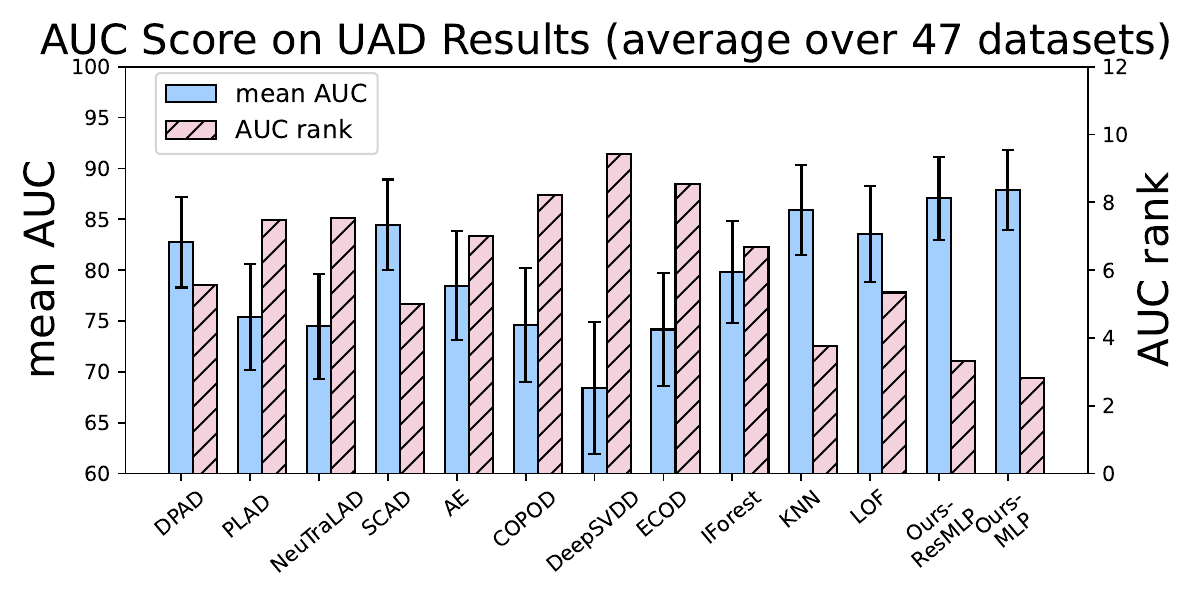}  
  \caption{AUC and Rank on UAD Results.}
\end{subfigure}
\begin{subfigure}{0.49\linewidth}
  \centering
  \includegraphics[width=\linewidth]{ 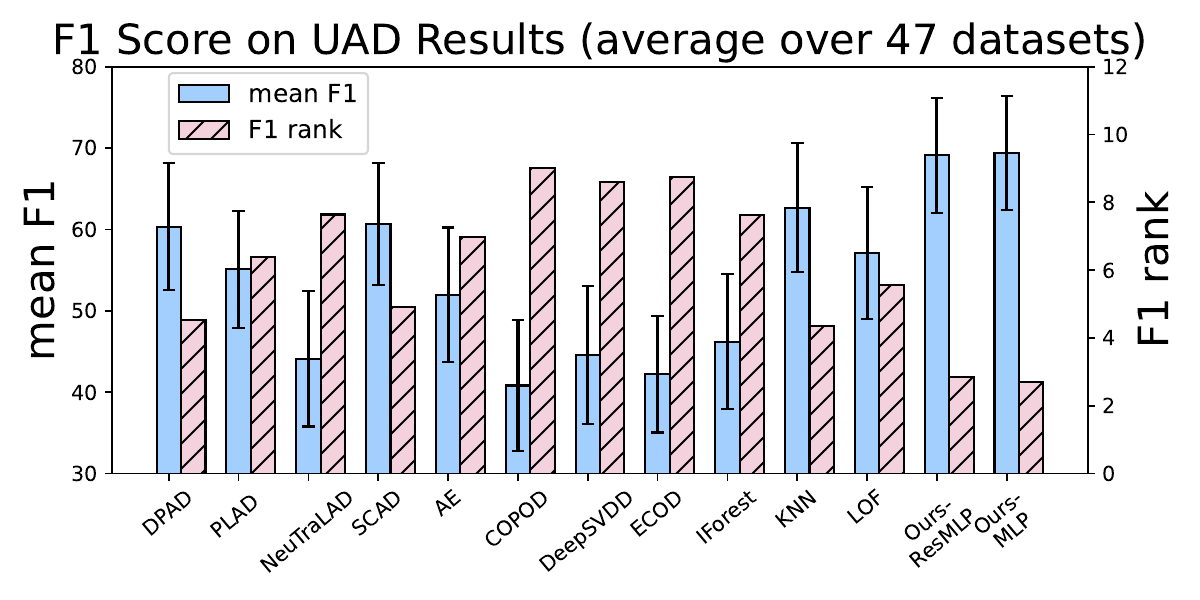}  
  \caption{F1 and Rank on UAD Results.}
\end{subfigure}
% \begin{subfigure}{.243\textwidth}
%   \centering
%   \includegraphics[width=\linewidth]{ ablation_type_auc_mlp.pdf}  
%   \caption{AUC and Rank on MLP with Different Noise Type}
% \end{subfigure}
% \begin{subfigure}{.243\textwidth}
%   \centering
%   \includegraphics[width=\linewidth]{ ablation_type_f1_mlp.pdf}  
%   \caption{F1 and Rank on MLP with Different Noise Type}
% \end{subfigure}
\caption{AUC (\%) and F1 (\%) score of the proposed method compared with 11 baselines on 47 benchmark datasets. Each experiment is repeated 10 times with random seed from 0 to 9, and mean value and 95\% confidence interval are reported. Rank (the lower the better) is calculated out of 12 tested methods.}
\label{fig: adbench results}
\end{figure*}

\paragraph{\textbf{Implementation}} All experiments are implemented by Pytorch \citep{paszke2017automatic} on NVIDIA Tesla V100 and Intel Xeon
Gold 6200 platform. We utilize two network architectures for the evaluation, VanillaMLP (\textbf{MLP}) and \textbf{ResMLP}. VanillaMLP is a plain ReLU-activated feed-forward neural network with 4 fully connected layers. ResMLP has a similar architecture to \citep{touvron2022resmlp} but has a simpler structure. Detailed architectures are in Appendix E\ref{app: network}. The network model is optimized by AMSGrad \citep{reddi2018convergence} with $10^{-4}$ learning rate and $5\times10^{-4}$ weight decay. In the results, we adopt Gaussian noise in the noise generation, maximum noise level $\sigma_{max}=2$, and the number of different noise distributions $m=3$. To enlarge the noised sample number, we generate 3 noise-augmented instances for the same input instance with 3 different noise ratios, $0.5$, $0.8$, and $1.0$ at each epoch. Unless specified, we train the model for $500$ epochs and manually decay the learning rate at the $100$-th epoch in a factor of $0.1$.

\paragraph{\textbf{Performance Metric}}
Since the output of $h$ is a noise evaluation vector, an operation $g(\cdot)$ is introduced, as explained by \eqref{eq_score}. Here we select the maximum. A comparative result on KDD-CUP99 among maximum, minimum, and mean is shown in Figure \ref{fig: metric} where we train the model for 300 epochs. The maximum reaches the highest performance and the fastest convergence speed. For the performance metric, we report the area under the ROC curve (AUC) and F1 score. The calculation of the F1 score and the determination of anomaly threshold $\tau$ are consistent with \citep{shenkar2022anomaly, qiu2021neural}. 
% Here the F1 score is equivalent to the precision@n where $n$ is the number of anomalies in the test set.

\begin{table*}[h!]
\setlength\tabcolsep{1.5pt}
\centering
% \sisetup{table-format=3.2(3.1), table-space-text-post=*, separate-uncertainty=true, detect-all}
\resizebox{\textwidth}{!}{
\begin{tabular}{l|cccccccccccc|cc}
\toprule
\hline
\textbf{Methods}& \makecell[c]{\textbf{DPAD}\\\textbf{(2024)}}& \makecell[c]{\textbf{PLAD} \\ \textbf{(2022)}}& \makecell[c]{\textbf{NeuTraLAD}\\\textbf{(2021)}} & \makecell[c]{\textbf{SCAD}\\\textbf{(2022)}} & \textbf{AE} & \textbf{AnoGAN} & \textbf{COPOD} & \textbf{DeepSVDD} & \makecell[c]{\textbf{ECOD} \\ \textbf{(2023)}} & \textbf{IForest} & \textbf{KNN} & \textbf{LOF} & \makecell[c]{\textbf{Ours-} \\\textbf{ResMLP}} & \makecell[c]{\textbf{Ours-} \\\textbf{MLP}} \\
\midrule
abalone   &70.64$\pm$13.2 & 64.23$\pm$15.4& 45.29$\pm$8.8   & 63.90$\pm$14.7 & \underline{71.69}$\pm$11.1 & 68.65$\pm$14.7 & 43.57$\pm$15.0 & 59.78$\pm$16.1 & 45.00$\pm$16.9 & 71.55$\pm$12.2 & 69.11$\pm$14.1 & 68.89$\pm$13.0 & \textbf{76.90}$\pm$11.0 & \textbf{76.78}$\pm$10.9 \\
arrhythmia & 76.38$\pm$1.0 & 63.15$\pm$5.1 & 52.13$\pm$0.0   & 69.60$\pm$2.0 & 76.12$\pm$1.8 & 69.61$\pm$2.6 & 75.23$\pm$1.4 & 72.82$\pm$2.2 & 75.21$\pm$1.5 & \underline{77.06}$\pm$1.8 & 75.87$\pm$2.0 & 75.85$\pm$1.9 & \textbf{77.34}$\pm$2.4 & 76.77$\pm$2.1 \\
breastw  & 98.36$\pm$0.3 & 87.4$\pm$15.0 & 78.77$\pm$3.0    & 97.29$\pm$0.7 & 98.83$\pm$0.4 & 99.11$\pm$0.3 & \textbf{99.29}$\pm$0.2 & 95.58$\pm$1.1 & 98.62$\pm$0.3 & \underline{99.36}$\pm$0.2 & 98.84$\pm$0.3 & 95.43$\pm$1.6 & 98.49$\pm$0.7 & 98.84$\pm$0.4 \\
cardio   & 82.13$\pm$3.3 & 74.82$\pm$8.9 & 55.84$\pm$2.6    & 69.37$\pm$1.8 & \underline{83.36}$\pm$0.8 & 76.78$\pm$12.2 & 65.98$\pm$0.7 & 53.12$\pm$7.7 & 78.03$\pm$0.6 & 81.74$\pm$2.0 & 73.96$\pm$1.0 & 77.37$\pm$1.5 & \textbf{85.59}$\pm$2.4 & \textbf{86.01}$\pm$3.3 \\
ecoli   & 92.25$\pm$4.4 & 68.40$\pm$2.6 & 48.46$\pm$12.0     & 88.17$\pm$5.7 & 93.18$\pm$4.0 & 90.86$\pm$6.2 & 49.16$\pm$22.2 & 87.63$\pm$5.8 & 50.98$\pm$11.6 & \underline{93.21}$\pm$3.6 & 93.19$\pm$4.2 & 91.42$\pm$5.7 & 92.98$\pm$3.7 & \textbf{93.53}$\pm$3.7 \\
glass   & 76.73$\pm$7.4 & 64.8$\pm$10.5 & 57.68$\pm$9.0     & \underline{78.89}$\pm$6.0 & 73.39$\pm$10.4 & 74.76$\pm$11.0 & 44.61$\pm$26.1 & 75.39$\pm$8.2 & 44.75$\pm$24.2 & 76.96$\pm$10.5 & 76.66$\pm$9.0 & 71.09$\pm$9.7 & \textbf{83.12}$\pm$10.1 & \textbf{84.41}$\pm$10.2 \\
ionosphere  & 96.66$\pm$0.4 & 64.26$\pm$15.2 & 90.77$\pm$2.4 & 96.83$\pm$0.7 & 90.40$\pm$1.2 & 86.19$\pm$5.9 & 80.05$\pm$1.6 & 95.56$\pm$1.1 & 74.33$\pm$1.6 & 89.53$\pm$2.9 & \underline{97.47}$\pm$0.8 & 95.44$\pm$1.3 & \textbf{97.53}$\pm$0.3 & 97.11$\pm$0.7 \\
kdd    & 81.76$\pm$11.9 & 90.87$\pm$6.6 & 93.6$\pm$1.7       & 91.56$\pm$4.0 & 95.12$\pm$3.3 & 92.71$\pm$2.5 & 76.51$\pm$1.2 & 72.81$\pm$27.8 & 78.75$\pm$1.3 & \underline{96.12}$\pm$0.3 & 94.45$\pm$0.2 & 85.05$\pm$0.2 & \textbf{95.69}$\pm$1.2 & \textbf{96.92}$\pm$0.6 \\
letter   & 93.39$\pm$3.3 & 59.05$\pm$15.3 & 90.95$\pm$4.6    & \underline{99.26}$\pm$0.5 & 88.91$\pm$4.3 & 80.23$\pm$8.5 & 50.03$\pm$15.8 & 94.25$\pm$2.4 & 50.09$\pm$15.4 & 91.15$\pm$3.4 & 97.95$\pm$1.6 & 95.61$\pm$2.6 & \textbf{99.42}$\pm$0.3 & \textbf{99.42}$\pm$0.4 \\
lympho  & 74.21$\pm$3.6 & 64.96$\pm$8.1 & 46.6$\pm$2.2     & 74.91$\pm$4.6 & 75.84$\pm$3.7 & 70.15$\pm$6.7 & 53.38$\pm$3.7 & 73.37$\pm$5.8 & 53.52$\pm$3.5 & 75.14$\pm$5.3 & \underline{77.86}$\pm$3.2 & 76.50$\pm$3.2 & \textbf{81.14}$\pm$4.7 & \textbf{82.37}$\pm$5.0 \\
mammo. & 88.32$\pm$1.7 & 82.38$\pm$2.3 & 65.57$\pm$2.3 & 78.18$\pm$2.2 & 85.65$\pm$1.1 & 78.67$\pm$13.4 & \underline{90.54}$\pm$0.1 & 69.87$\pm$7.1 & \textbf{90.63}$\pm$0.1 & 87.91$\pm$0.7 & 87.55$\pm$0.3 & 84.12$\pm$1.2 & 90.11$\pm$0.9 & 89.89$\pm$0.4 \\
mulcross   & \textbf{100.0}$\pm$0.0 & \underline{99.90}$\pm$0.0 & 76.59$\pm$11.3  & \textbf{100.0}$\pm$0.0 & \textbf{100.0}$\pm$0.0 & 99.89$\pm$0.1 & 93.24$\pm$0.0 & \textbf{100.0}$\pm$0.0 & 95.97$\pm$0.1 & \underline{99.90}$\pm$0.1 & \textbf{100.0}$\pm$0.0 & \textbf{100.0}$\pm$0.0 & \textbf{100.0}$\pm$0.0 & \textbf{100.0}$\pm$0.0 \\
musk  & 71.36$\pm$2.9 & 74.42$\pm$3.3 & 78.86$\pm$1.6       & 81.30$\pm$0.7 & 30.68$\pm$0.6 & 41.37$\pm$12.2 & 32.03$\pm$0.4 & 62.02$\pm$6.0 & 30.16$\pm$0.4 & 46.74$\pm$3.7 & \underline{81.94}$\pm$0.4 & 81.31$\pm$0.6 & \textbf{85.06}$\pm$0.5 & \textbf{83.69}$\pm$1.0 \\
optdigits & 93.43$\pm$4.3 & 80.43$\pm$15.5 & 67.3$\pm$15.0   & \underline{97.80}$\pm$1.8 & 95.69$\pm$2.6 & 90.04$\pm$6.6 & 71.72$\pm$8.8 & 94.86$\pm$3.7 & 68.69$\pm$10.1 & 97.55$\pm$1.7 & 97.25$\pm$2.2 & 97.17$\pm$2.3 & \textbf{97.93}$\pm$1.7 & \textbf{98.00}$\pm$1.7 \\
pendigits & 98.04$\pm$2.1 & 86.45$\pm$14.9 & 97.23$\pm$2.4   & \underline{99.65}$\pm$0.3 & 95.46$\pm$4.7 & 92.56$\pm$9.8 & 49.90$\pm$25.2 & 94.25$\pm$5.3 & 49.93$\pm$22.9 & 98.08$\pm$1.7 & 99.59$\pm$0.5 & 99.11$\pm$1.0 & \textbf{99.82}$\pm$0.2 & \textbf{99.84}$\pm$0.2 \\
pima   & 70.5$\pm$1.7 & 63.67$\pm$5.7 & 48.63$\pm$4.6     & 67.97$\pm$1.3  & 69.73$\pm$1.5  & 70.94$\pm$5.0 & 65.58$\pm$1.0  & 58.18$\pm$2.9  & 59.46$\pm$1.3  & 73.33$\pm$1.4  & \textbf{74.63}$\pm$1.0 & 71.06$\pm$1.5  & \underline{72.42}$\pm$2.6  & \underline{73.53}$\pm$2.9  \\ 
satimage  & 92.83$\pm$5.1 & 54.83$\pm$19.4 & 47.98$\pm$12.0  & 93.35$\pm$4.8  & 93.39$\pm$5.4  & 91.98$\pm$6.8             & 69.24$\pm$26.8 & 75.10$\pm$7.7   & 62.75$\pm$22.8 & \textbf{95.33}$\pm$3.8 & \underline{94.85}$\pm$4.4  & 89.67$\pm$6.6  & 94.51$\pm$3.6  & 93.72$\pm$4.7  \\ 
seismic  & \textbf{74.81}$\pm$0.8 & 72.59$\pm$2.8 & 67.75$\pm$1.2   & 72.47$\pm$0.8  & 71.13$\pm$0.5  & 72.82$\pm$1.5 & 73.87$\pm$0.5  & 54.32$\pm$10.2 & 70.17$\pm$0.4  & 73.48$\pm$0.6  & \underline{74.57}$\pm$0.5 & 60.87$\pm$1.6  & 71.38$\pm$1.3  & 72.62$\pm$1.9  \\ 
shuttle & \underline{98.93}$\pm$1.4 & 88.55$\pm$14.2 & 97.42$\pm$2.4    & 98.30$\pm$2.0   & 92.49$\pm$8.8  & 88.42$\pm$16.2             & 31.08$\pm$31.4 & 96.51$\pm$6.5  & 35.53$\pm$30.8 & 91.55$\pm$6.7  & 98.66$\pm$1.4 & 95.60$\pm$6.2   & \textbf{99.35}$\pm$0.8  & \textbf{99.40}$\pm$1.0   \\ 
speech  & 54.38$\pm$2.7 & 57.32$\pm$4.9 & 56.31$\pm$4.8    & \underline{57.85}$\pm$2.4 & 47.08$\pm$0.5  & 50.33$\pm$5.4 & 49.15$\pm$0.6  & 52.20$\pm$3.8   & 47.08$\pm$0.5  & 46.88$\pm$1.6  & 48.67$\pm$0.7  & 49.81$\pm$0.5  & 57.14$\pm$2.6  & \textbf{58.38}$\pm$2.1  \\ 
thyroid   & 91.96$\pm$1.0 & \underline{97.02}$\pm$1.7 & 80.02$\pm$3.4  & 86.53$\pm$1.4  & 83.38$\pm$0.6  & 72.63$\pm$4.3            & 78.46$\pm$0.0  & 57.35$\pm$3.5  & 79.17$\pm$0.0  & 90.08$\pm$1.4  & 91.71$\pm$0.0 & 88.02$\pm$0.0  & \textbf{97.47}$\pm$0.1  & \textbf{96.91}$\pm$0.2  \\ 
vertebral & 87.31$\pm$2.6 & 56.47$\pm$28.6 & 50.35$\pm$3.5  & 78.56$\pm$3.3  & 88.58$\pm$1.3 & 87.96$\pm$3.2 & 78.63$\pm$2.0  & 80.43$\pm$3.8  & 60.02$\pm$2.9  & 86.08$\pm$2.3  & 87.18$\pm$1.0  & \underline{88.20}$\pm$1.5   & \textbf{90.91}$\pm$1.4  & \textbf{89.91}$\pm$1.2  \\ 
vowels  & 80.22$\pm$4.1 & 79.77$\pm$17.1 & 97.25$\pm$1.1    & \textbf{99.52}$\pm$0.2 & 62.80$\pm$0.8   & 59.32$\pm$7.8 & 49.73$\pm$0.8  & 72.01$\pm$6.5  & 59.39$\pm$0.7  & 78.17$\pm$2.2  & 97.16$\pm$0.4  & 95.53$\pm$0.7  & \underline{99.42}$\pm$0.3  & \underline{99.14}$\pm$0.2  \\ 
wbc    & 95.10$\pm$1.1 & 68.07$\pm$15.0 & 53.01$\pm$12.6      & 93.11$\pm$1.3  & 95.63$\pm$0.6  & 94.01$\pm$2.5 & 86.53$\pm$1.0  & 90.22$\pm$3.3  & 62.46$\pm$1.4  & \underline{95.74}$\pm$0.8  & 94.68$\pm$0.9  & 95.09$\pm$0.9  & \textbf{96.98}$\pm$1.1  & 96.73$\pm$1.1  \\ 
wine  & 95.68$\pm$4.5 & 67.13$\pm$7.8 & 43.95$\pm$12.0      & 86.21$\pm$7.8  & 97.00$\pm$4.1   & 94.47$\pm$10.4   & 49.37$\pm$7.8  & 95.46$\pm$5.5  & 49.37$\pm$11.5 & 96.23$\pm$3.1  & 97.81$\pm$2.8  & \underline{97.98}$\pm$2.4  & \textbf{98.76}$\pm$1.8  & \textbf{98.33}$\pm$2.3  \\ \hline
mean   &88.05$\pm$12.4 & 67.42$\pm$19.6 & 71.45$\pm$22.6   & 88.59$\pm$15.3 & 85.68$\pm$13.2& 80.17$\pm$14.8  & 53.38$\pm$22.4& 80.85$\pm$19.1& 52.77$\pm$20.9 & 87.0$\pm$12.8& \underline{90.37}$\pm$13.5 & 88.72$\pm$13.4& \textbf{92.68}$\pm$11.0& \textbf{92.27}$\pm$11.1\\ 
mean rank  &4.96& 9.16 & 10.16 & 5.96 & 6.00 & 8.12 & 10.12 & 8.68 & 10.40 & 4.96 & 3.92 & 5.88 & \textbf{2.04} & \textbf{1.68}\\ \hline
p-value (ResMLP) &0.0003&0.0000&0.00000& 0.0001 & 0.0066 & 0.0002 & 0.0000 & 0.0000 & 0.0000 & 0.0054 & 0.0025 & 0.0000 & - & 0.47 \\
p-value (MLP) &0.0002&0.0000&0.0000&  0.0001 & 0.0051 & 0.0002 & 0.0000 & 0.0000 & 0.0000 & 0.0036 & 0.0022 & 0.0000 & 0.47 & -\\ \hline
\bottomrule
\end{tabular}}
\caption{AUC (\%) score of the proposed method compared with 12 baselines on 25 benchmark datasets. Each experiment is repeated 10 times with random seed from 0 to 9, and mean value and standard deviation are reported. Mean is the average AUC score under all experiments. Mean rank (the lower the better) is calculated out of 10 tested methods. Mammo. refers to mammography.}
\label{tab: main result}
\end{table*}

\subsection{Unsupervised Anomaly Detection Results}
The average result under the unsupervised anomaly detection setting of ten runs is reported in Figure \ref{fig: adbench results}. We conducted a comparative analysis of 11 tabular UAD methods across 47 datasets, evaluating average AUC, average F1, average ranking in AUC, and average ranking
in F1. It is seen that our methods not only significantly outperform the AE methods, but also reach the highest ranking. The detailed results on each dataset and p-value from paired t-test are reported in Appendix F\ref{app: ad results} (Tables \ref{tab: adbench auc} and \ref{tab: adbench f1}), which emphasizes the statistical significance of the improvements achieved by our methods. 

\subsection{One-Class Classification Results}

We compared our noise evaluation method with 12 baseline methods on the OCC dataset setting. The results in Table \ref{tab: main result} show that our anomaly detection techniques, Ours-ResMLP and Ours-MLP, consistently outperform baselines across various tabular datasets, with mean AUC values of $92.68\pm11.10$ and $92.27\pm11.1$, indicating greater effectiveness and lower variability. While traditional methods like IForest and KNN perform well, our methods excel, especially on complex datasets like musk and optdigits. Though our methods may not always lead in AUC, the difference is minimal, around 1\%. For faster inference, the MLP model is recommended. The last two lines in the table are the p-values of paired t-test of our two methods (ResMLP and MLP) against each baseline method. The paired t-test is based on 25 pairs  (as there are 25 datasets). Our approach also achieved the highest average F1 score and rank, $94.18\%$ and $1.52$, as detailed in Appendix G\ref{app: f1}.

Compared with many popular UAD methods, our noise evaluation method has the following advantages:
\begin{itemize}
    \item Noise evaluation shows superior performance in the extensive empirical experiments, indicating our method can accurately identify anomalies in tabular data. 
    \item The generalization and anomaly detection ability of noise evaluation can be theoretically guaranteed, whereas many deep learning based methods lack such guarantees \citep{hussain2023reliable}.
    \item The implementation of our method is easier. The perturbed sample generation can be pre-generated without extra training. In addition, no hyper-parameter is tuned in the learning objective.
\end{itemize}

 \begin{table}[H]
 % {r}{0.5\textwidth}
\centering
\setlength\tabcolsep{0.9pt}
\resizebox{0.8\linewidth}{!}{
\begin{tabular}{l|ccc} 
% \hline
\toprule
\textbf{Type} & \textbf{Parameter} & \textbf{Value} & \textbf{Offset}\\ 
\midrule
Gaussian & $\mu, \sigma$ & $\mu=0, \sigma=\hat{\sigma}$ & $0$  \\ 
Laplace & $\mu, \sigma$ & $\mu=0, \sigma=\hat{\sigma}$ & $0$  \\ 
Uniform & $a, b$ & $a=-\sqrt{3}\hat{\sigma}, b=\sqrt{3}\hat{\sigma}$ & $0$  \\ 
Rayleigh & $s$ & $s=\sqrt{\frac{2}{4-\pi}}\hat{\sigma}$ & $-\sqrt{\frac{\pi}{4-\pi}}\hat{\sigma}$  \\ 
Gamma & $\alpha, \beta$ & $\alpha=\sqrt{\frac{1}{\beta}}\hat{\sigma}, \beta=\beta$ & $-\sqrt{\beta}\hat{\sigma}$ \\ 
Poisson & $\lambda$ & $\lambda=\hat{\sigma}$ & $-\hat{\sigma}$ \\ 
\bottomrule
\end{tabular}}
\caption{Eight different noise types are adopted. 
We adjust the parameter to make it $0$ mean value and $\hat{\sigma}$ standard deviation, where $\hat{\sigma}$ is the designated noise level. If the noise is non-negative by default, we offset its mean to $0$. For the Gamma noise, $\beta$ is a non-negative hyper-parameter, where we evaluate $\beta=1$ and $\beta=3$. For Salt\&Pepper and Bernoulli noise, we generate a probability vector based on a uniform distribution. Then, generate a binary vector using the probability vector. If there is noise, we change the value into the maximum or minimum value in the batch or flip the sign of the element of $\boldsymbol{x}$.}
\label{tab: noise type}
% \vspace{-25pt}
\end{table}

\subsection{Ablation Study}
As discussed in the Proposed Method Section, we have several alternatives. In this section, we study how different noise levels, noise ratios, and noise types affect the performance of our method. We utilize the OCC setting to perform the ablation study. Detailed results are shown in Appendix H.

\begin{figure*}[h]
\begin{subfigure}{.249\textwidth}
  \centering
  \includegraphics[width=\linewidth]{ 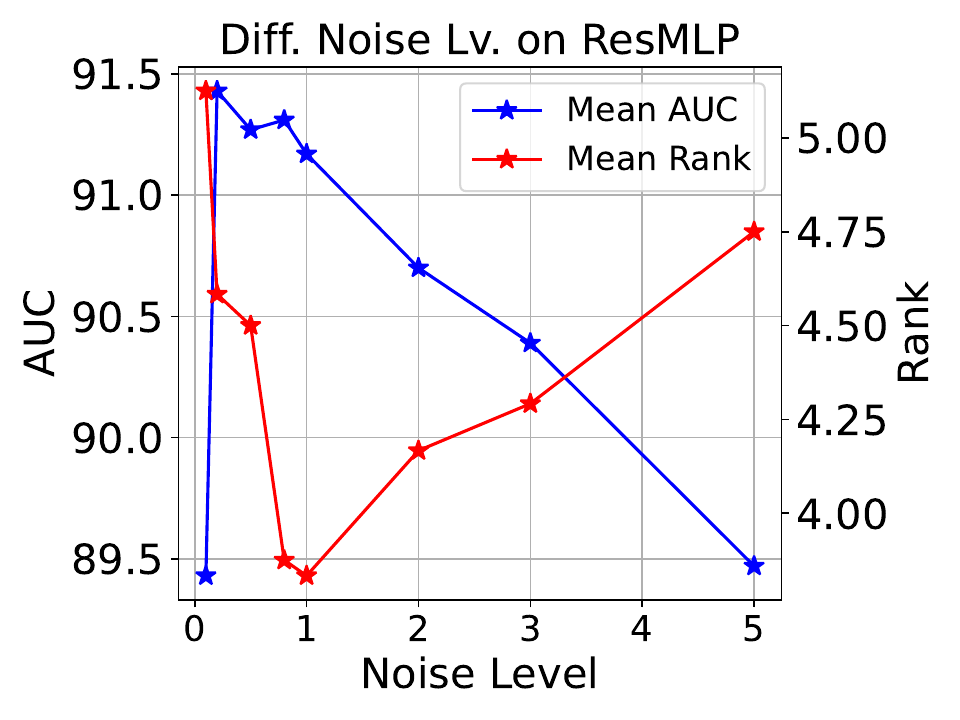}  
  \caption{AUC and Its Rank on ResMLP with Different Noise Levels}
  % \label{fig: ablation pfrac}
\end{subfigure}
\begin{subfigure}{.249\textwidth}
  \centering
  \includegraphics[width=\linewidth]{ 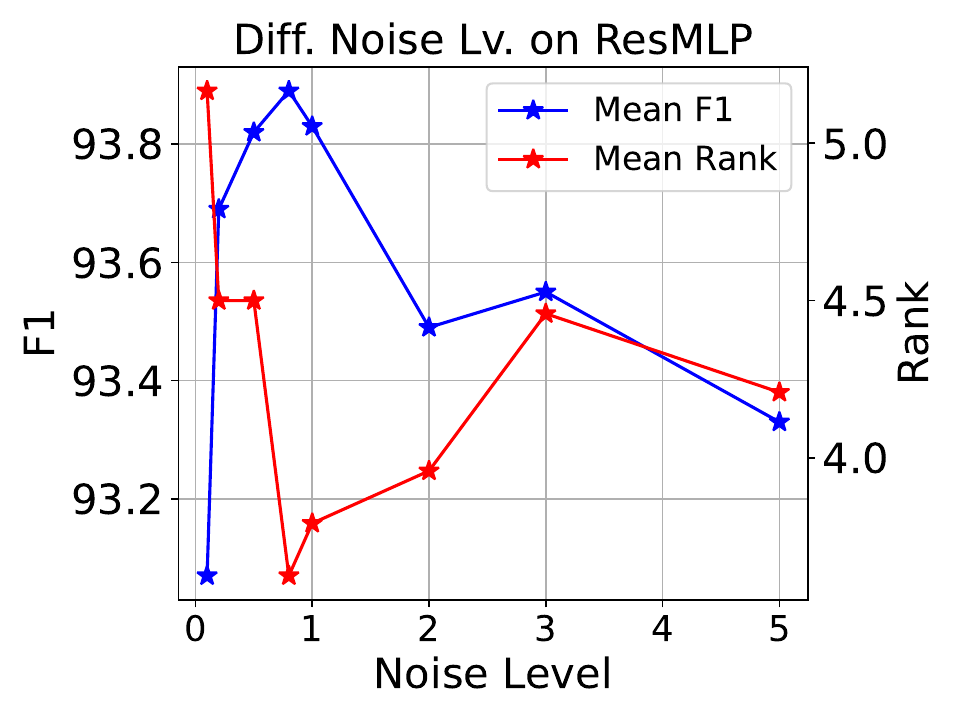}  
  \caption{F1 and Its Rank on ResMLP with Different Noise Levels}
  % \label{fig: ablation mfrac}
\end{subfigure}
\begin{subfigure}{.243\textwidth}
  \centering
  \includegraphics[width=\linewidth]{ 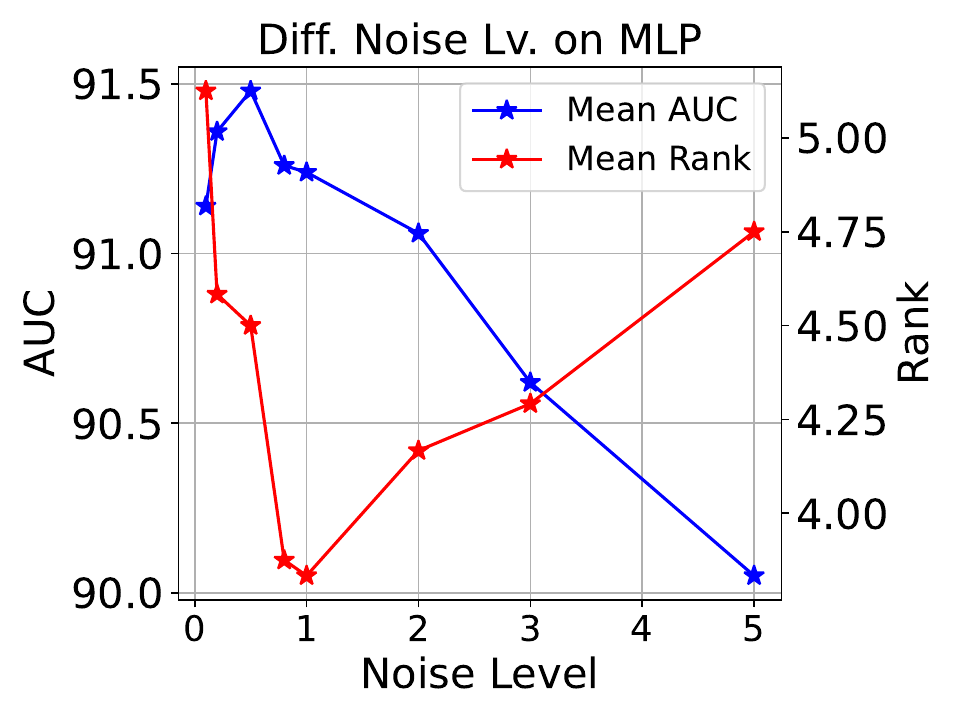}  
  \caption{AUC and Its Rank on MLP with Different Noise Levels}
  % \label{fig: ablation period}
\end{subfigure}
\begin{subfigure}{.243\textwidth}
  \centering
  \includegraphics[width=\linewidth]{ 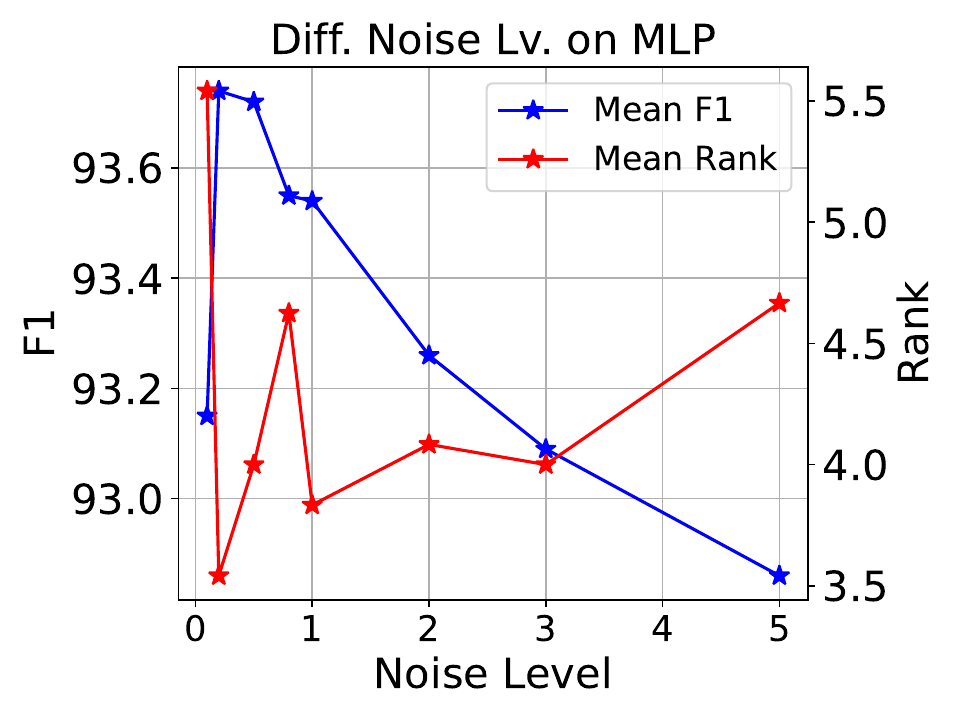}  
  \caption{F1 and Its Rank on MLP with Different Noise Levels}
  % \label{fig: ablation rank}
\end{subfigure}
\caption{Sensitivity of Different Noise Level in $[0.1, 0.2, 0.5, 0.8, 1.0, 2.0, 3.0, 5.0]$. The mean rank (the lower the better) is calculated out of 8 noise levels.}
\label{fig: ablation noise level}
\end{figure*}

\paragraph{\textbf{Sensitivity of Different Noise Levels}} 
To explore how different noise levels affect performance, we use Gaussian noise and generate 3 noised instances per training batch, consistent with previous settings. We test noise levels in $[0.1, 0.2, 0.5, 0.8, 1.0, 2.0, 3.0, 5.0]$. The results are reported in Figure \ref{fig: ablation noise level}. Results show that too small noise levels confuse the model due to the minimal distance between normal samples and anomalies, while too high noise levels expand the output value range and sampling space, reducing effectiveness as theorem \ref{theorem: gap} suggested. Hence, the optimal noise level is around $1.0$.

\paragraph{\textbf{Sensitivity of Different Noise Types}}We explore different noise types with a mean of $0$ and a standard deviation (noise level) of $\sigma$. We use Salt\&Pepper noise, Gaussian, Laplace, Uniform, Rayleigh, Gamma, Poisson, and Bernoulli distributions. For Salt\&Pepper and Bernoulli noise, a probability vector from a uniform distribution generates a binary vector, dictating feature value alterations by changing values or reversing signs. Other distributions are adjusted to have zero mean and $\sigma$ standard deviation. Detailed parameters are in described in Table \ref{tab: noise type}, with results in Appendix H\ref{app: type results}, Figure \ref{fig: ablation noise type}. Results show Salt\&Pepper and Bernoulli noise performs poorly, indicating the effectiveness of our noise generation design. Gaussian, Rayleigh, and Uniform noise maintain stable performance with 92\% AUC and 94\% F1 scores, indicating our method's robustness with various noise types.

\paragraph{\textbf{Qualitative Visualization}} We utilize t-Distributed
Stochastic Neighbor Embedding (t-SNE) visualization \citep{van2008visualizing} to show the allocation of normal, noised, and true anomalous samples in the neural network embedding space qualitatively. In Figure \ref{fig: visualization}, we visualize the penultimate layer of the neural network using KDD-CUP99. It is seen that the introduction of noised instances into the dataset ostensibly aids the model in constructing a discriminative boundary that proficiently segregates the in-distribution data from its out-of-distribution counterparts. Empirical observations reveal that the actual anomalous data predominantly falls outside of this established boundary, a phenomenon consistently manifested in the experimental results for both ResMLP and MLP models. 
\begin{figure}
% {r}{0.55\textwidth}
        \vspace{-15pt}
        \centering
        \begin{subfigure}{.22\textwidth}
  \centering
  \includegraphics[width=\linewidth]{ 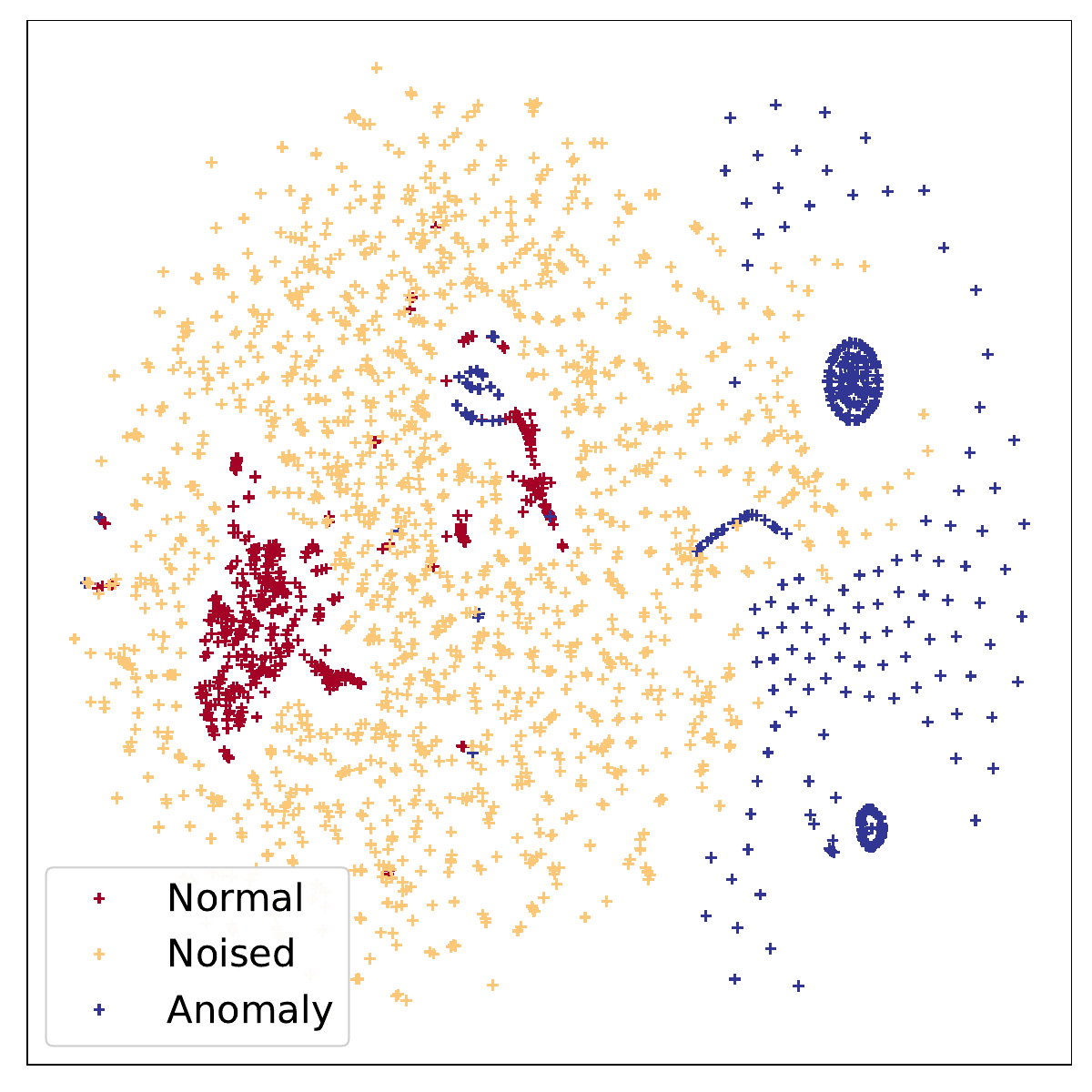}  
  \caption{ResMLP T-SNE Visual.}
\end{subfigure}
\begin{subfigure}{.22\textwidth}
  \centering
  \includegraphics[width=\linewidth]{ 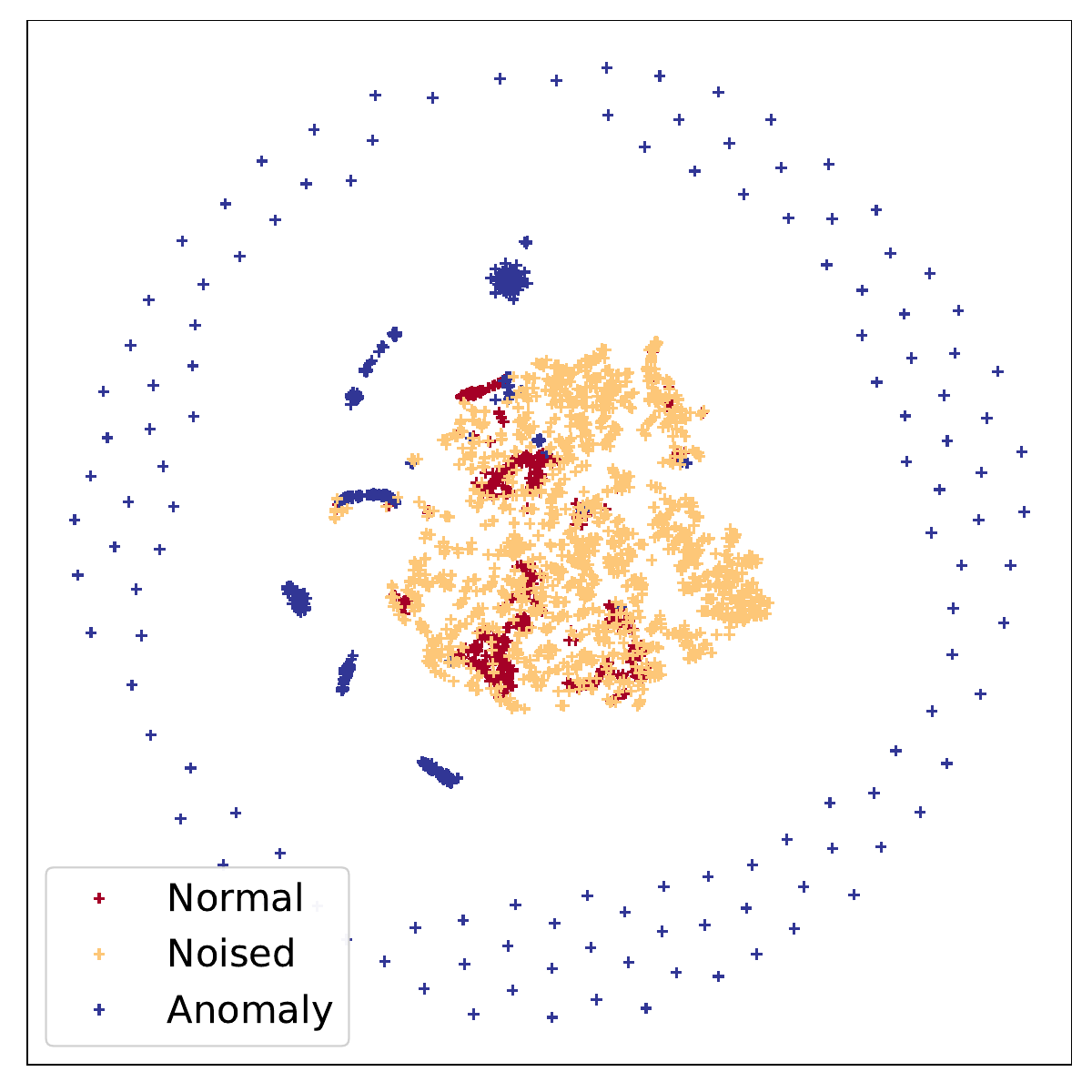}  
  \caption{MLP T-SNE Visualization}
\end{subfigure}
        \caption{
        Qualitative visualization of the penultimate layer on the KDD-CUP99 dataset shows normal (red), noised (yellow), and anomalous (blue) instances. The noised instances help the model establish a decision boundary between in-distribution and out-of-distribution data. Real anomalies consistently fall outside this boundary in both ResMLP and MLP, confirming their effectiveness in anomaly detection.
        % Qualitative visualization of the penultimate layer of the normal (red), noised (yellow), and anomalous (blue) instances on the KDD-CUP99 dataset. We see that the generated noised instances can help the model establish a decision boundary distinguishing in-distribution data and out-of-distribution data. Notably, the real anomalies tend to fall outside the model's decision boundary in both ResMLP and MLP results, affirming their capability in anomaly detection.
        }\label{fig: visualization}
    % \vspace{-20pt}
\end{figure}
% This empirical evidence suggests that the models possess a commendable degree of precision in detecting and isolating anomalies within the data.

\section{Conclusion}

In conclusion, we presented a novel noise evaluation-based method for unsupervised anomaly detection in tabular data. We assume the model can learn the anomalous pattern from the noised normal data. Predicting the magnitude of the noise shows inspiration on how much and where the abnormality is. We theoretically proved the generalizability and reliability of our method. Extensive experiments demonstrated that our approach outperforms other anomaly detection methods through 47 real tabular datasets in the UAD setting and 25 real tabular datasets in the OCC setting. An ablation study suggests that using Gaussian, Rayleigh, and Uniform noise has a stable performance. One potential limitation of this work is that we focused on tabular data only. Nevertheless, it is possible to extend our method to image data via the following two steps: 1) extract features of input images using a pretrained encoder; 2) apply our method to the extracted image features. This could be an interesting future work.

\section{Acknowledgments}
This work was supported by the Shenzhen Science and Technology Program under the Grant No.JCYJ20210324130208022 (Fundamental Algorithms of Natural Language Understanding for Chinese Medical Text Processing) and the General Program of Guangdong Basic and Applied Basic Research under Grant No.2024A1515011771.

\bibliography{aaai25}

% \section*{References}

% References follow the acknowledgments in the camera-ready paper. Use unnumbered first-level heading for
% the references. Any choice of citation style is acceptable as long as you are
% consistent. It is permissible to reduce the font size to \verb+small+ (9 point)
% when listing the references.
% Note that the Reference section does not count towards the page limit.
% \medskip

% {
% \small

% [1] Alexander, J.A.\ \& Mozer, M.C.\ (1995) Template-based algorithms for
% connectionist rule extraction. In G.\ Tesauro, D.S.\ Touretzky and T.K.\ Leen
% (eds.), {\it Advances in Neural Information Processing Systems 7},
% pp.\ 609--616. Cambridge, MA: MIT Press.

% [2] Bower, J.M.\ \& Beeman, D.\ (1995) {\it The Book of GENESIS: Exploring
%   Realistic Neural Models with the GEneral NEural SImulation System.}  New York:
% TELOS/Springer--Verlag.

% [3] Hasselmo, M.E., Schnell, E.\ \& Barkai, E.\ (1995) Dynamics of learning and
% recall at excitatory recurrent synapses and cholinergic modulation in rat
% hippocampal region CA3. {\it Journal of Neuroscience} {\bf 15}(7):5249-5262.
% }

%%%%%%%%%%%%%%%%%%%%%%%%%%%%%%%%%%%%%%%%%%%%%%%%%%%%%%%%%%%%
\clearpage

\appendix

% \title{Unsupervised Anomaly Detection for Tabular Data\\ Using Noise Evaluation}
% \maketitle
\section{Supplementary Material for: \\ Unsupervised Anomaly Detection for Tabular Data Using Noise Evaluation}
\vspace{20pt}
\section{Appendix A: Connection Between Noise Evaluation and Denoising Score Matching}
We first review the denoising score matching \citep{song2019generative, vincent2011connection}. The \textit{score} of a probability density $p(\boldsymbol{x})$ is defined as $\nabla_{\boldsymbol{x}}\log{p(\boldsymbol{x})}$. Score matching train a score network parametrized by $\theta$, $s_{\theta}(\boldsymbol{x})$, to estimate $\nabla_{\boldsymbol{x}}\log{p_{data}(\boldsymbol{x})}$. Denoising score matching is a variant of score matching. It first perturbs the data $\boldsymbol{x}$ with a known noise distribution $q_{\sigma}(\hat{\boldsymbol{x}} | \boldsymbol{x})$ and estimate the score of the perturbed data distribution $q_{\sigma}(\hat{\boldsymbol{x}}) \triangleq \int q_{\sigma}(\hat{\boldsymbol{x}} | \boldsymbol{x}) p_{data}(\boldsymbol{x})\text{d}\boldsymbol{x}$. The learning objective was proved equivalent to the following:
\begin{equation} \label{eq: dsm}
    \min_{\theta} \frac{1}{2} \mathbb{E}_{q_{\sigma}(\hat{\boldsymbol{x}} | \boldsymbol{x})} \|s_{\theta}(\hat{\boldsymbol{x}}) -  \nabla_{\hat{\boldsymbol{x}}}\log{q_{\sigma}(\hat{\boldsymbol{x}} | \boldsymbol{x})}\|_2^2.
\end{equation}

Here we claim optimizing the noise evaluation learning objective is a lower bound of the denoising score matching.

\begin{proof}

Without loss of generality, similar to \citep{vincent2011connection}, we consider the perturbation kernel $q_{\sigma}(\hat{\boldsymbol{x}} | \boldsymbol{x})$ satisfies
\begin{equation*}
    \frac{\partial q_{\sigma}(\hat{\boldsymbol{x}} | \boldsymbol{x})}{\partial \hat{\boldsymbol{x}}} = \Delta (\boldsymbol{x} - \hat{\boldsymbol{x}}), \Delta>0.
\end{equation*}
If the perturbation kernel is Gaussian noise, $\Delta = \frac{1}{\sigma^2}$. 
% Since we utilize diverse noise,  
 
Let us now choose the model $p$ as 
\begin{equation*}
    p_{\theta}(\boldsymbol{x}) = \frac{1}{Z(\theta)} \exp(-f_{\theta}(\boldsymbol{x})),
\end{equation*}
\begin{equation*}
    f_{\theta}(\boldsymbol{x}) = - \Delta\left(\int h_{\theta} (\boldsymbol{x}) \text{d}\boldsymbol{x} \right), h_{\theta} (\boldsymbol{x}) \ge \boldsymbol{0},
\end{equation*}
% \begin{equation*}
%     f_{\theta_1, \theta_2}(\boldsymbol{x}) = - \left(\int h_{\theta_1} (\boldsymbol{x}) \text{d}\boldsymbol{x} + \Delta \langle \theta_2, \boldsymbol{x} \rangle \right),
% \end{equation*}
where $Z(\theta)$ is an intractable partition function and $h$ is our noise evaluation model. 
Then, 
\begin{align*}
    s_{\theta}(\boldsymbol{x}) &= \frac{\partial \log p_{\theta}(\boldsymbol{x})}{\partial \boldsymbol{x}} = -\frac{f_{\theta}(\boldsymbol{x})}{\partial \boldsymbol{x}} = \Delta h_{\theta} (\boldsymbol{x}).
\end{align*}
Since the noised sample in our method is generated as $\hat{\boldsymbol{x}} = \boldsymbol{x} + \boldsymbol{\epsilon}$, then we have.
\begin{align*}
    &\frac{1}{2} \mathbb{E}_{q_{\sigma}(\hat{\boldsymbol{x}} | \boldsymbol{x}) }\|s_{\theta}(\hat{\boldsymbol{x}}) -  \nabla_{\hat{\boldsymbol{x}}}\log{q_{\sigma}(\hat{\boldsymbol{x}} | \boldsymbol{x})}\|_2^2 \\ 
    =&  \frac{1}{2} \Delta \mathbb{E}_{q_{\sigma}(\hat{\boldsymbol{x}} | \boldsymbol{x})} \|h_{\theta}(\hat{\boldsymbol{x}}) - (\boldsymbol{x} - \hat{\boldsymbol{x}}) \|_2^2 \\
    \ge&  \frac{1}{2} \Delta \mathbb{E}_{q_{\sigma}(\hat{\boldsymbol{x}} | \boldsymbol{x})} \|h_{\theta}(\hat{\boldsymbol{x}}) - |\boldsymbol{x} - \hat{\boldsymbol{x}}| \|_2^2 \\
    =&  \frac{1}{2} \Delta \mathbb{E}_{q_{\sigma}(\hat{\boldsymbol{x}} | \boldsymbol{x})} \|h_{\theta}(\hat{\boldsymbol{x}}) - |\epsilon| \|_2^2
\end{align*}
For the input data from the normal class, the noise magnitude is naturally $\boldsymbol{0}$. According to \citep{song2019generative}, setting $\boldsymbol{\epsilon} = \boldsymbol{0}$ in the learning objective makes the model can learn the underlying data distribution, i.e., $s_{\theta}(\boldsymbol{x}) = \nabla_{\boldsymbol{x}}\log{q_{\sigma}(\boldsymbol{x}} ) \approx \nabla_{\boldsymbol{x}}\log{p_{data}(\boldsymbol{x}})$. Hence, we show that the noise evaluation learning
objective is a lower bound of the denoising score matching.
\end{proof}
Since the goal of our method is anomaly detection instead of sample generation, predicting the magnitude of the noise is enough to detect samples far from the $p_{data}(\boldsymbol{x})$ which are usually anomalies. For samples belonging to the normal class, it is natural that the gradient of its probability density is 0. We provide an intuitive example here with Gaussian distribution. Suppose $x \sim \mathcal{N}(0, \sigma)$, its gradient of log density function $\frac{d}{dx}\log{p(x}) = -\frac{1}{\sigma^2} x$. We see if the magnitude of the gradient is large, the sample is far away from the distribution center $0$.  
In our method, we involve diverse noise with different variances because $p_{data}(\boldsymbol{x})$ is unknown and $\frac{\partial q_{\sigma}(\hat{\boldsymbol{x}} | \boldsymbol{x}_i)}{\partial \hat{\boldsymbol{x}}}$ and $\frac{\partial q_{\sigma}(\hat{\boldsymbol{x}} | \boldsymbol{x}_j)}{\partial \hat{\boldsymbol{x}}}$ maybe very different for the same $\hat{\boldsymbol{x}}$ with different $i, j$. Since we want to enlarge such gradients, we use various noise levels. It is similar to \citep{song2019generative, song2020improved} using an annealed noise level.
Therefore, noise evaluation can detect anomaly data from the perspective of score matching.

\section{Appendix B: Noise Evaluation Training Algorithm} \label{app: training alg}
\begin{algorithm}[H]
\caption{Noise Evaluation Training} \label{alg: training}
\begin{algorithmic}[1]
\Require Training dataset $\boldsymbol{\mathcal{X}}$, max. iterations $T$;
\State Normalize $\boldsymbol{\mathcal{X}}$ to have mean of $0$ and standard deviation of $1$;
\State Initialize model $\boldsymbol{\theta}$;
\State $t \gets 1$;
\While{$t \leq T$}
    \For{each batch of data $\boldsymbol{\mathcal{B}} \in \boldsymbol{\mathcal{X}}$}
        \State  \textit{//get batch size and dimension.}
        \State $b, d \gets \boldsymbol{\mathcal{B}}.shape$;
        \State \textit{//Algorithm \ref{alg: noise generation}.}
        \State $\boldsymbol{\mathcal{E}} \gets$ \texttt{NoiseGeneration}$(b, d)$; \label{alg: noise line}
        \State $\hat{\boldsymbol{\mathcal{B}}} \gets \boldsymbol{\mathcal{B}} + \boldsymbol{\mathcal{E}}$; 
        % \State \textit{//Equition \eqref{eq: final objective}.}
        \State $l \gets \sum_{i}\left \| h_{\theta}(\boldsymbol{x}_i)\right \|_{2}^2 + \sum_{i}\left \| h_{\theta}(\hat{\boldsymbol{x}}_i) - \left|\boldsymbol{\epsilon}_i\right| \right \|_{2}^2, \boldsymbol{x}_i\in \boldsymbol{\mathcal{B}}, \hat{\boldsymbol{x}}_i\in \hat{\boldsymbol{\mathcal{B}}}, \boldsymbol{\epsilon}_i \in \boldsymbol{\mathcal{E}}$; \textit{//Equition \eqref{eq: final objective}.}
        \State Optimize $\boldsymbol{\theta}$ by Adam \citep{kingma2014adam, reddi2018convergence} optimizer;
    \EndFor
    \State $t \gets t + 1$;
\EndWhile
% \State \Return $\boldsymbol{\theta}$;
\Ensure Optimized model $h_{\boldsymbol{\theta}}$;
\end{algorithmic}
\end{algorithm}
The overall training algorithm for the noise evaluation training is in Algorithm \ref{alg: training}. Notice that for each training epoch, we randomly generate a new noised instance for each training sample. This helps us enlarge the sampling number from the noise distribution for better learning. In line \ref{alg: noise line}, there are some optional noise generation schemes such as using different noise types, different noise levels, and different noise ratios. 

\paragraph{Time and Space Complexity} Let $b$ denote the batch size, $w_{\max}$ represent the maximum width of the hidden layers in an $L$-layer neural network, and $d$ indicate the dimension of the input data. Under these parameters, the time complexity of the proposed methods is at most $\mathcal{O}(bdw_{\max}LT)$, where $T$ is the maximum number of iterations. The space complexity is at most $\mathcal{O}(bd + dw_{\max} + (L-1)w_{\max}^2)$. These complexities scale linearly with the number of samples, which underscores the scalability of the proposed methods to large datasets. Moreover, in the case of high-dimensional data (i.e., large $d$), choosing a smaller $w_{\max}$ can further enhance computational efficiency.

\section{Appendix C: Proof of Theoretical Analysis} \label{app: proof}

\subsection{Proof for Proposition \ref{prop_1}}\label{app: proof_prop1}
\begin{proof}
For convenience, denote the noise variable as $E$, then $\hat{X}=X+E$. We have
\begin{equation}
\begin{aligned} &H(\hat{X} \mid X) \\
=&-\sum_x p(X=x) \sum_{\hat{x}} p(\hat{X}=\hat{x} \mid X=x) \log p(\hat{X}=\hat{x} \mid X=x) \\
=&-\sum_x p(X=x) \\
&\sum_{\hat{x}} p(E=\hat{x}-x \mid X=x) \log p(E=\hat{x}-x \mid X=x) \\ 
=&-\sum_x p(X=x) \sum_{e^{\prime}} p\left(E=e^{\prime} \mid X=x\right) \log p\left(E=e^{\prime} \mid X=x\right) \\
=&H(E \mid X)
\end{aligned}
\end{equation}
Since $E$ and $X$ are independent, we have $H(E \mid X)=H(E)$. On the other hand, we have $H(\hat{X})\geq H(\hat{X}|X)$, which can be proved by the definition of conditional entropy and Jensen's inequality (or log sum inequality more specifically). Then we arrive at
    $$H(\hat{X})\geq H(\hat{X}\mid X)=H(E\mid X)=H(E).$$
Similarly, by symmetry, we have
$$H(\hat{X})\geq H(\hat{X}\mid E)=H(X\mid E)=H(X).$$
This inequality holds for each variable in $\boldsymbol{x}$. Then the total entropy of the $d$ variables increases, which means the data becomes more disordered.
\end{proof}

\subsection{Definition of $\mathcal{H}$-divergence}\label{app_H_div}
\begin{definition}
    (Based on Kifer et al., 2004 \citep{kifer2004detecting}) Given a domain $\mathcal{X}$ with $D$ and $D'$ probability distributions over $\mathcal{X}$, let $\mathcal{H}$ be a hypothesis class on $\mathcal{X}$ and denote by $I(h)$ the set for which $h \in \mathcal{H}$ is the characteristic function; that is, $x \in I(h) \iff h(x) = 1$. The $\mathcal{H}$-divergence between $D$ and $D'$ is
\[
d_{\mathcal{H}}(D, D') = 2 \sup_{h \in \mathcal{H}} \left| \Pr_{D} [I(h)] - \Pr_{D'} [I(h)] \right|.
\]
\end{definition}
The definition is defined in \citep{ben2010theory}. We repeat here to provide reader convenience. In addition, $\mathcal{H}\Delta\mathcal{H}$ is a symmetric difference hypothesis space, where $f\in\mathcal{H}\Delta\mathcal{H}, f(\boldsymbol{x})=I(g(h(\boldsymbol{x}))>\tau)~XOR~I(g(h'(\boldsymbol{x}))>\tau)$ for some $h, h'\in \mathcal{H}$.

\subsection{Proof of Theorem \ref{theorem: gap}}\label{app: proof_eed}
Before proving the theorem, we provide the following necessary lemma according to \citep{ben2010theory, bartlett2019nearly}.
\begin{lemma}
    Let $\mathcal{H}$ be a hypothesis space on $\mathbb{R}^d$ with VC dimension $d_{vc}$. $h\in\mathcal{H}$ is a ReLu-activated deep neural network with $L$ layers, and the number of neurons at each layer is in the order of $p$.  If $\mathcal{X}$ and $\mathcal{X}'$ 
    are samples of size $N$ from two distributions $\mathcal{D}$ and $\mathcal{D}'$ respectively and $\hat{d}_{\mathcal{H}}(\mathcal{X}, \mathcal{X}')$ is the empirical $\mathcal{H}$-divergence between samples, then for any $\delta \in (0, 1)$, with probability at least $1 - \delta$,
$$
d_{\mathcal{H}}(\mathcal{D}, \mathcal{D}') \leq \hat{d}_{\mathcal{H}}(\mathcal{X}, \mathcal{X}') + 4 \sqrt{\frac{d_{vc} \log(2N) + \log(\frac{2}{\delta})}{N}},
$$
where $d_{vc}=\mathcal{O}(p L \log (p L))$.
\end{lemma}
% \subsection{Proof of Lemma 1}\label{proof_lemma1}
\begin{proof}
    We slightly modify the lemma in \citep{ben2010theory} and Theorem 3.4 in \citep{kifer2004detecting} with explicit definition of our neural network and specified VC dimension $\mathcal{O}(pL\log(pL))$.
\end{proof}

Following \citep{ben2010theory}, we have an additional lemma.
\begin{lemma}
For any hypotheses $h, h' \in \mathcal{H}$,
    \[
\left| \varepsilon(h, h') - \varepsilon(h, h') \right| \leq \frac{1}{2} d_{\mathcal{H}\Delta\mathcal{H}}(\mathcal{D}, \mathcal{D}').
\]
\end{lemma}
\begin{proof}
    By the definition of $\mathcal{H}\Delta\mathcal{H}$-distance,

\begin{align*}
&d_{\mathcal{H}\Delta\mathcal{H}}(\mathcal{D}, \mathcal{D}') \\
&= 2 \sup_{h,h' \in \mathcal{H}} \left| \Pr_{x \sim \mathcal{D}} [h(x) \neq h'(x)] - \Pr_{x \sim \mathcal{D}'} [h(x) \neq h'(x)] \right| \\
&= 2 \sup_{h,h' \in \mathcal{H}} \left| \varepsilon(h, h') - \varepsilon'(h, h') \right| \geq 2\left| \varepsilon(h, h') - \varepsilon'(h, h') \right|.
\end{align*}
\end{proof}

Then, we have the proof.
\begin{proof}
Define 
$$\varepsilon_{\mathcal{D}}(h, h^*) := \mathbb{E}_{\boldsymbol{x}\sim\mathcal{D}}\left[\left|I\left(g\left(h_{\theta}(x)\right) >\tau\right) - I\left(g\left(h^*_{\theta}(x)\right) >\tau\right)\right|\right].$$
We have
\begin{align*}
    \varepsilon_{\tilde{\mathcal{D}}_{H}}(h) &\le \varepsilon_{\tilde{\mathcal{D}}_{H}}(h^*) + \varepsilon_{\tilde{\mathcal{D}}_{H}}(h, h^*) \\
    &\le \varepsilon_{\tilde{\mathcal{D}}_{H}}(h^*) + \varepsilon_{\hat{\mathcal{D}}}(h, h^*) + |\varepsilon_{\tilde{\mathcal{D}}_{H}}(h, h^*) - \varepsilon_{\hat{\mathcal{D}}}(h, h^*)| \\
    &\le \varepsilon_{\tilde{\mathcal{D}}_{H}}(h^*) + \varepsilon_{\hat{\mathcal{D}}}(h, h^*) + \frac{1}{2}d_{\mathcal{H}\Delta\mathcal{H}}(\hat{\mathcal{D}}, \tilde{\mathcal{D}}_{H})\\
    &\le \varepsilon_{\tilde{\mathcal{D}}_{H}}(h^*) + \varepsilon_{\hat{\mathcal{D}}}(h) + \varepsilon_{\hat{\mathcal{D}}}(h^*) + \frac{1}{2}d_{\mathcal{H}\Delta\mathcal{H}}(\hat{\mathcal{D}}, \tilde{\mathcal{D}}_{H})\\
    &= \varepsilon_{\hat{\mathcal{D}}}(h) + \frac{1}{2}d_{\mathcal{H}\Delta\mathcal{H}}(\hat{\mathcal{D}}, \tilde{\mathcal{D}}_{H}) + \lambda \\
    &\le \varepsilon_{\hat{\mathcal{D}}}(h) + \frac{1}{2} \hat{d}_{\mathcal{H}\Delta\mathcal{H}}(\tilde{\mathcal{X}}_{H}, \hat{\mathcal{X}}) + 4 \sqrt{\frac{2d_{vc} \log(2N) + \log(\frac{2}{\delta})}{N}} + \lambda.
\end{align*}
The VD dimension of $\mathcal{H}\Delta\mathcal{H}$ is at most twice the VC dimension of $\mathcal{H}$ \citep{anthony1999neural}.
\end{proof}

\subsection{Proof of Theorem \ref{theorem_dmin}}\label{app: proof_dmin}

\begin{proof}
Assumption \ref{ass: gap} indicates that
$$c\|{\boldsymbol{x}} - \tilde{\boldsymbol{x}} \| \le | g(h_{\boldsymbol{\theta}}({\boldsymbol{x}}))- g(h_{\boldsymbol{\theta}}(\tilde{\boldsymbol{x}}))|.$$
According to the definition of $d_{\min}$, we have
\begin{equation}\label{eq_ghcxx_0}
g(h_{\boldsymbol{\theta}}(\tilde{\boldsymbol{x}}))\geq cd_{\min}+g(h_{\boldsymbol{\theta}}({\boldsymbol{x}}))
\end{equation}
or 
\begin{equation}\label{eq_ghcxx}
g(h_{\boldsymbol{\theta}}(\tilde{\boldsymbol{x}}))\leq g(h_{\boldsymbol{\theta}}({\boldsymbol{x}}))-cd_{\min}.
\end{equation}
Since $\epsilon<cd_{\min}$, we have $g(h_{\boldsymbol{\theta}}(\tilde{\boldsymbol{x}}))<0$, which contradicts the fact that $g(h_{\boldsymbol{\theta}}(\tilde{\boldsymbol{x}}))$ is always positive. Therefore, \eqref{eq_ghcxx} will not happen and we will only have \eqref{eq_ghcxx_0}. It follows that
\begin{equation}\label{eq_ghcxx_1}
g(h_{\boldsymbol{\theta}}(\tilde{\boldsymbol{x}}))\geq cd_{\min}.
\end{equation}
This means if $cd_{\min}>\tau$, $\tilde{\boldsymbol{x}}$ can be detected successfully.
Putting the conditions $\epsilon<cd_{\min}$ and $cd_{\min}>\tau$, $\tilde{\boldsymbol{x}}$ together, we complete the proof.
\end{proof}

\begin{table}[H]
\centering
\begin{tabular}{l|cccc} 
% \hline
\toprule
\textbf{Dataset} & \textbf{\# Sample} & \textbf{Dims.} & \textbf{\# Class} & \textbf{\% Anomaly} \\ 
\midrule
abalone & 4177 & 8 & 28 & 94.19 \\ 
arrhythmia & 452 & 280 & 2 & 45.80 \\ 
breastw & 699 & 10 & 2 & 34.48 \\ 
cardio & 2126 & 22 & 2 & 22.15 \\ 
ecoli & 336 & 8 & 8 & 77.16 \\ 
glass & 214 & 10 & 6 & 72.74 \\ 
ionosphere & 351 & 35 & 2 & 35.90 \\ 
kdd & 5209460 & 39 & 2 & 80.52 \\ 
letter & 20000 & 17 & 26 & 96.15 \\ 
lympho & 142 & 19 & 2 & 42.96 \\ 
mammography & 11183 & 7 & 2 & 2.32 \\ 
mulcross & 262144 & 5 & 2 & 10.00 \\ 
musk & 7074 & 167 & 2 & 17.30 \\
optdigits & 5620 & 65 & 10 & 90.00 \\ 
pendigits & 10992 & 17 & 10 & 90.00 \\ 
pima & 768 & 9 & 2 & 34.90 \\ 
satimage & 6435 & 37 & 6 & 83.33 \\ 
seismic & 2584 & 15 & 2 & 6.58 \\ 
shuttle & 58000 & 10 & 7 & 75.03 \\ 
speech & 3686 & 401 & 2 & 1.65 \\ 
thyroid & 7200 & 7 & 2 & 7.29 \\ 
vertebral & 310 & 7 & 2 & 67.74 \\ 
vowels & 1456 & 13 & 2 & 3.43 \\ 
wbc & 569 & 31 & 2 & 37.26 \\ 
wine & 178 & 14 & 3 & 66.67 \\
\bottomrule
\end{tabular}
\caption{25 real-world tabular datasets tested in this paper under OCC setting. For datasets with more than 2 labeled classes, an average anomaly ratio is recorded. Kdd refers to KDD-CUP99 \citep{kdd}.}
\label{tab: dataset}
\end{table}

\section{Appendix D: Dataset Summary} \label{app: dataset}
We conducted our experiment under the UAD dataset setting with 47 benchmark datasets commonly used in UAD research. The dataset information is shown in Table \ref{tab: adbench dataset}. As for the OCC dataset setting. The detailed dataset information including the number of samples, dimensionality, the number of classes, and the percent of anomaly is summarized in Table \ref{tab: dataset}. Note that there are datasets with the same name. However, due to the different dataset processing and training data splitting. They are completely different under OCC setting from the UAD setting.

\begin{table}[H]
\centering
\begin{tabular}{l|ccc}
\toprule
\textbf{Dataset} & \textbf{\# Sample} & \textbf{Dim.} & \textbf{\% Anomaly} \\
\midrule
ALOI & 49534 & 27 & 3.04 \\
anthyroid & 7200 & 6 & 7.42 \\
backdoor & 95329 & 196 & 2.44 \\
breast & 683 & 9 & 34.99 \\
campaign & 41188 & 62 & 11.27 \\
cardio & 1831 & 21 & 9.61 \\
Cardiotocography & 2114 & 21 & 22.04 \\
celeba & 202599 & 39 & 2.24 \\
census & 299285 & 500 & 0.60 \\
cover & 286048 & 10 & 6.20 \\
donors & 619326 & 10 & 5.93 \\
fault & 1941 & 27 & 34.67 \\
fraud & 284807 & 29 & 0.17 \\
glass & 214 & 9 & 42.21 \\
Hepatitis & 80 & 17 & 16.25 \\
http & 567498 & 3 & 0.39 \\
InternetAds & 1966 & 1555 & 18.72 \\
Ionosphere & 351 & 32 & 35.90 \\
landsat & 6435 & 36 & 20.71 \\
letter & 1600 & 32 & 15.10 \\
Lymphography & 148 & 18 & 4.05 \\
magic & 19020 & 10 & 35.16 \\
mammography & 11183 & 6 & 2.32 \\
mnist & 7603 & 100 & 9.21 \\
musk & 3062 & 166 & 3.17 \\
optdigits & 5216 & 64 & 2.88 \\
PageBlocks & 5393 & 10 & 9.46 \\
pendigits & 6870 & 16 & 2.27 \\
Pima & 768 & 8 & 34.90 \\
satellite & 6435 & 36 & 31.64 \\
satimage-2 & 5803 & 36 & 1.22 \\
shuttle & 49097 & 9 & 7.15 \\
skin & 245057 & 3 & 20.75 \\
smtp & 95156 & 3 & 0.03 \\
SpamBase & 4207 & 57 & 39.91 \\
speech & 3686 & 400 & 1.65 \\
Stamps & 340 & 9 & 9.12 \\
thyroid & 3772 & 6 & 2.47 \\
vertebral & 240 & 6 & 12.50 \\
vowels & 1456 & 12 & 3.43 \\
Waveform & 3443 & 21 & 2.90 \\
WBC & 223 & 9 & 4.48 \\
WDBC & 367 & 30 & 2.72 \\
Wilt & 4819 & 5 & 5.33 \\
wine & 129 & 13 & 7.75 \\
WPBC & 198 & 33 & 23.74 \\
yeast & 1484 & 8 & 34.16 \\
\bottomrule
\end{tabular}
\caption{47 real-world tabular datasets tested under UAD setting.}
\label{tab: adbench dataset}
\end{table}

\section{Appendix E: Neural Network Architecture} \label{app: network}
\subsection{VanillaMLP}
The VanillaMLP is a plain Relu-activated feed-forward neural network with 4 fully connected layers. The architecture is shown in Figure \ref{fig: mlp arch}. The input dimension is $d$. If $d\le64$, the size of the hidden layer is set to $64$. If $d>64$, the size of the hidden layer is set to $256$. In the experiments, datasets, arrhythmia, musk, and speech, are with $d>64$.

\begin{figure}[h]
\begin{subfigure}{.37\linewidth}
  \centering
  \includegraphics[width=\linewidth]{ 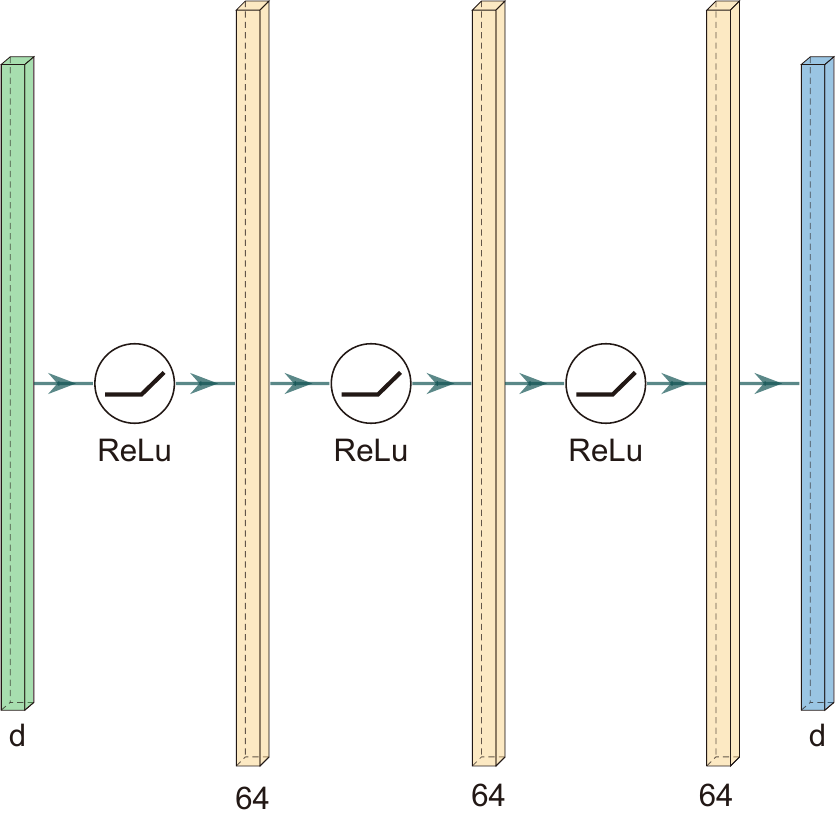}  
  \caption{MLP Architecture}\label{fig: mlp arch}
\end{subfigure}
\quad\quad\quad
\begin{subfigure}{.47\linewidth}
  \centering
  \includegraphics[width=\linewidth]{ 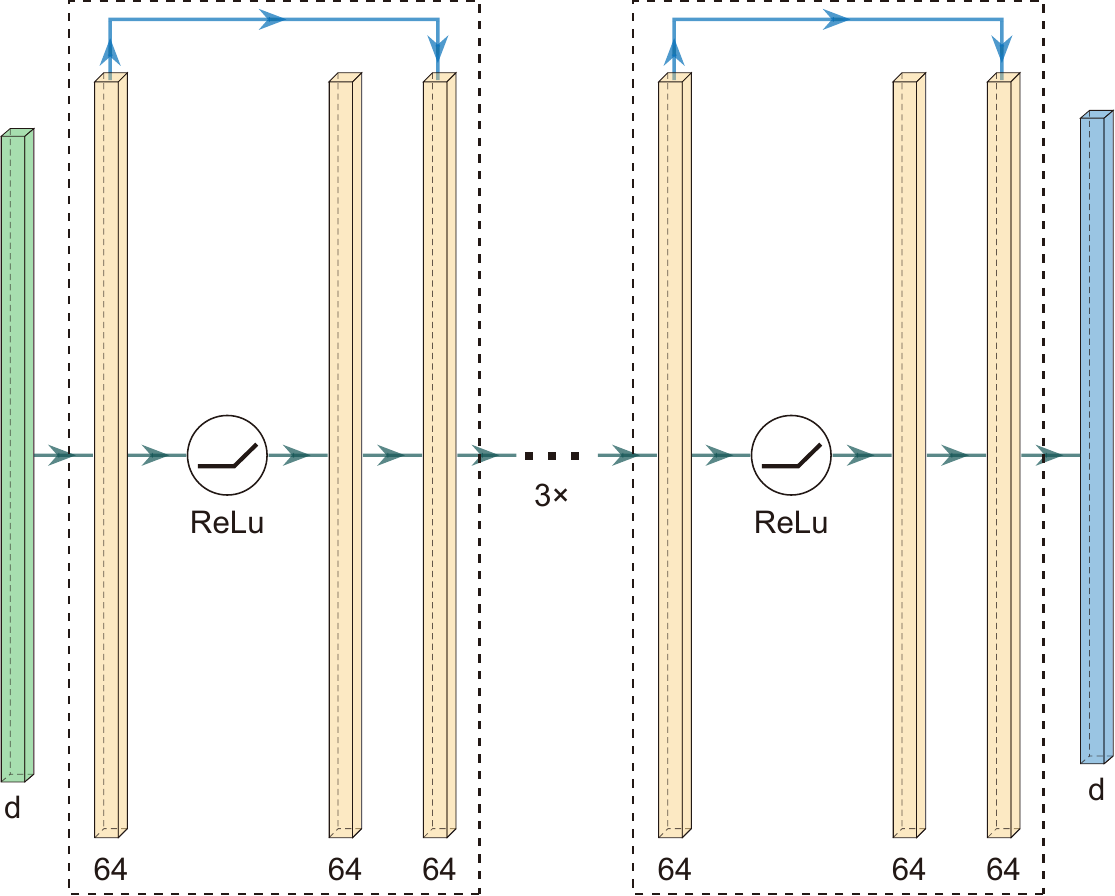}  
  \caption{ResMLP Architecture}\label{fig: resmlp arch}
\end{subfigure}
\caption{Deep Neural Network Architectures.}
\label{fig: arch}
\end{figure}

\subsection{ResMLP}
The ResMLP is a Relu-activated feed-forward neural network with 5 residual blocks. The architecture is shown in Figure \ref{fig: resmlp arch}. The input dimension is $d$. If $d\le64$, the size of the hidden layer is set to $64$. If $d>64$, the size of the hidden layer is set to $256$. In the experiments, datasets, arrhythmia, musk, and speech, are with $d>64$.

% \begin{figure}[h]
%   \centering
%   \includegraphics[width=0.7\linewidth]{ resmlp_arch.pdf}  
%   \caption{ResMLP Architecture} \label{fig: resmlp arch}
% \end{figure}

\begin{table*}[ht]
\centering
\setlength\tabcolsep{1.5pt}
\resizebox{\textwidth}{!}{
\begin{tabular}{l|ccccccccccc|cc}
\toprule\hline
         \textbf{Methods} &   \textbf{DPAD} &  \textbf{PLAD}&  \textbf{NeuTraLAD}    &  \textbf{SCAD} &           \textbf{AE} &        \textbf{COPOD} &     \textbf{DeepSVDD} &         \textbf{ECOD} &      \textbf{IForest} &          \textbf{KNN} &         \textbf{LOF} &       \makecell[c]{\textbf{Ours-} \\\textbf{ResMLP}} & \makecell[c]{\textbf{Ours-} \\\textbf{MLP}} \\
\midrule
            ALOI &  64.48$\pm$0.4 & 53.16$\pm$1.4  &    53.54$\pm$0.9 & 53.97$\pm$0.5 &   54.96$\pm$0.1 &   51.51$\pm$0.1 &   51.44$\pm$1.3 &   53.07$\pm$0.1 &   53.78$\pm$0.4 &   \underline{66.72}$\pm$0.2 &  \textbf{74.69}$\pm$0.2 &   60.59$\pm$0.7 &   59.95$\pm$0.7 \\
      annthyroid &  90.91$\pm$2.0 & 95.09$\pm$2.8    &  81.18$\pm$2.6 & 86.52$\pm$2.2 &   84.66$\pm$1.5 &   77.59$\pm$0.2 &   56.97$\pm$2.6 &   79.11$\pm$0.2 &   91.39$\pm$1.4 &    93.6$\pm$0.4 &  \underline{94.78}$\pm$0.3 &   \textbf{97.02}$\pm$0.5 &   \textbf{96.78}$\pm$0.3 \\
        backdoor & 95.43$\pm$0.7 & 91.18$\pm$13.3   &   94.20$\pm$0.7 & 93.37$\pm$0.5 &   90.92$\pm$0.0 &   78.93$\pm$0.1 &   93.04$\pm$1.9 &   84.54$\pm$0.1 &   76.77$\pm$1.9 &   95.65$\pm$0.0 &  \underline{96.45}$\pm$0.1 &   96.18$\pm$0.9 &   \textbf{97.12}$\pm$0.4 \\
         breastw &  99.11$\pm$0.7  &74.02$\pm$2.7   &   78.93$\pm$3.9 &98.04$\pm$0.8 &   98.91$\pm$0.4 &   \textbf{99.56}$\pm$0.1 &   96.11$\pm$0.7 &   99.15$\pm$0.2 &   \underline{99.47}$\pm$0.1 &   98.96$\pm$0.3 &  96.05$\pm$1.1 &   98.99$\pm$0.4 &   99.14$\pm$0.5 \\
        campaign &  75.07$\pm$1.7&  72.31$\pm$3.3  &    70.88$\pm$4.2 &\textbf{79.31}$\pm$0.7 &   76.94$\pm$0.1 &   78.32$\pm$0.1 &  56.27$\pm$11.7 &   76.94$\pm$0.2 &   73.54$\pm$0.8 &    \underline{78.4}$\pm$0.1 &  70.37$\pm$0.8 &   69.33$\pm$1.4 &    74.6$\pm$1.8 \\
          cardio & 91.80$\pm$3.1 & 61.16$\pm$7.2   &   72.38$\pm$2.9 & 89.57$\pm$1.3 &   \underline{96.25}$\pm$0.4 &   92.39$\pm$0.4 &   59.95$\pm$5.2 &   93.44$\pm$0.2 &    94.6$\pm$1.0 &   92.53$\pm$0.9 &  92.49$\pm$1.3 &   \textbf{97.06}$\pm$0.9 &   \textbf{96.66}$\pm$0.7 \\
Cardio. &  76.92$\pm$2.2 & 65.39$\pm$6.9   &   54.98$\pm$4.6 & 69.54$\pm$2.2 &   \underline{83.35}$\pm$1.0 &   66.37$\pm$0.9 &  49.21$\pm$11.4 &   78.49$\pm$0.7 &   80.16$\pm$1.2 &    72.3$\pm$1.3 &  76.96$\pm$1.2 &   \textbf{85.04}$\pm$3.2 &   \underline{84.81}$\pm$3.1 \\
          celeba & 63.50$\pm$2.7 & \textbf{83.47}$\pm$2.5&      58.15$\pm$9.2 & 72.91$\pm$1.9 &    \underline{79.8}$\pm$0.1 &    {75.1}$\pm$0.0 &  54.24$\pm$22.4 &    75.7$\pm$0.1 &   70.97$\pm$1.4 &   68.26$\pm$0.2 &  45.77$\pm$0.3 &   59.21$\pm$2.8 &   67.25$\pm$2.7 \\
          census &  50.03$\pm$0.0 & 55.20$\pm$4.1  &    53.46$\pm$5.6 & 70.54$\pm$0.5 &   \underline{70.74}$\pm$0.1 &    67.4$\pm$0.1 &   52.08$\pm$7.8 &   65.92$\pm$0.1 &   63.43$\pm$2.2 &   \textbf{72.06}$\pm$0.1 &  60.49$\pm$0.2 &   65.63$\pm$2.7 &   68.45$\pm$2.7 \\
           cover &  60.42$\pm$10.8 &  87.66$\pm$0.5    &  71.61$\pm$10.4 & 96.68$\pm$1.2 &   95.53$\pm$0.9 &   88.41$\pm$0.0 &   36.57$\pm$5.8 &   92.01$\pm$0.1 &   85.07$\pm$3.1 &    98.6$\pm$0.1 &  \underline{98.92}$\pm$0.2 &   \textbf{99.69}$\pm$0.1 &   \textbf{99.64}$\pm$0.1 \\
          donors &   95.07$\pm$6.6 & 93.46$\pm$2.9   &   \textbf{99.95}$\pm$0.0 & \textbf{99.95}$\pm$0.0 &   85.79$\pm$3.0 &   81.49$\pm$0.0 &  63.88$\pm$25.7 &   88.85$\pm$0.0 &   89.62$\pm$2.7 &    \underline{99.9}$\pm$0.0 &  98.42$\pm$0.0 &   99.78$\pm$0.0 &    99.6$\pm$0.2 \\
           fault & 73.06$\pm$1.0 & 61.51$\pm$4.8 &     69.39$\pm$1.3 & \underline{79.27}$\pm$1.0 &   55.26$\pm$1.2 &   45.13$\pm$0.8 &    50.6$\pm$6.2 &   46.79$\pm$0.6 &   65.84$\pm$1.6 &   \textbf{79.83}$\pm$1.1 &  65.91$\pm$1.7 &   72.27$\pm$1.3 &   73.49$\pm$1.2 \\
           fraud &  89.73$\pm$5.6  & 96.29$\pm$0.7  &    91.59$\pm$3.5 & 94.71$\pm$0.4 &   \underline{95.38}$\pm$0.0 &   94.74$\pm$0.0 &   85.15$\pm$6.7 &   94.97$\pm$0.0 &   94.88$\pm$0.4 &   \textbf{96.17}$\pm$0.1 &  77.85$\pm$3.4 &   94.44$\pm$0.3 &   95.11$\pm$0.3 \\
           glass &  86.34 $\pm$2.6 & 85.76$\pm$8.1 &     89.17$\pm$5.0 & \underline{92.23}$\pm$2.5 &   74.21$\pm$1.7 &   74.18$\pm$1.0 &   87.86$\pm$3.3 &   71.87$\pm$4.0 &   79.79$\pm$3.9 &   86.27$\pm$3.3 &  75.53$\pm$4.7 &   88.48$\pm$3.4 &   \textbf{94.07}$\pm$1.9 \\
       Hepatitis &  77.67$\pm$8.9 & 73.08$\pm$3.9   &   51.70$\pm$11.2  & 69.7$\pm$6.2 &   79.63$\pm$4.6 &   80.89$\pm$2.4 &   68.44$\pm$6.0 &   74.22$\pm$2.6 &   79.02$\pm$5.8 &   \underline{82.44}$\pm$3.1 &  80.93$\pm$3.2 &   81.73$\pm$3.0 &   \textbf{87.75}$\pm$2.3 \\
            http &   \underline{99.98}$\pm$0.0  & \textbf{99.99}$\pm$0.2  &    79.82$\pm$32.1 & 99.86$\pm$0.3 &   {99.94}$\pm$0.0 &   99.15$\pm$0.0 &  21.43$\pm$43.1 &   97.86$\pm$0.0 &   99.14$\pm$0.6 &   99.96$\pm$0.0 &   93.8$\pm$0.4 &   \textbf{99.99}$\pm$0.0 &   \textbf{99.99}$\pm$0.0 \\
     InternetAds & \underline{89.99}$\pm$0.9 &  87.44$\pm$1.5 &     88.94$\pm$1.0 & 88.76$\pm$1.4 &   87.84$\pm$0.3 &   67.92$\pm$0.7 &   89.91$\pm$0.4 &   67.63$\pm$0.5 &   44.39$\pm$2.1 &   88.41$\pm$0.7 &  {88.79}$\pm$0.6 &   \textbf{92.49}$\pm$0.7 &   \textbf{92.36}$\pm$0.7 \\
      Ionosphere &94.55$\pm$1.6 &  52.05$\pm$9.3 &    85.68$\pm$2.5 &  96.62$\pm$0.6 &   90.99$\pm$1.1 &    78.6$\pm$1.8 &   97.21$\pm$0.5 &   72.61$\pm$1.4 &   88.49$\pm$2.3 &   \underline{97.42}$\pm$0.7 &  94.86$\pm$1.7 &   \textbf{97.51}$\pm$0.5 &   \textbf{97.56}$\pm$0.4 \\
         landsat &  \underline{74.67}$\pm$2.2 &  68.35$\pm$5.6   &   72.21$\pm$7.2 & 73.75$\pm$0.7 &   41.89$\pm$0.7 &   42.19$\pm$0.4 &   65.49$\pm$2.3 &   36.86$\pm$0.6 &   62.43$\pm$2.8 &   \textbf{76.56}$\pm$0.2 &  {74.33}$\pm$1.1 &   58.79$\pm$3.3 &   62.16$\pm$6.4 \\
          letter &  72.34$\pm$2.2 & 70.19$\pm$5.5  &    \underline{91.73}$\pm$1.3 & \textbf{94.36}$\pm$0.8 &   52.36$\pm$0.4 &   55.53$\pm$0.9 &   62.58$\pm$2.8 &   57.18$\pm$0.7 &   62.87$\pm$2.4 &   88.12$\pm$0.9 &  86.85$\pm$1.2 &   {90.87}$\pm$0.8 &   90.71$\pm$1.2 \\
    Lympho. & 98.71$\pm$0.2 & 71.43$\pm$38.5 &     44.48$\pm$9.6 & 99.37$\pm$0.9 &   98.36$\pm$2.4 &   \textbf{99.73}$\pm$0.4 &   97.95$\pm$1.5 &   \underline{99.62}$\pm$0.4 &    99.2$\pm$0.9 &   98.41$\pm$2.5 &  99.19$\pm$0.7 &   98.95$\pm$1.9 &   98.64$\pm$1.9 \\
     magic. &  77.94$\pm$4.6  & 78.46$\pm$2.7  &   71.09$\pm$2.0 & 80.05$\pm$0.3 &   70.18$\pm$0.9 &   68.24$\pm$0.3 &   63.49$\pm$0.6 &   63.55$\pm$0.2 &   77.16$\pm$1.1 &   \underline{83.31}$\pm$0.3 &  83.29$\pm$0.2 &   \textbf{88.05}$\pm$0.4 &   \textbf{88.59}$\pm$0.4 \\
     mammo. &  87.06$\pm$2.1 &  80.92$\pm$2.4  &    67.43$\pm$1.9 & 78.79$\pm$1.6 &   85.45$\pm$1.0 &   \underline{90.61}$\pm$0.1 &  64.47$\pm$11.1 &   \textbf{90.68}$\pm$0.1 &   87.99$\pm$0.7 &   87.66$\pm$0.2 &   82.9$\pm$1.0 &   89.94$\pm$0.4 &   90.25$\pm$0.8 \\
           mnist &  85.55$\pm$4.6 &  81.5$\pm$3.5  &    84.10$\pm$2.0 & 90.89$\pm$0.8 &   90.69$\pm$0.2 &   77.46$\pm$0.4 &   57.21$\pm$7.3 &   74.52$\pm$0.4 &   86.48$\pm$1.9 &   \textbf{94.08}$\pm$0.1 &  \underline{92.83}$\pm$0.3 &   91.74$\pm$0.7 &   91.36$\pm$0.8 \\
            musk &  99.99$\pm$0.0 & 91.51$\pm$7.8  &   100.00$\pm$0.0 &100.0$\pm$0.0 &   100.0$\pm$0.0 &   94.56$\pm$0.2 &   100.0$\pm$0.0 &   95.74$\pm$0.3 &    95.6$\pm$3.0 &   100.0$\pm$0.0 &  100.0$\pm$0.0 &   100.0$\pm$0.0 &   100.0$\pm$0.0 \\
       optdigits &   76.52$\pm$6.0  & 93.61$\pm$1.8   &   95.34$\pm$3.0 &\underline{96.49}$\pm$0.9 &   56.76$\pm$0.5 &   68.09$\pm$0.5 &  50.31$\pm$15.1 &   60.29$\pm$0.4 &   82.09$\pm$4.1 &   94.69$\pm$0.8 &  \underline{97.33}$\pm$0.4 &   92.44$\pm$3.0 &   89.43$\pm$2.9 \\
      PageBlocks &  96.12$\pm$0.4 & 85.77$\pm$3.0  &    90.79$\pm$0.6 & 94.52$\pm$0.7 &   94.75$\pm$0.2 &   87.45$\pm$0.3 &   75.82$\pm$9.9 &   91.32$\pm$0.2 &   92.35$\pm$0.5 &    96.3$\pm$0.2 &  \underline{96.88}$\pm$0.2 &   \textbf{97.78}$\pm$0.2 &   \textbf{97.76}$\pm$0.4 \\
       pendigits & 92.57$\pm$4.3 &  77.81$\pm$16.3 &     92.44$\pm$5.9 & 99.25$\pm$0.4 &   94.33$\pm$0.2 &   90.67$\pm$0.2 &  45.35$\pm$19.3 &   92.73$\pm$0.2 &   96.91$\pm$0.7 &   \textbf{99.89}$\pm$0.0 &  99.14$\pm$0.3 &   \underline{99.64}$\pm$0.1 &   99.48$\pm$0.2 \\
            Pima & 69.02$\pm$2.3  &63.54$\pm$2.7&      49.61$\pm$3.8& 67.61$\pm$1.3 &   69.93$\pm$0.7 &   64.69$\pm$1.4 &   55.99$\pm$1.8 &   59.29$\pm$0.6 &   73.53$\pm$1.3 &    \textbf{74.1}$\pm$1.0 &  70.79$\pm$1.7 &    \underline{71.2}$\pm$1.7 &    73.9$\pm$1.8 \\
       satellite & 85.43$\pm$1.1  &50.90$\pm$4.3   &   76.87$\pm$2.2 & \textbf{87.74}$\pm$0.5 &   67.21$\pm$0.6 &   63.53$\pm$0.4 &    73.5$\pm$2.2 &   58.36$\pm$0.3 &    80.6$\pm$1.0 &   \underline{87.59}$\pm$0.2 &   84.6$\pm$0.4 &   81.22$\pm$0.8 &   82.06$\pm$0.7 \\
        satimage & 99.81$\pm$1.0 & 71.72$\pm$16.7 &91.76$\pm$5.3 &  99.81$\pm$0.0 &   98.04$\pm$0.2 &   97.41$\pm$0.0 &   95.96$\pm$0.7 &   96.45$\pm$0.1 &   99.35$\pm$0.2 &   \textbf{99.88}$\pm$0.0 &   99.6$\pm$0.1 &   \underline{99.87}$\pm$0.0 &   99.86$\pm$0.0 \\
         shuttle &  99.96$\pm$0.0 & 99.89$\pm$0.1  &    99.92$\pm$0.1& \textbf{99.99}$\pm$0.0 &   99.43$\pm$0.1 &   99.45$\pm$0.0 &   99.66$\pm$0.0 &    99.3$\pm$0.0 &   99.58$\pm$0.1 &   99.89$\pm$0.0 &  \underline{99.97}$\pm$0.0 &   \textbf{99.99}$\pm$0.0 &   \textbf{99.98}$\pm$0.0 \\
            skin & \underline{99.48}$\pm$0.2 & 98.43$\pm$0.8   &   89.65$\pm$2.6 &70.11$\pm$17.0 &   66.47$\pm$2.6 &   47.12$\pm$0.1 &   62.49$\pm$4.6 &   48.91$\pm$0.1 &   89.09$\pm$0.8 &   \textbf{99.81}$\pm$0.0 &  91.83$\pm$0.7 &   97.74$\pm$0.2 &   99.22$\pm$0.0 \\
            smtp &  \underline{96.00}$\pm$2.0 & 87.24$\pm$8.4  &    94.09$\pm$2.6 &\textbf{96.81}$\pm$1.8 &   83.26$\pm$0.8 &    91.2$\pm$0.0 &   89.52$\pm$5.2 &   88.01$\pm$0.1 &    90.7$\pm$0.5 &    93.1$\pm$0.2 &  93.05$\pm$0.8 &   92.37$\pm$1.4 &   93.96$\pm$2.0 \\
        SpamBase & 81.51$\pm$0.9 & 77.41$\pm$3.3   &   78.76$\pm$1.7 & \textbf{86.78}$\pm$1.2 &   79.92$\pm$1.0 &   68.73$\pm$0.5 &   74.43$\pm$4.1 &   65.56$\pm$0.4 &   83.32$\pm$1.7 &    83.3$\pm$0.7 &  80.61$\pm$1.0 &   84.44$\pm$0.8 &   \underline{84.84}$\pm$1.2 \\
          speech & \underline{57.12}$\pm$2.3 & 58.68$\pm$2.2    &  55.43$\pm$5.1  &56.55$\pm$2.0 &   47.29$\pm$0.4 &   49.18$\pm$0.3 &   50.83$\pm$1.4 &   47.29$\pm$0.3 &   48.18$\pm$1.5 &   48.62$\pm$0.9 &  49.64$\pm$0.6 &   55.98$\pm$1.9 &   \textbf{58.13}$\pm$2.2 \\
          Stamps & \underline{95.06}$\pm$1.2  &54.08$\pm$11.2 &     67.03$\pm$11.1 & 88.76$\pm$3.2 &   94.03$\pm$1.5 &   93.18$\pm$0.7 &   79.19$\pm$5.1 &   88.11$\pm$2.1 &   92.93$\pm$1.7 &   93.83$\pm$1.5 &  93.12$\pm$1.3 &   {94.49}$\pm$1.2 &   \textbf{95.25}$\pm$1.2 \\
         thyroid &  97.95$\pm$1.5  &97.39$\pm$2.9   &   74.34$\pm$7.5 &94.95$\pm$1.6 &   \textbf{98.71}$\pm$0.1 &   93.93$\pm$0.2 &  80.64$\pm$10.1 &   97.73$\pm$0.2 &   98.97$\pm$0.2 &   \underline{98.61}$\pm$0.1 &  96.31$\pm$1.1 &   97.74$\pm$0.5 &   98.49$\pm$0.4 \\
       vertebral & 55.21$\pm$5.6 & \underline{58.72}$\pm$13.8  &    54.54$\pm$8.6 & {51.34}$\pm$3.4 &   51.15$\pm$4.1 &   33.32$\pm$2.3 &   34.57$\pm$4.0 &   39.88$\pm$1.4 &   42.07$\pm$3.0 &   43.17$\pm$3.4 &  39.77$\pm$2.3 &   \textbf{71.24}$\pm$5.3 &    \textbf{69.8}$\pm$3.9 \\
          vowels &  85.20$\pm$4.5 & 75.42$\pm$7.7  &    96.97$\pm$1.3& \textbf{99.53}$\pm$0.2 &   62.91$\pm$1.0 &   49.88$\pm$0.9 &    74.1$\pm$5.1 &   58.97$\pm$0.7 &   79.15$\pm$2.2 &   96.93$\pm$0.2 &  95.51$\pm$0.9 &    99.0$\pm$0.2 &   \underline{99.03}$\pm$0.4 \\
        Waveform & 68.80$\pm$5.4 & 83.53$\pm$1.2   &   72.88$\pm$1.2 & 68.54$\pm$3.3 &   65.22$\pm$0.5 &   73.42$\pm$0.7 &   60.59$\pm$5.5 &   60.04$\pm$0.3 &    73.0$\pm$2.0 &   \underline{75.54}$\pm$0.7 &  \textbf{75.34}$\pm$0.8 &   75.29$\pm$1.6 &   72.38$\pm$3.0 \\
             WBC & 98.66$\pm$0.3  &48.87$\pm$8.2  &    63.30$\pm$8.5 & 95.61$\pm$0.3 &   99.04$\pm$0.6 &   \textbf{99.58}$\pm$0.3 &   93.32$\pm$1.9 &   99.53$\pm$0.3 &   \underline{99.53}$\pm$0.2 &   98.98$\pm$0.5 &   97.3$\pm$1.5 &   98.41$\pm$0.8 &   99.17$\pm$0.7 \\
            WDBC &  97.67$\pm$2.4 & 36.12$\pm$18.6   &   44.22$\pm$11.9& 98.46$\pm$1.5 &    99.0$\pm$0.6 &    \underline{99.3}$\pm$0.3 &   98.54$\pm$0.6 &   96.76$\pm$0.8 &   99.11$\pm$0.4 &   99.04$\pm$0.3 &  99.12$\pm$0.5 &   \textbf{99.97}$\pm$0.1 &   \textbf{99.76}$\pm$0.4 \\
            Wilt & 62.79$\pm$3.2 & \textbf{99.03}$\pm$3.1   &   65.44$\pm$2.4 & {74.16}$\pm$4.6 &    37.9$\pm$2.0 &   34.43$\pm$0.4 &    45.2$\pm$2.0 &   39.31$\pm$0.4 &   48.08$\pm$2.2 &   64.24$\pm$1.8 &  69.54$\pm$3.0 &   \underline{92.75}$\pm$1.2 &    \underline{94.3}$\pm$0.9 \\
            wine &  96.07$\pm$4.5 & 72.54$\pm$2.9  &    64.58$\pm$10.8 & 95.72$\pm$3.0 &   90.39$\pm$1.5 &   87.16$\pm$1.8 &   97.08$\pm$3.3 &   73.83$\pm$1.9 &   89.87$\pm$4.0 &   96.75$\pm$2.4 &  \underline{96.75}$\pm$1.4 &   \textbf{98.62}$\pm$2.1 &   \textbf{99.42}$\pm$0.9 \\
            WPBC &  \underline{56.61}$\pm$2.0 & 62.95$\pm$6.8   &   53.51$\pm$3.5 &46.67$\pm$3.3 &   49.65$\pm$3.2 &   50.32$\pm$2.8 &    49.9$\pm$2.3 &   48.21$\pm$1.4 &   50.83$\pm$3.1 &   52.99$\pm$2.8 &  {52.54}$\pm$2.8 &   \textbf{59.79}$\pm$1.8 &   \textbf{61.25}$\pm$2.4 \\
           yeast &  51.35$\pm$3.1 & \underline{59.59}$\pm$3.2  &    52.10$\pm$3.4 & 50.95$\pm$2.2 &   42.18$\pm$0.8 &   38.28$\pm$0.5 &   45.06$\pm$3.1 &   44.15$\pm$1.5 &   40.66$\pm$1.0 &   45.06$\pm$1.3 &  {46.92}$\pm$1.4 &   \textbf{58.54}$\pm$1.8 &   \textbf{57.53}$\pm$2.0 \\ \hline
       mean rank &  5.55  & 7.46& 7.55&        5.00 &  7.00 &  8.21 &  9.43&  8.55 &  6.68 &            3.76 & 5.234 &     \textbf{3.31} &     \textbf{2.80} \\
        mean &82.75$\pm$15.6 &75.40$\pm$18.3 &74.47$\pm$17.9 & 84.44$\pm$18.0 & 78.46$\pm$18.7 & 74.60$\pm$19.6 & 68.37$\pm$22.7 & 74.15$\pm$19.3 & 79.82$\pm$17.5 & 85.91$\pm$15.6 & 83.58$\pm$16.5 & \textbf{87.07$\pm$14.3} & \textbf{87.89$\pm$13.69} \\ \hline
        p-value &0.0020   &0.0000& 0.0000&  0.0325 &  0.0001 &  0.0000 &  0.0000 & 0.0000 &  0.0001 &    0.3128 & 0.0030 &     - &     0.0088 \\
        p-value & 0.0002  & 0.0000& 0.0000&  0.0039 &  0.0000 &  0.0000 &  0.0000 & 0.0000 &  0.0000 &    0.0716 & 0.0006 &     0.0088 &    -\\\hline
\bottomrule
\end{tabular}}
\caption{AUC (\%) score of the proposed method compared with 8 baselines on 47 benchmark datasets.
Each experiment is repeated 10 times with seed from 0 to 9, and mean value and standard deviation are reported. Mean is the
average AUC score under all experiments. Mean rank is calculated out of 10 tested methods. Cardio. refers to Cardiotocography. Mammo. refers to mammography. Lympho. refers to Lymphography.}
\label{tab: adbench auc}
\end{table*}

\begin{table*}[ht]
\centering
\setlength\tabcolsep{1.5pt}
\resizebox{\textwidth}{!}{
\begin{tabular}{l|ccccccccccc|cc}
\toprule\hline
         \textbf{Methods} &   \textbf{DPAD} &  \textbf{PLAD}&  \textbf{NeuTraLAD}    &       \textbf{SCAD} &           \textbf{AE} &        \textbf{COPOD} &     \textbf{DeepSVDD} &         \textbf{ECOD} &      \textbf{IForest} &          \textbf{KNN} &         \textbf{LOF} &       \makecell[c]{\textbf{Ours-} \\\textbf{ResMLP}} & \makecell[c]{\textbf{Ours-} \\\textbf{MLP}} \\
\midrule
ALOI       & 14.15$\pm$0.2 & 9.40$\pm$0.8 & 7.79$\pm$0.6      &              11.79$\pm$0.6 &              8.02$\pm$0.1 &               5.26$\pm$0.2 &     6.92$\pm$0.5 &                5.42$\pm$0.1 &           6.72$\pm$0.6 &  \underline{15.87$\pm$0.3} &      \textbf{20.16$\pm$0.8} &               14.32$\pm$0.6 &                13.53$\pm$0.7 \\
annthyroid   & 63.97$\pm$3.6 & \textbf{76.22}$\pm$6.8 & 42.06$\pm$5.0    &              54.68$\pm$0.9 &             51.02$\pm$1.7 &               31.6$\pm$0.3 &    21.99$\pm$3.0 &               39.03$\pm$0.3 &          56.69$\pm$2.3 &              63.73$\pm$1.6 &              {67.87}$\pm$1.6 &       \underline{75.59$\pm$2.6} &    {74.74$\pm$1.6} \\
backdoor  & 85.62$\pm$0.7 & 80.68$\pm$14.1 & \textbf{87.35}$\pm$0.2       &      \underline{87.25$\pm$0.2} &             85.14$\pm$0.0 &              13.67$\pm$0.3 &    84.56$\pm$2.6 &               16.87$\pm$0.2 &           3.24$\pm$1.1 &               85.3$\pm$0.0 &  {86.28$\pm$0.2} &               85.64$\pm$0.8 &                85.47$\pm$0.2 \\
breastw     & 96.9$\pm$0.2 & 73.64$\pm$2.8 & 72.38$\pm$3.2     &              94.77$\pm$0.7 &             95.94$\pm$0.6 &  \underline{97.12$\pm$0.3} &    91.72$\pm$1.5 &               95.31$\pm$0.5 &  \textbf{97.28$\pm$0.5} &              96.37$\pm$0.6 &              91.97$\pm$1.4 &                96.3$\pm$0.9 &                96.57$\pm$1.1 \\
campaign   & 46.22$\pm$1.9 & 46.29$\pm$2.9 & 42.22$\pm$3.9      &      \textbf{50.63$\pm$0.6} &             49.15$\pm$0.2 &               49.6$\pm$0.2 &   31.62$\pm$11.6 &               48.84$\pm$0.3 &           43.3$\pm$0.7 &  \underline{50.56$\pm$0.5} &               42.1$\pm$0.9 &               44.79$\pm$2.4 &                49.94$\pm$1.7 \\
cardio    & 70.91$\pm$8.6 & 44.32$\pm$5.9 & 39.66$\pm$5.4       &              66.05$\pm$1.5 &             \underline{78.33$\pm$1.9} &              65.77$\pm$1.7 &    43.41$\pm$6.3 &                65.8$\pm$0.9 &           70.1$\pm$3.5 &              67.71$\pm$3.0 &               65.0$\pm$4.1 &       \textbf{81.63$\pm$2.8} &    \textbf{80.28$\pm$1.8} \\
Cardiotocography & 62.15$\pm$2.7 & 56.22$\pm$5.4 & 41.07$\pm$3.8 &              53.86$\pm$1.6 &             \underline{65.14$\pm$1.3} &              48.42$\pm$0.9 &    37.47$\pm$7.7 &               62.66$\pm$0.8 &          62.82$\pm$1.8 &              56.12$\pm$1.7 &              61.24$\pm$1.1 &       \textbf{69.42$\pm$4.6} &    \textbf{69.23$\pm$3.5} \\
celeba    & 11.09$\pm$1.4 & \underline{24.98}$\pm$1.8 & 7.0$\pm$4.6       &              14.03$\pm$2.4 &      \textbf{27.3$\pm$0.3} &              23.39$\pm$0.2 &     8.87$\pm$8.0 &   {23.54$\pm$0.2} &          17.63$\pm$1.7 &              15.45$\pm$0.4 &               1.41$\pm$0.2 &                 8.2$\pm$0.7 &                10.51$\pm$1.8 \\
census    & 12.36$\pm$0.0 & 20.52$\pm$2.7 & 15.88$\pm$2.7       &              \underline{23.32}$\pm$0.6 &             21.46$\pm$0.1 &              12.86$\pm$0.1 &    17.56$\pm$2.4 &               12.22$\pm$0.1 &          11.24$\pm$1.2 &               22.3$\pm$0.2 &              13.72$\pm$0.3 &       \textbf{26.86$\pm$1.5} &    \textbf{23.44$\pm$3.1} \\
cover      & 25.37$\pm$24.1 & 33.63$\pm$2.3 & 19.27$\pm$16.1      &              65.25$\pm$9.7 &             20.16$\pm$5.0 &              18.17$\pm$0.2 &     2.72$\pm$2.9 &               23.54$\pm$0.2 &          11.33$\pm$1.9 &              74.89$\pm$0.3 &              \underline{79.03}$\pm$2.1 &       \textbf{91.38$\pm$0.3} &    \textbf{89.97$\pm$1.0} \\
donors    & 84.49$\pm$10.8 & 84.12$\pm$15.7 & 97.82$\pm$0.9       &  \underline{97.87$\pm$0.6} &             32.06$\pm$4.8 &               41.8$\pm$0.1 &   28.41$\pm$33.1 &               44.89$\pm$0.1 &          43.97$\pm$5.5 &      \textbf{99.14$\pm$0.1} &              86.23$\pm$0.3 &               94.51$\pm$1.2 &                92.56$\pm$2.4 \\
fault    & 67.7$\pm$0.7 & 60.22$\pm$3.2 & 67.19$\pm$1.1        &  \underline{72.62$\pm$0.9} &             54.87$\pm$0.6 &              47.77$\pm$0.7 &    50.88$\pm$4.9 &               48.65$\pm$0.1 &          61.68$\pm$1.7 &      \textbf{73.13$\pm$1.0} &              64.19$\pm$0.6 &               68.03$\pm$0.9 &                 68.3$\pm$1.0 \\
fraud     & 51.91$\pm$10.9 & 78.66$\pm$6.0 & 45.05$\pm$21.9       &  \underline{65.13$\pm$2.4} &             32.44$\pm$0.6 &              44.72$\pm$0.5 &   52.48$\pm$14.7 &               38.86$\pm$0.6 &          30.77$\pm$4.5 &              47.12$\pm$2.6 &                0.0$\pm$0.0 &               60.19$\pm$9.2 &        \textbf{78.68$\pm$1.6} \\
glass    & 31.11$\pm$4.4 & \underline{46.67}$\pm$9.3 & 33.33$\pm$24.3        &               38.1$\pm$8.7 &              12.5$\pm$3.9 &              15.28$\pm$8.3 &    {28.89}$\pm$9.9 &               15.56$\pm$6.1 &           12.5$\pm$7.1 &              20.37$\pm$8.4 &              15.56$\pm$6.1 &   {40.28$\pm$8.3} &        \textbf{56.57$\pm$7.8} \\
Hepatitis   & \underline{60.0}$\pm$12.3 & 51.92$\pm$3.8 & 29.23$\pm$10.2     &              36.26$\pm$8.6 &             51.92$\pm$6.8 &              53.85$\pm$0.0 &   44.62$\pm$10.0 &                40.0$\pm$3.4 &         50.96$\pm$10.0 &              {55.13}$\pm$5.8 &              52.31$\pm$3.4 &   \textbf{65.38$\pm$4.1} &        \textbf{69.23$\pm$4.9} \\
http     & \textbf{99.68}$\pm$0.1 & 99.22$\pm$0.8 & 1.9$\pm$0.1        &             84.57$\pm$31.3 &             92.09$\pm$0.8 &               1.99$\pm$0.0 &    9.91$\pm$17.7 &                 2.2$\pm$0.0 &          13.9$\pm$15.3 &  {99.56$\pm$0.1} &               0.05$\pm$0.0 &               99.25$\pm$0.5 &        \underline{99.57$\pm$0.2} \\
InternetAds   & \underline{82.34}$\pm$0.7 & 79.46$\pm$1.8 & 81.63$\pm$1.2   &              81.17$\pm$2.1 &             79.93$\pm$1.3 &              50.75$\pm$0.5 &     {81.3$\pm$1.5} &               50.43$\pm$0.3 &          24.12$\pm$2.6 &              80.12$\pm$2.1 &              80.65$\pm$0.7 &       \textbf{84.24$\pm$1.4} &    \textbf{84.19$\pm$0.8} \\
Ionosphere & 88.57$\pm$2.5 & 56.51$\pm$8.7 & 80.63$\pm$1.3      &              89.55$\pm$0.9 &             81.65$\pm$1.6 &              69.44$\pm$1.4 &    91.59$\pm$1.2 &               65.71$\pm$0.7 &          79.17$\pm$2.6 &              \underline{92.06$\pm$1.4} &              87.14$\pm$3.3 &   \textbf{94.33$\pm$1.9} &        \textbf{94.88$\pm$0.7} \\
landsat   & 56.55$\pm$2.2 & 52.17$\pm$4.0 & 53.67$\pm$8.5       &       \textbf{60.4$\pm$0.4} &             31.85$\pm$0.5 &              28.81$\pm$0.4 &    49.21$\pm$2.2 &               25.25$\pm$0.6 &           44.7$\pm$1.1 &              58.73$\pm$0.4 &  \underline{59.01$\pm$1.1} &               42.69$\pm$2.4 &                45.82$\pm$4.1 \\
letter   & 35.2$\pm$5.3 & 44.50$\pm$6.6 & 58.6$\pm$6.5        &              \underline{63.17$\pm$3.9} &             13.88$\pm$0.8 &               14.0$\pm$1.4 &     26.0$\pm$2.5 &                16.0$\pm$1.2 &          16.88$\pm$3.4 &              46.83$\pm$3.9 &               48.4$\pm$3.2 &       \textbf{67.79$\pm$2.0} &    \textbf{65.45$\pm$2.8} \\
Lymphography   & 83.33$\pm$0.0 & 70.83$\pm$23.3 & 0.0$\pm$0.0  &             88.89$\pm$13.6 &            79.17$\pm$24.8 &      \textbf{93.75$\pm$8.6} &    70.0$\pm$18.3 &    \underline{90.0$\pm$9.1} &          83.33$\pm$8.9 &             80.56$\pm$24.5 &             83.33$\pm$16.7 &               86.9$\pm$17.9 &               75.76$\pm$21.6 \\
magic   & 71.52$\pm$3.7 & 72.46$\pm$2.6 & 66.77$\pm$1.1   &               73.5$\pm$0.2 &             64.71$\pm$0.8 &               63.1$\pm$0.2 &    60.36$\pm$0.3 &               59.94$\pm$0.2 &          70.12$\pm$1.2 &              \underline{76.02$\pm$0.3} &              75.99$\pm$0.3 &   \textbf{81.02$\pm$0.3} &        \textbf{81.74$\pm$0.3} \\
mammography   & 43.54$\pm$3.8 & 27.4$\pm$2.4 & 7.85$\pm$1.3   &              26.47$\pm$1.6 &             43.41$\pm$2.1 &      \textbf{53.32$\pm$0.8} &   27.54$\pm$12.8 &   \underline{53.31$\pm$1.2} &          38.27$\pm$5.3 &              42.31$\pm$0.8 &              36.54$\pm$0.8 &               50.71$\pm$3.0 &                49.83$\pm$1.9 \\
mnist   & 59.46$\pm$5.8 & 58.50$\pm$5.5 & 55.0$\pm$1.9         &              67.62$\pm$2.0 &              65.8$\pm$1.0 &              38.57$\pm$0.6 &    34.94$\pm$5.6 &               33.71$\pm$1.2 &          53.55$\pm$4.4 &      \textbf{72.26$\pm$0.7} &              69.89$\pm$1.4 &   \underline{70.22$\pm$1.9} &                69.66$\pm$1.4 \\
musk       & 98.76$\pm$2.0 & 75.26$\pm$22.1 & 100.0$\pm$0.      &              \textbf{100.0$\pm$0.0} &             100.0$\pm$0.0 &              48.97$\pm$1.1 &    \textbf{100.0$\pm$0.0} &               55.05$\pm$1.6 &         54.12$\pm$16.7 &              \textbf{100.0$\pm$0.0} &              \textbf{100.0$\pm$0.0} &   \textbf{100.0$\pm$0.0} &        \textbf{100.0$\pm$0.0} \\
optdigits  & 25.33$\pm$8.0 & 52.33$\pm$8.0 & 53.47$\pm$10.9      &              50.44$\pm$7.6 &               0.1$\pm$0.3 &               3.08$\pm$0.3 &      2.8$\pm$5.5 &                2.67$\pm$0.0 &           14.0$\pm$6.1 &               32.0$\pm$7.2 &      \textbf{58.27$\pm$5.0} &  \underline{54.33$\pm$11.5} &                37.93$\pm$4.9 \\
PageBlocks  & 80.63$\pm$1.6 & 68.24$\pm$4.4 & 62.67$\pm$1.8     &              75.42$\pm$2.0 &             68.96$\pm$0.5 &              48.97$\pm$0.9 &    60.31$\pm$9.5 &               57.22$\pm$0.7 &          63.46$\pm$1.8 &              78.99$\pm$0.9 &              \underline{81.29$\pm$0.5} &   \textbf{83.37$\pm$1.0} &        \textbf{83.94$\pm$1.4} \\
pendigits  & 49.36$\pm$19.0 & 44.39$\pm$18.8 & 36.54$\pm$23.2      &              76.92$\pm$7.9 &             44.23$\pm$1.3 &              35.18$\pm$0.2 &     10.9$\pm$9.0 &               42.18$\pm$0.3 &          55.93$\pm$3.9 &       \textbf{93.7$\pm$1.3} &              75.13$\pm$4.2 &   \underline{86.97$\pm$2.2} &                83.72$\pm$3.6 \\
Pima   & 66.19$\pm$1.9 & 62.97$\pm$1.8 & 51.72$\pm$1.8          &              64.99$\pm$2.3 &             67.96$\pm$0.5 &              62.03$\pm$1.5 &    54.93$\pm$1.1 &               58.21$\pm$0.8 &          68.61$\pm$1.6 &       \textbf{69.9$\pm$1.4} &              67.24$\pm$1.6 &               68.41$\pm$1.8 &    \underline{69.85$\pm$1.7} \\
satellite  & 75.63$\pm$1.1 & 51.45$\pm$4.0 & 68.26$\pm$1.1      &      \textbf{77.72$\pm$0.3} &             62.53$\pm$0.6 &               56.8$\pm$0.4 &    65.11$\pm$2.2 &               53.94$\pm$0.1 &          69.17$\pm$0.5 &  \underline{76.35$\pm$0.3} &              75.83$\pm$0.5 &               73.01$\pm$0.8 &                73.87$\pm$0.5 \\
satimage   & 89.3$\pm$3.0 & 32.04$\pm$3.5 & 11.83$\pm$4.0      &  \underline{93.66$\pm$2.1} &             85.51$\pm$0.7 &              77.29$\pm$1.2 &    87.89$\pm$1.9 &               71.55$\pm$0.6 &          86.09$\pm$1.9 &              92.25$\pm$1.9 &              85.35$\pm$3.2 &       \textbf{94.77$\pm$0.7} &                 90.7$\pm$2.5 \\
shuttle   & 98.17$\pm$0.2 & 98.63$\pm$0.3 & 98.15$\pm$1.0       &      \textbf{99.01$\pm$0.1} &             95.95$\pm$0.2 &              95.96$\pm$0.1 &    97.56$\pm$0.5 &               91.55$\pm$0.3 &          94.72$\pm$1.8 &              98.03$\pm$0.2 &              98.26$\pm$0.2 &   \underline{98.56$\pm$0.2} &                98.48$\pm$0.2 \\
skin     & \underline{96.05}$\pm$1.5 & 93.16$\pm$2.4 & 71.7$\pm$3.8         &             54.57$\pm$17.3 &             38.65$\pm$4.1 &              19.63$\pm$0.1 &    44.62$\pm$3.2 &               21.88$\pm$0.0 &          77.11$\pm$1.6 &      \textbf{97.72$\pm$0.1} &              80.39$\pm$1.7 &               91.07$\pm$0.3 &    {94.35$\pm$0.5} \\
smtp     & \underline{66.67}$\pm$0.0 & 52.50$\pm$16.0 & 48.67$\pm$10.9        &             42.78$\pm$14.7 &             66.67$\pm$0.0 &                0.0$\pm$0.0 &   55.33$\pm$23.5 &               66.67$\pm$0.0 &            0.0$\pm$0.0 &  \underline{66.67$\pm$0.0} &              38.0$\pm$21.0 &               63.89$\pm$6.8 &        \textbf{66.67$\pm$0.0} \\
SpamBase   & 78.13$\pm$0.7 & 74.14$\pm$3.0 & 75.84$\pm$1.6      &      \textbf{81.68$\pm$1.3} &             76.53$\pm$1.0 &              69.59$\pm$0.3 &    71.82$\pm$3.3 &               67.29$\pm$0.1 &          79.37$\pm$1.6 &              79.66$\pm$0.7 &              77.52$\pm$1.3 &               80.76$\pm$0.9 &    \underline{80.87$\pm$1.1} \\
speech   & {8.85}$\pm$2.5 & \underline{10.25}$\pm$4.5 & 7.54$\pm$4.2        &               5.25$\pm$2.7 &              4.92$\pm$0.0 &               3.28$\pm$0.0 &     3.93$\pm$1.9 &                4.92$\pm$0.0 &           5.33$\pm$2.1 &               4.92$\pm$1.0 &               {6.23}$\pm$0.7 &   \textbf{12.02$\pm$2.0} &        \textbf{13.11$\pm$1.5} \\
Stamps   & 75.48$\pm$4.8 & 30.65$\pm$10.4 & 24.52$\pm$12.0        &             54.19$\pm$11.9 &             61.75$\pm$8.2 &              \underline{67.74}$\pm$3.9 &    37.42$\pm$9.6 &               48.39$\pm$6.5 &          61.29$\pm$5.7 &              64.52$\pm$8.2 &              62.58$\pm$8.1 &   \textbf{73.66$\pm$2.4} &        \textbf{74.84$\pm$5.6} \\
thyroid   & \underline{76.77}$\pm$5.5 & 74.19$\pm$6.5 & 7.96$\pm$7.0       &              61.72$\pm$7.2 &  {76.5$\pm$1.1} &              30.24$\pm$0.4 &   47.96$\pm$17.6 &               61.94$\pm$1.4 &  \textbf{81.32$\pm$2.9} &              73.84$\pm$1.6 &              58.06$\pm$5.0 &               70.25$\pm$5.9 &                76.24$\pm$5.2 \\
vertebral  & \underline{31.33}$\pm$9.6 & \underline{31.33}$\pm$14.5 & 28.67$\pm$9.6      &              {20.67}$\pm$2.8 &             17.14$\pm$1.3 &               0.42$\pm$1.2 &     9.33$\pm$4.3 &               13.33$\pm$0.0 &           15.0$\pm$3.6 &              13.33$\pm$3.0 &              16.67$\pm$4.1 &       \textbf{42.22$\pm$8.3} &     \textbf{41.0$\pm$3.9} \\
vowels     & 43.2$\pm$6.9 & 39.2$\pm$11.2 & 65.2$\pm$8.7       &       \textbf{88.0$\pm$3.3} &             19.71$\pm$0.8 &               4.75$\pm$1.0 &     33.2$\pm$9.4 &                22.0$\pm$0.0 &          28.25$\pm$4.8 &              65.33$\pm$4.3 &               56.8$\pm$4.1 &   \underline{86.67$\pm$2.4} &                 85.4$\pm$2.5 \\
Waveform   & 20.6$\pm$2.9 & 57.00$\pm$3.4 & 47.8$\pm$1.6      &              19.75$\pm$4.4 &              9.14$\pm$0.7 &                9.5$\pm$1.1 &     15.8$\pm$2.5 &                 7.4$\pm$0.5 &          10.75$\pm$1.8 &       \textbf{29.5$\pm$2.7} &   \underline{28.8$\pm$1.1} &                27.0$\pm$3.3 &                 20.2$\pm$2.5 \\
WBC   & 82.0$\pm$4.0 & 30.00$\pm$33.9 & 12.0$\pm$7.5           &               57.5$\pm$5.0 &             85.71$\pm$5.3 &   \underline{87.5$\pm$7.1} &    54.0$\pm$11.4 &        \textbf{88.0$\pm$4.5} &          86.25$\pm$5.2 &              83.33$\pm$8.2 &              68.0$\pm$11.0 &               81.67$\pm$7.5 &                 83.0$\pm$9.5 \\
WDBC    & 72.0$\pm$11.7 & 20.00$\pm$4.5 & 0.0$\pm$0.0         &               70.0$\pm$8.2 &             78.57$\pm$9.0 &              \underline{81.25}$\pm$3.5 &     72.0$\pm$8.4 &                50.0$\pm$7.1 &           72.5$\pm$8.9 &              76.67$\pm$5.2 &              80.0$\pm$10.0 &       \textbf{98.33$\pm$4.1} &     \textbf{92.0$\pm$7.9} \\
Wilt   & 14.86$\pm$5.4 & \textbf{89.88}$\pm$3.4 & 28.87$\pm$5.4          &             {31.23}$\pm$5.5 &              2.95$\pm$0.7 &               1.61$\pm$0.1 &     3.97$\pm$1.4 &                5.68$\pm$0.2 &            1.9$\pm$0.4 &               2.01$\pm$0.5 &              17.98$\pm$4.0 &   \underline{62.52$\pm$5.7} &         \underline{68.6$\pm$2.9} \\
wine    & 82.0$\pm$17.2 & 46.0$\pm$38.5 & 24.0$\pm$13.6         &              72.5$\pm$15.0 &              60.0$\pm$8.2 &              58.75$\pm$3.5 &    \underline{78.0}$\pm$13.0 &                34.0$\pm$8.9 &         63.75$\pm$13.0 &             71.67$\pm$11.7 &               74.0$\pm$8.9 &   \textbf{90.0$\pm$12.6} &         \textbf{93.0$\pm$6.7} \\
WPBC      & \underline{48.09}$\pm$3.2 & 48.94$\pm$5.3 & 41.7$\pm$5.1        &              32.98$\pm$2.1 &             36.17$\pm$5.2 &              31.38$\pm$4.5 &    37.02$\pm$2.4 &               36.17$\pm$0.0 &          35.64$\pm$2.2 &              35.46$\pm$2.6 &               {38.3}$\pm$3.4 &   \textbf{50.35$\pm$2.6} &        \textbf{51.06$\pm$2.0} \\
yeast   & 52.54$\pm$2.7 & \textbf{58.54}$\pm$2.3 & 52.24$\pm$3.2         &              {52.07}$\pm$1.9 &             44.24$\pm$0.8 &              41.72$\pm$0.6 &    47.61$\pm$1.2 &               46.43$\pm$1.2 &          43.44$\pm$1.0 &               47.4$\pm$1.1 &              49.03$\pm$1.1 &       \underline{57.66$\pm$1.3} &    \underline{57.22$\pm$1.6} \\
 \hline
       mean rank & 4.53 & 6.38& 7.63& 4.91 &  6.97 &      9.00 &        8.60 &  8.74 &   7.61 &      4.36 & 5.57 &     \textbf{2.85} &    \textbf{ 2.70} \\
            mean &60.34$\pm$27.4& 55.1$\pm$25.1& 44.1$\pm$29.1& 60.64$\pm$26.2& 51.95$\pm$29.0& 40.82$\pm$28.1& 44.56$\pm$29.5& 42.22$\pm$25.1& 46.22$\pm$29.1& 62.66$\pm$27.8& 57.06$\pm$28.4& \textbf{69.09}$\pm$24.7& \textbf{69.4}$\pm$24.5 \\\hline
                    p-value &  0.0002& 0.0001& 0.0000&   0.0000 &  0.0000 &  0.0000 &  0.0000 & 0.0000 &  0.0000 &    0.0011 & 0.0000 &     - &     0.68699 \\
         p-value &  0.0001 & 0.0001& 0.0000&  0.0000 &  0.0000 &  0.0000 &  0.0000 & 0.0000 &  0.0000 &    0.0017 & 0.0002 &     0.68699 &     - \\\hline
\bottomrule
\end{tabular}
}
\caption{F1 (\%) score of the proposed method compared with 8 baselines on 47 benchmark datasets.
Each experiment is repeated 10 times with seed from 0 to 9, and mean value and standard deviation are reported. Mean is the
average F1 score under all experiments. Mean rank is calculated out of 10 tested methods. Cardio. refers to Cardiotocography. Mammo. refers to mammography. Lympho. refers to Lymphography.}
\label{tab: adbench f1}
\end{table*}

\section{Appendix F: Detailed Anomaly Detection Results} \label{app: ad results}
The detailed results in terms of AUC score and F1 score under the UAD setting are reported here in Table \ref{tab: adbench auc} and \ref{tab: adbench f1}. We do not include the results for AnoGAN baseline because it is not efficient to run. So, we only compare 11 baselines for the UAD setting. The paired t-test is conducted based on the average AUC/F1 score through 10 runs in Tables 2, 5, 6, and 7. For example, in Table 2, every AD method has 25 values (as there are 25 datasets). Thus for any two AD methods, we have 25 pairs of values, on which we performed the paired t-test. The last two lines in the table are the p-values of our two methods (ResMLP and MLP) against each baseline method.
Since the anomaly in the test set is rare, F1 score metric is more important to show the significance of our method. It is seen that our method has the highest mean F1 score (69.40\%) and average rank (2.70). Compared with other baseline methods, our method has statistically significant improvement in terms of pair t-test p-value.

\section{Appendix G: F1 Score Results on OCC setting} \label{app: f1}
The performance in terms of F1 score under the OCC setting is presented in Table \ref{tab: main result f1}. Our approach also obtains the highest mean F1 score and mean rank out of 12 baselines.
\begin{table*}[h]
\setlength\tabcolsep{1.5pt}
\centering
% \sisetup{table-format=3.2(3.1), table-space-text-post=*, separate-uncertainty=true, detect-all}
\resizebox{\textwidth}{!}{
\begin{tabular}{l|cccccccccccc|cc}
\toprule
\hline
\textbf{Methods}& \textbf{DPAD} & \textbf{PLAD} & \textbf{NeuTraLAD} & \textbf{SCAD} & \textbf{AE} & \textbf{AnoGAN} & \textbf{COPOD} & \textbf{DeepSVDD} & \textbf{ECOD} & \textbf{IForest} & \textbf{KNN} & \textbf{LOF} & \makecell[c]{\textbf{Ours-} \\\textbf{ResMLP}} & \makecell[c]{\textbf{Ours-} \\\textbf{MLP}} \\
\midrule
abalone     & \underline{97.23}$\pm$2.7 & 97.12$\pm$2.8 & 96.96$\pm$2.9 & 97.05$\pm$2.8 & 97.17$\pm$2.6 & 97.13$\pm$2.7 & 97.01$\pm$2.7 & 96.99$\pm$2.9  & 97.00$\pm$2.7  & 97.22$\pm$2.6 & 97.09$\pm$2.8  & 97.15$\pm$2.7  & \textbf{97.38}$\pm$2.5  & \textbf{97.39}$\pm$2.5  \\ 
arrhythmia  & \underline{80.10}$\pm$0.8 & 70.53$\pm$5.0 & 66.18$\pm$0.0 & 74.01$\pm$0.7 & 78.03$\pm$1.0 & 73.14$\pm$1.8 & 73.75$\pm$1.1 & 76.37$\pm$1.4  & 74.59$\pm$1.5 & 75.74$\pm$1.4  & 78.08$\pm$1.2  & 78.23$\pm$1.2 & \textbf{80.12}$\pm$1.4  & \textbf{79.16}$\pm$1.4  \\ 
breastw     & 95.6$\pm$0.6 & 80.50$\pm$8.7 & 71.78$\pm$3.0 & 93.73$\pm$0.9 & 95.75$\pm$0.6 & 96.18$\pm$0.9 & 96.35$\pm$0.8 & 91.10$\pm$1.5   & 94.27$\pm$0.7 & \textbf{96.83}$\pm$0.7 & \underline{96.15}$\pm$0.8  & 91.74$\pm$2.0  & 95.97$\pm$0.8  & 95.79$\pm$1.0  \\ 
cardio      &\underline{67.09}$\pm$4.0 & 62.85$\pm$9.6 & 42.38$\pm$2.1 & 54.63$\pm$2.4 & 65.36$\pm$0.9 & 58.28$\pm$15.7 & 47.82$\pm$1.0 & 39.66$\pm$6.1  & 61.54$\pm$0.6 & 65.03$\pm$2.1  & 59.01$\pm$1.2  & 62.44$\pm$2.0  & \textbf{70.15}$\pm$3.0  & \textbf{71.22}$\pm$3.8  \\ 
ecoli       & \underline{96.09}$\pm$1.7 & 89.38$\pm$10.5 & 86.32$\pm$9.4 & 94.05$\pm$2.2 & 95.77$\pm$2.1 & 95.10$\pm$2.1 & 88.91$\pm$4.2 & 93.42$\pm$2.5  & 87.55$\pm$6.2 & 95.71$\pm$1.7  & 95.87$\pm$1.8 & 95.73$\pm$1.8  & \textbf{96.47}$\pm$1.5  & \textbf{96.61}$\pm$1.5  \\ 
glass       & \underline{90.00}$\pm$5.9 & 84.75$\pm$7.2 & 87.66$\pm$4.1 & 89.87$\pm$5.3 & 87.60$\pm$8.0  & 88.81$\pm$7.2 & 84.95$\pm$6.7 & 88.60$\pm$6.5   & 84.96$\pm$7.3 & 89.10$\pm$6.8   & 89.11$\pm$6.4  & 86.70$\pm$6.3   & \textbf{91.31}$\pm$5.8  & \textbf{91.76}$\pm$5.6  \\ 
ionosphere  & 91.11$\pm$0.8 & 64.29$\pm$10.4 & 85.71$\pm$2.2 & 90.32$\pm$0.7 & 79.98$\pm$1.9 & 78.10$\pm$6.1 & 70.57$\pm$1.5 & 89.75$\pm$1.1  & 66.37$\pm$1.5 & 80.58$\pm$4.5  & \underline{91.61}$\pm$1.9 & 87.91$\pm$2.4  & \textbf{95.24}$\pm$0.6  & \textbf{92.63}$\pm$2.0  \\ 
kdd         & 88.04$\pm$4.6 & 93.34$\pm$0.5 & 94.2$\pm$1.1 &93.81$\pm$0.8&93.74$\pm$2.9 & 91.25$\pm$2.0 & 90.93$\pm$0.8 & 87.06$\pm$8.5 & 91.63$\pm$0.3 & 94.26$\pm$0.6 & 94.69$\pm$0.1 & \underline{95.58}$\pm$0.1 & \textbf{96.50}$\pm$0.4 & \textbf{96.72}$\pm$0.8      \\
letter       & 98.86$\pm$0.3 & 98.01$\pm$0.1 & 98.99$\pm$0.2 & \textbf{99.60}$\pm$0.1  & 98.67$\pm$0.3  & 98.47$\pm$0.3  & 98.05$\pm$0.1  & 98.77$\pm$0.2  & 98.05$\pm$0.1  & 98.98$\pm$0.2  & 99.4$\pm$0.2   & 98.86$\pm$0.3  & \underline{99.55}$\pm$0.1  & \underline{99.57}$\pm$0.1  \\ 
lympho       & \underline{77.05}$\pm$2.7 & 68.85$\pm$1.6 & 58.36$\pm$3.4 & 74.59$\pm$4.4  & 75.91$\pm$3.0 & 71.80$\pm$4.8 & 61.61$\pm$2.2  & 73.77$\pm$4.6  & 62.76$\pm$2.7  & 73.36$\pm$3.8  & 75.78$\pm$3.3  & 75.55$\pm$3.2  & \textbf{80.56}$\pm$3.3  & \textbf{81.73}$\pm$2.9  \\ 
mammo.  & \underline{50.08}$\pm$4.0 & 31.28$\pm$7.5 & 7.08$\pm$1.1 & 25.42$\pm$3.2  & 44.09$\pm$2.7  & 36.41$\pm$11.9 & 53.37$\pm$0.9  & 33.40$\pm$12.4  & \textbf{53.65}$\pm$0.8 & 38.81$\pm$3.6  & 41.36$\pm$2.0  & 37.23$\pm$1.4  & \underline{49.92}$\pm$5.0  & 45.69$\pm$5.1  \\ 
mulcross     & 99.95$\pm$0.1 & 99.67$\pm$0.6 & 25.41$\pm$26.5 & 99.95$\pm$0.1  & 100.0$\pm$0.0  & 98.52$\pm$2.0 & 66.00$\pm$0.1   & 100.0$\pm$0.0 & 74.58$\pm$0.2  & 99.11$\pm$0.6  & 100.0$\pm$0.0  & 100.0$\pm$0.0  & \textbf{100.0}$\pm$0.0  & \textbf{100.0}$\pm$0.0  \\ 
musk         & 50.29$\pm$4.0 & 53.84$\pm$4.5 & 62.73$\pm$0.8 & \underline{66.90}$\pm$0.6  & 6.99$\pm$0.5   & 24.29$\pm$10.3 & 11.68$\pm$0.3  & 44.30$\pm$6.3   & 15.59$\pm$0.4  & 26.65$\pm$5.0  & 63.64$\pm$0.7  & 66.58$\pm$1.0  & \textbf{71.84}$\pm$1.9  & \textbf{68.78}$\pm$1.3  \\ 
optdigits    & 96.66$\pm$1.7 & 91.73$\pm$2.5 & 92.83$\pm$2.2 & \underline{98.64}$\pm$0.9 & 97.95$\pm$1.1  & 95.84$\pm$1.7             & 92.18$\pm$1.4  & 96.95$\pm$1.4  & 92.04$\pm$1.4  & 98.13$\pm$1.1  & 98.60$\pm$0.8   & 98.63$\pm$0.8  & \textbf{98.80}$\pm$0.8   & \textbf{98.91}$\pm$0.8  \\
pendigits    & 98.81$\pm$0.7 & 96.28$\pm$1.2 & 98.76$\pm$0.6 & \underline{99.64}$\pm$0.3 & 98.48$\pm$1.1  & 98.15$\pm$1.2  & 94.82$\pm$1.0  & 97.53$\pm$1.1  & 94.77$\pm$0.5  & 99.04$\pm$0.5  & 99.59$\pm$0.3  & 99.38$\pm$0.5  & \textbf{99.77}$\pm$0.2  & \textbf{99.78}$\pm$0.1  \\ 
pima         & 67.31$\pm$1.3 & 64.55$\pm$3.3 & 51.34$\pm$3.2 & 65.07$\pm$1.1  & 67.72$\pm$1.2  & 67.01$\pm$3.1 & 62.69$\pm$1.2  & 56.96$\pm$2.6  & 58.12$\pm$0.9  & 68.94$\pm$1.4  & \textbf{70.06}$\pm$1.0 & 67.24$\pm$1.2  & \underline{69.03}$\pm$2.0  & \underline{69.78}$\pm$2.1  \\ 
satimage     & 94.76$\pm$2.3 & 83.49$\pm$7.0 & 83.19$\pm$7.0 & 94.76$\pm$1.9 & 95.05$\pm$2.4 & 94.75$\pm$2.6            & 89.43$\pm$2.9 & 86.66$\pm$5.2  & 87.21$\pm$3.5 & 96.05$\pm$2.0 & \textbf{95.65}$\pm$2.6 & 93.89$\pm$3.0 & \underline{95.57}$\pm$2.1 & \underline{95.50}$\pm$2.2  \\ 
seismic      & \underline{34.12}$\pm$3.0 & \textbf{35.29}$\pm$4.1 & 17.18$\pm$1.9& 29.06$\pm$2.9 & 28.12$\pm$1.4 & 29.56$\pm$1.7 & 30.07$\pm$1.5 & 24.89$\pm$6.2  & 29.9$\pm$0.3  & 29.36$\pm$1.5 & {30.22}$\pm$2.0 & 16.12$\pm$1.8 & {31.18}$\pm$1.3 & {34.03}$\pm$1.7 \\ 
shuttle      & 99.30$\pm$1.1 & 98.15$\pm$3.4 & 98.6$\pm$2.6& \underline{99.38}$\pm$0.8 & 92.72$\pm$13.2 & 91.74$\pm$16.2           & 87.8$\pm$14.3 & 96.67$\pm$11.0 & 88.32$\pm$14.4 & 93.28$\pm$12.0 & 99.24$\pm$1.1 & 98.98$\pm$1.3 & \textbf{99.41}$\pm$0.9 & \textbf{99.41}$\pm$0.9 \\ 
speech       & 7.54$\pm$2.5 & \underline{9.84}$\pm$3.8 & 4.92$\pm$4.1 & 5.41$\pm$2.2 & 4.64$\pm$0.6 & 3.28$\pm$2.3 & 3.28$\pm$0.0 & 4.31$\pm$2.3 & 4.92$\pm$0.0 & 4.37$\pm$2.3 & 5.76$\pm$0.9 & {6.84}$\pm$1.0 & \textbf{13.35}$\pm$1.8 & \textbf{11.94}$\pm$1.6 \\
thyroid      & 55.52$\pm$3.9 & \textbf{69.20}$\pm$9.9 & 27.6$\pm$7.2 & {42.48}$\pm$2.6 & 42.64$\pm$0.8 & 37.20$\pm$1.2           & 24.40$\pm$0.0 & 18.31$\pm$4.8 & 30.80$\pm$0.0 & 42.53$\pm$3.2 & 49.20$\pm$0.0 & 30.80$\pm$0.0 & \underline{65.04}$\pm$1.0 & 64.48$\pm$1.3 \\ 
vertebral    & \underline{90.19}$\pm$1.1 & 81.27$\pm$6.0 & 81.14$\pm$0.8 & 86.33$\pm$1.5 & {89.97}$\pm$0.9 & 89.62$\pm$2.3 & 87.24$\pm$0.8 & 85.17$\pm$1.5 & 82.63$\pm$0.9 & 89.38$\pm$1.1 & 88.93$\pm$0.7 & 89.51$\pm$1.1 & \textbf{92.04}$\pm$0.7 & \textbf{91.50}$\pm$0.7  \\ 
vowels       & 40.4$\pm$8.5 & 48.67$\pm$34.5 & 62.91$\pm$8.0& \underline{86.20}$\pm$2.6 & 19.94$\pm$0.8 & 9.60$\pm$7.4 & 4.44$\pm$0.8 & 33.54$\pm$8.8 & 21.20$\pm$1.0 & 28.33$\pm$4.1 & 64.86$\pm$4.2 & 57.49$\pm$3.5 & \textbf{87.71}$\pm$3.5 & \textbf{86.57}$\pm$1.9 \\ 
wbc          & \underline{89.25}$\pm$1.7 & 66.67$\pm$12.1 & 57.26$\pm$8.7& 86.89$\pm$1.4 & {88.97}$\pm$0.9 & 87.45$\pm$3.6 & 79.53$\pm$1.0 & 83.61$\pm$4.0 & 62.74$\pm$1.1 & 89.9$\pm$1.1 & 88.54$\pm$1.0 & 88.96$\pm$0.9 & \textbf{92.65}$\pm$1.3 & 92.05$\pm$1.6 \\
wine         & 95.44$\pm$3.6 & 87.88$\pm$5.3 & 80.16$\pm$4.4 & 90.77$\pm$4.4 & 96.33$\pm$4.0 & 95.54$\pm$4.7 & 80.69$\pm$2.8 & 95.25$\pm$3.8 & 80.83$\pm$4.1 & 95.77$\pm$2.5 & \underline{97.10}$\pm$2.6 & 96.85$\pm$2.4 & \textbf{97.98}$\pm$2.0 & \textbf{97.71}$\pm$2.3 \\ \hline
mean         & 92.67$\pm$15.1 &89.00$\pm$17.4 & 87.95$\pm$20.5 & 92.8$\pm$15.7 & 91.37$\pm$18.3 & 88.58$\pm$21.4 & 87.83$\pm$19.4 & 90.39$\pm$18.5 & 87.87$\pm$18.3 & 91.85$\pm$17.3 & 93.43$\pm$14.8 & 92.6$\pm$16.1 & \textbf{94.23}$\pm$14.0 & \textbf{94.18}$\pm$13.8 \\ 
mean rank    &4.36 & 8.88 & 10.16 & 6.12 & 6.20 & 8.32 & 10.12 & 9.08 & 10.32 & 5.76 & 3.92 & 5.56 & \textbf{1.52} & \textbf{1.56} \\ \hline
p-value &0.0292& 0.0001& 0.0000& 0.0006 & 0.0190 & 0.0068 & 0.0008 & 0.0007 & 0.0002 & 0.0106 & 0.0023 & 0.0024 & - & 0.22 \\
p-value &0.0363& 0.0000& 0.0000&  0.0006 & 0.0198 & 0.0069 & 0.0008 & 0.0008 & 0.0002 & 0.0107 & 0.0030 & 0.0037 & 0.22 & -\\ \hline
\bottomrule
\end{tabular}}
\caption{F1 (\%) score of the proposed method compared with 12 baselines on 25 benchmark datasets. Each experiment is repeated 10 times with seed from 0 to 9, and mean value and standard deviation are reported. Mean is the average F1 score under all experiments. Mean rank is calculated out of 10 tested methods. Mammo. refers to mammography.}
\label{tab: main result f1}
\end{table*}

\section{Appendix H: Ablation Experiment Results}
\subsection{Sensitivity of Different Noise Levels} \label{app: lv results}
To explore how different noise levels contribute to the performance, we keep the noise type as Gaussian noise. In each training batch, we also generate $3$ noised instances which is consistent with the former setting. We search through different noise levels in $[0.1, 0.2, 0.5, 0.8, 1.0, 2.0, 3.0, 5.0]$. We report the average AUC, F1 score, and rank (out of $8$ different noise levels) in Figure \ref{fig: ablation noise level}. We observe that in all results if the noise level is too small, the average performance is low. This indicates noise with a small level would confuse the model because the distance between the normal samples and anomalies is very small. On the other hand, excessively high noise levels result in an expanded output value range, thereby enlarging the sampling space. This enlargement impacts the expressiveness of the model, reducing its effectiveness. 
However, when the noise level grows too large the score is also decreased. This means a large noise level has a large learning target value and a large sampling space, which may result in under-fitting training. The best choice of noise level appears around $1.0$. From the perspective of denoising score matching\citep{song2019generative}, a higher noise level will lead to inaccurate score estimation while a lower noise level will make it hard to learn data distribution in a relatively low-density region. Hence, a diverse noise level selection is preferred.

\begin{figure*}[t]
\begin{subfigure}{.249\textwidth}
  \centering
  \includegraphics[width=\linewidth]{ 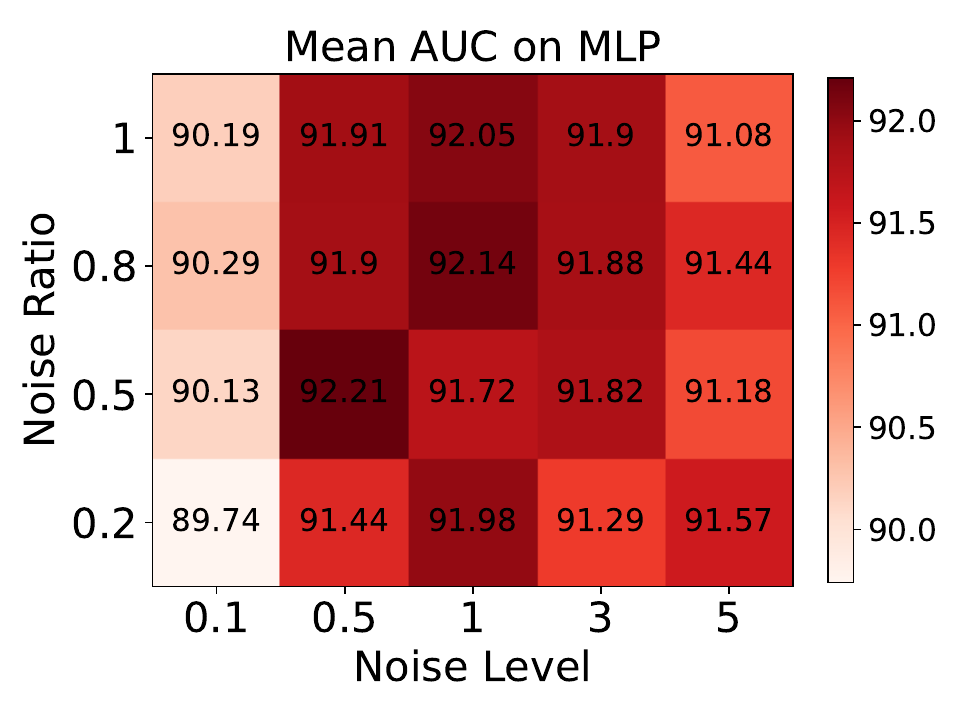}  
  \caption{AUC on MLP with Different Noise Levels and Ratios}
\end{subfigure}
\begin{subfigure}{.249\textwidth}
  \centering
  \includegraphics[width=\linewidth]{ 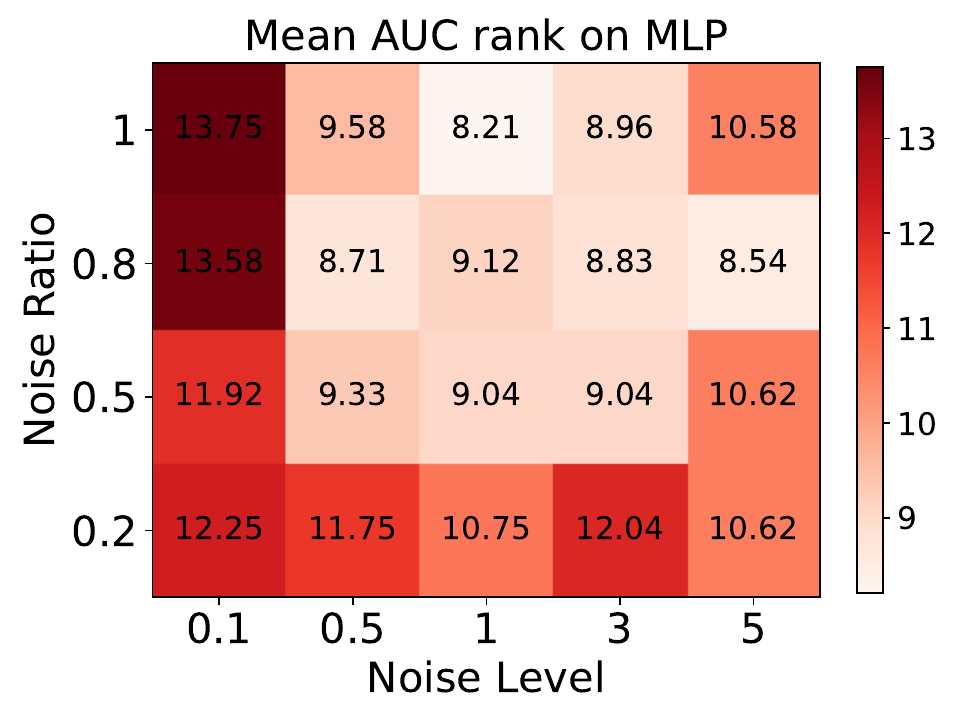}  
  \caption{Rank by AUC on MLP with Different Noise Levels and Ratios}
\end{subfigure}
\begin{subfigure}{.243\textwidth}
  \centering
  \includegraphics[width=\linewidth]{ 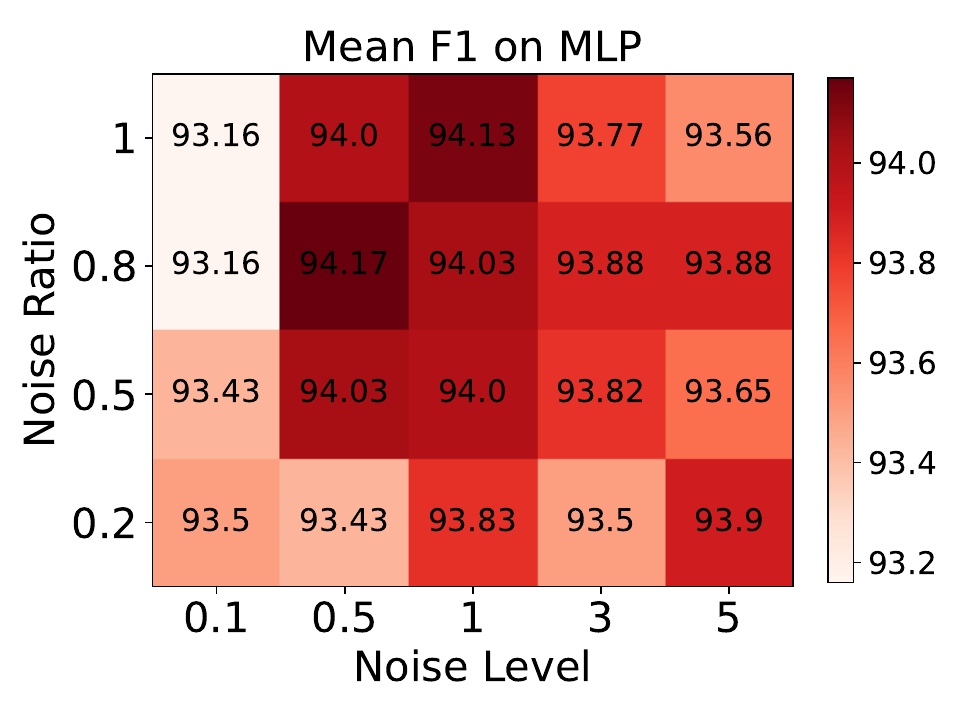}  
  \caption{F1 on MLP with Different Noise Levels and Ratios}
\end{subfigure}
\begin{subfigure}{.243\textwidth}
  \centering
  \includegraphics[width=\linewidth]{ 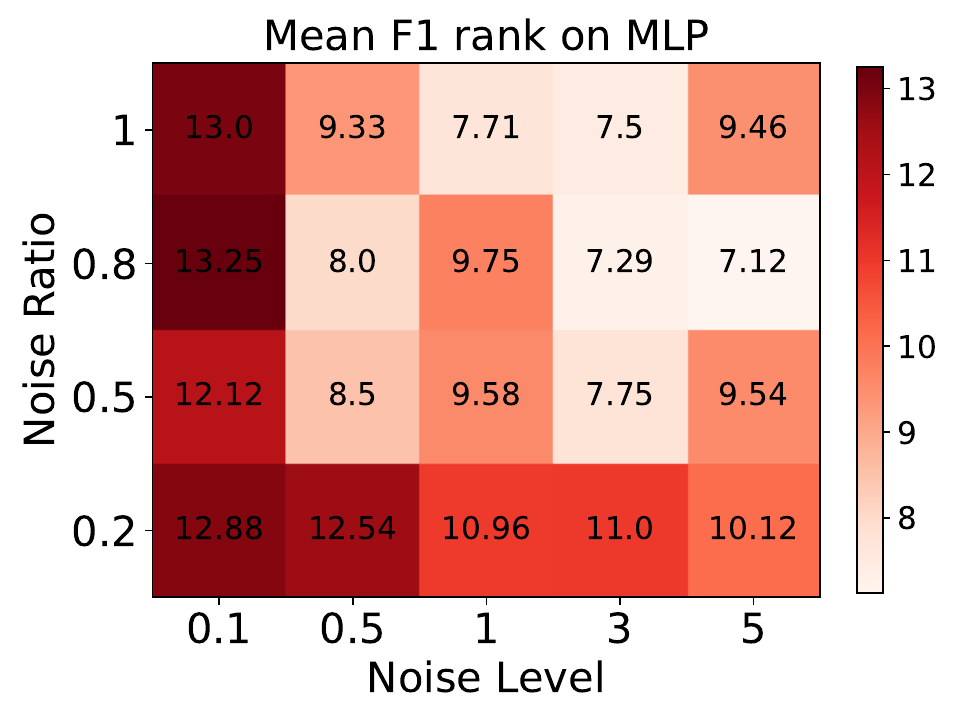}  
  \caption{Rank by F1 on MLP with Different Noise Levels and Ratios}
\end{subfigure}
\caption{Sensitivity of Different Noise Levels and Ratio on MLP. We test different noise levels in $[0.1, 0.5, 1.0, 3.0, 5.0]$, and noise ratios in $[0.2, 0.8, 0.5, 1.0]$. The mean rank (the lower the better) is calculated out of $20$ different settings.}
\label{fig: ablation noise ratio}
\end{figure*}

\begin{figure*}[t]
\begin{subfigure}{.249\textwidth}
  \centering
  \includegraphics[width=\linewidth]{ 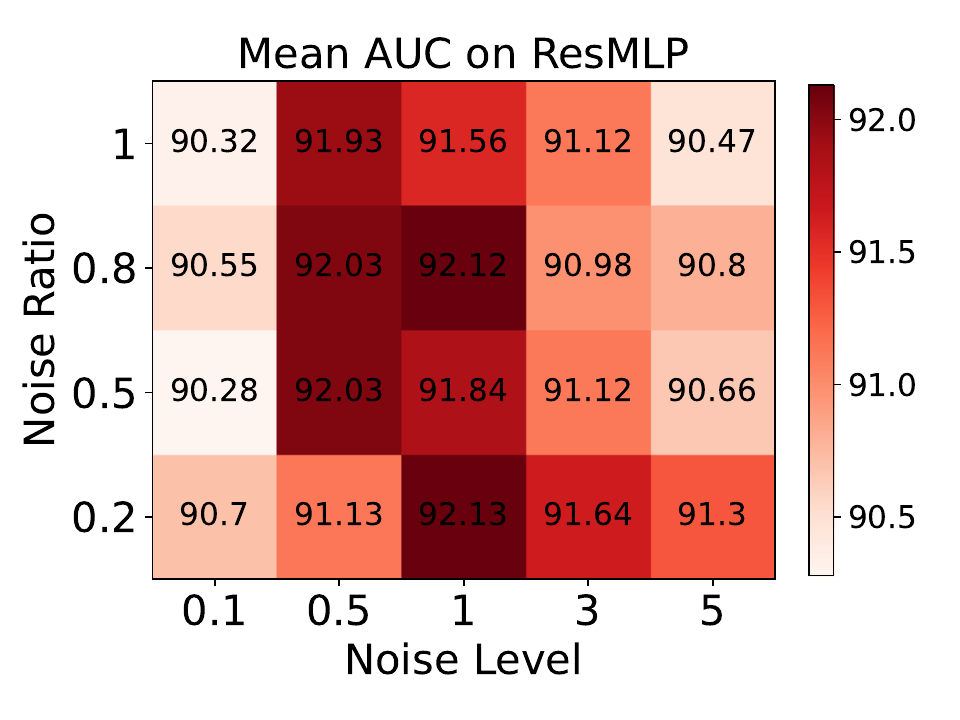}  
  \caption{AUC on ResMLP with Different Noise Levels and Ratios}
\end{subfigure}
\begin{subfigure}{.249\textwidth}
  \centering
  \includegraphics[width=\linewidth]{ 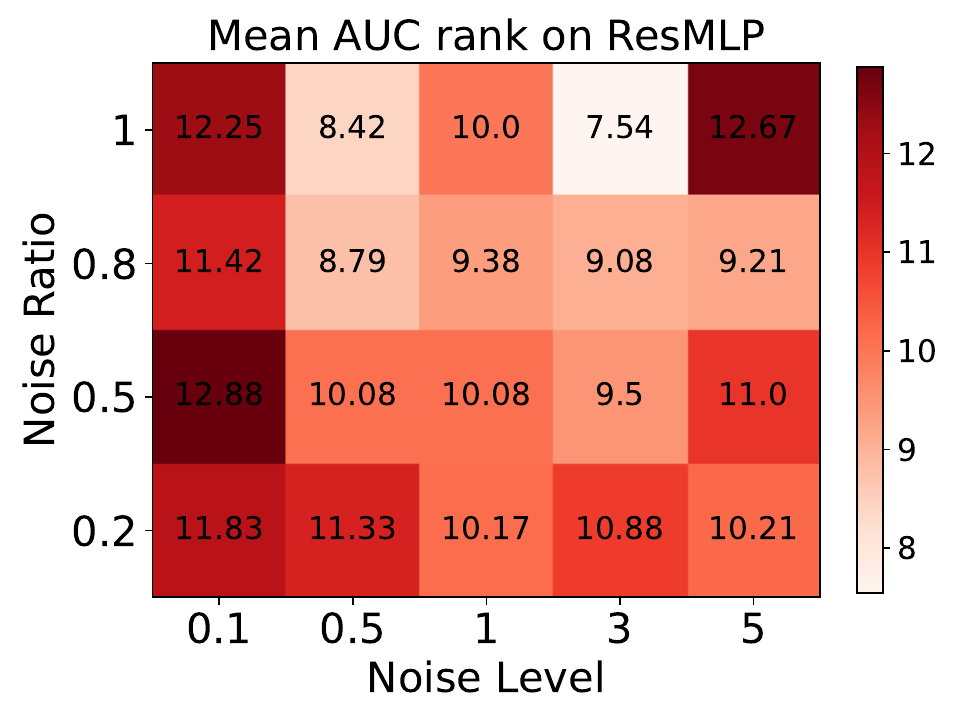}  
  \caption{Rank by AUC on ResMLP with Different Noise Levels and Ratios}
\end{subfigure}
\begin{subfigure}{.243\textwidth}
  \centering
  \includegraphics[width=\linewidth]{ 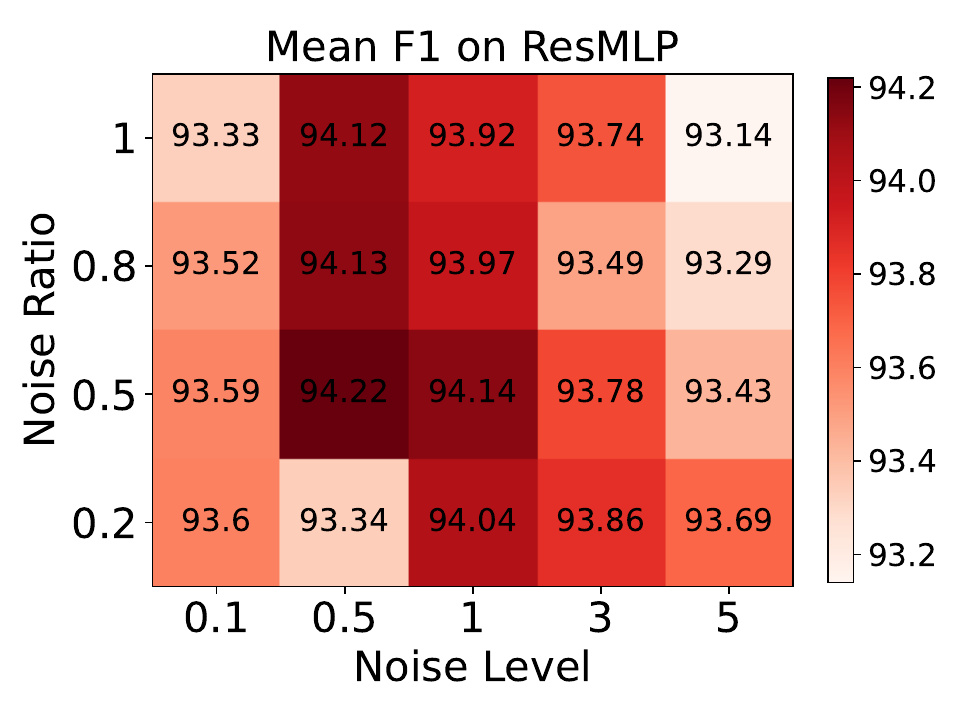}  
  \caption{F1 on ResMLP with Different Noise Levels and Ratios}
\end{subfigure}
\begin{subfigure}{.243\textwidth}
  \centering
  \includegraphics[width=\linewidth]{ 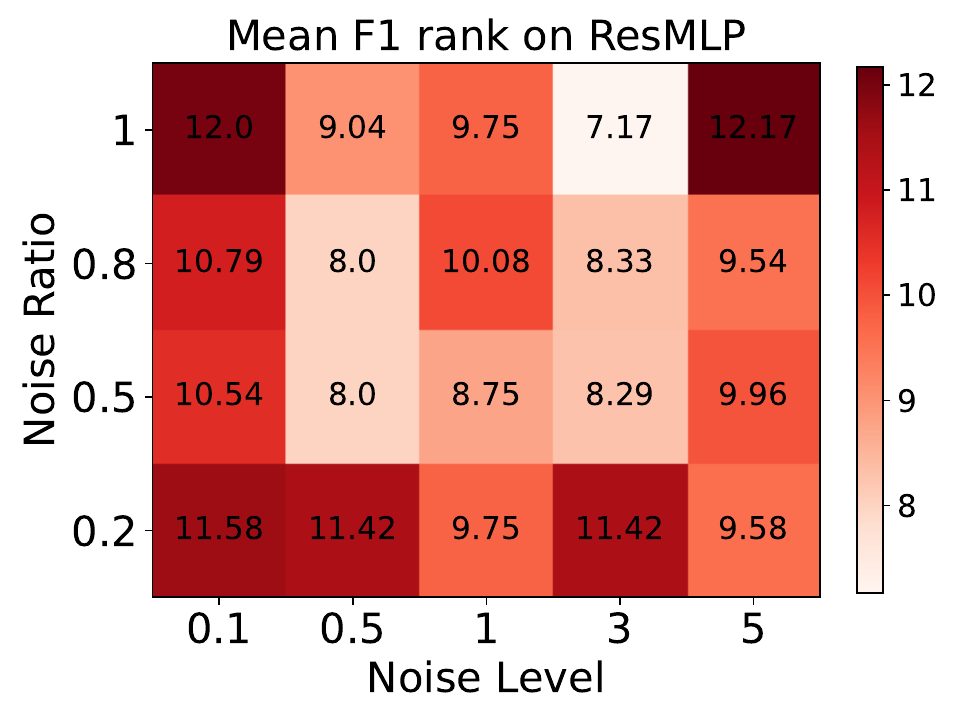}  
  \caption{Rank by F1 on ResMLP with Different Noise Levels and Ratios}
\end{subfigure}
\caption{Sensitivity of Different Noise Levels and Ratio on ResMLP. We test different noise levels in $[0.1, 0.5, 1.0, 3.0, 5.0]$, and noise ratios in $[0.2, 0.8, 0.5, 1.0]$. The mean rank (the lower the better) is calculated out of $20$ different settings.}
\label{fig: ablation noise ratio resmlp}
\end{figure*}

% \section{Sensitivity of Different Noise Levels and Ratio} \label{app: lv_p results}
% Due to limited pages, we put the results of the sensitivity of different noise levels and ratios on MLP and ResMLP here in the appendix. The results are shown in Figure and \ref{fig: ablation noise ratio} and \ref{fig: ablation noise ratio resmlp}. In ResMLP result, similar results are observed that a higher noise ratio is beneficial to the model performance. Noise levels either too high or too low can lead to a bad performance.

\subsection{\textbf{Sensitivity of Different Noise Ratios}} \label{app: lv_p results}
The noise ratio means the percentage of noise appearing the feature in a sample. For different noise ratios, we keep utilizing Gaussian noise and the same noise level generation as the default. In each training batch, we generate $1$ noised instances with the specific noise ratio instead. We vary the noise ratio in $[0.2, 0.5, 0.8, 1.0]$. We plot a heatmap jointly with different noise levels in Figure and \ref{fig: ablation noise ratio} and \ref{fig: ablation noise ratio resmlp}. We observe an augmentation in model accuracy concomitant with an increase in the noise ratio across the spectrum of noise levels tested. This trend suggests that our model exhibits an affinity for a higher noise ratio in its training regimen, implying that a moderated presence of noise may be beneficial to the model's training efficacy and predictive accuracy. In ResMLP result, similar results are observed that a higher noise ratio is beneficial to the model performance. Noise levels either too high or too low can lead to a bad performance.
% Notably, the model's performance attains its pinnacle when the noise level approximates $1.0$. Performance deterioration occurs at lower noise levels due to the difficulty in learning from minimal perturbation. On the other hand, excessively high noise levels result in an expanded output value range, thereby enlarging the sampling space. This enlargement impacts the expressiveness of the model, reducing its effectiveness.
% A low noise level degrades the performance because it is hard to learn if the perturbation is too small. If the noise level is too high, the range of the output value would become large. Then, the sampling space also becomes large. It results in a low expressiveness of the model.
% We see that the performance gets better as the noise ratio grows at all noise levels. The performance reaches its best when the noise level is close to $1.0$. This concludes our noise prediction model prefers a higher noise ratio during the training.

\begin{figure*}[t]
\begin{subfigure}{.249\textwidth}
  \centering
  \includegraphics[width=\linewidth]{ 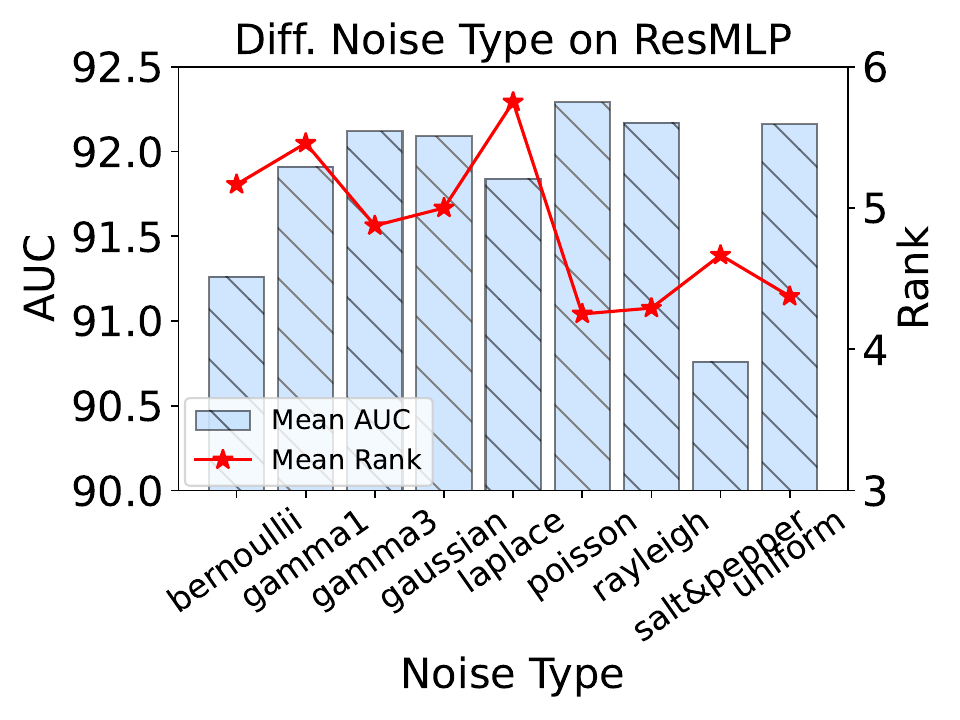}  
  \caption{AUC and Rank on ResMLP with Different Noise Type}
\end{subfigure}
\begin{subfigure}{.249\textwidth}
  \centering
  \includegraphics[width=\linewidth]{ 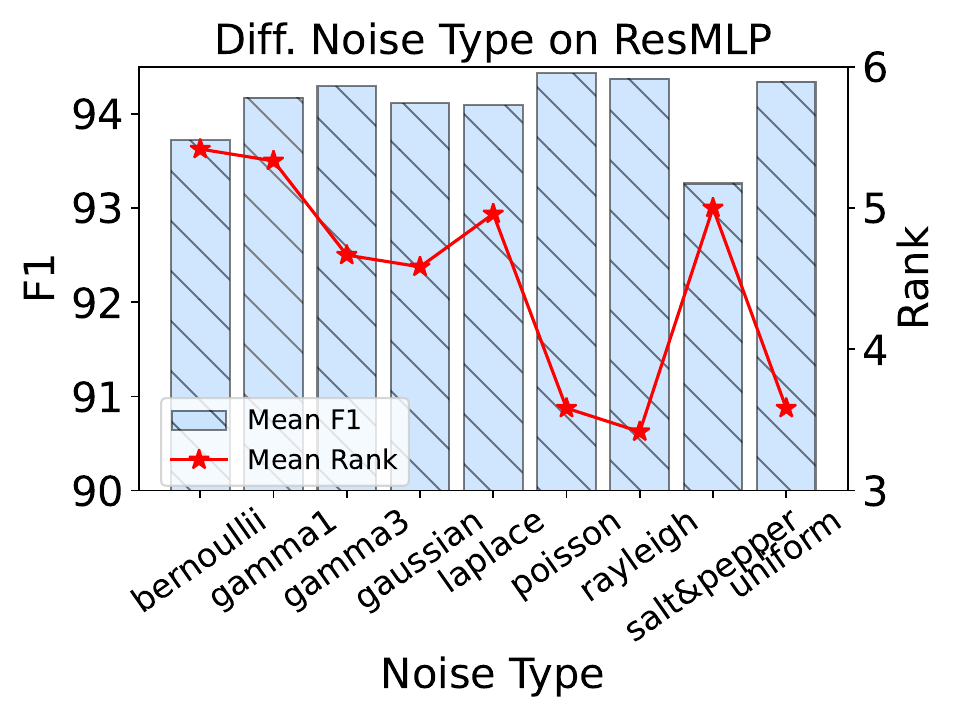}  
  \caption{F1 and Rank on ResMLP with Different Noise Type}
\end{subfigure}
\begin{subfigure}{.243\textwidth}
  \centering
  \includegraphics[width=\linewidth]{ 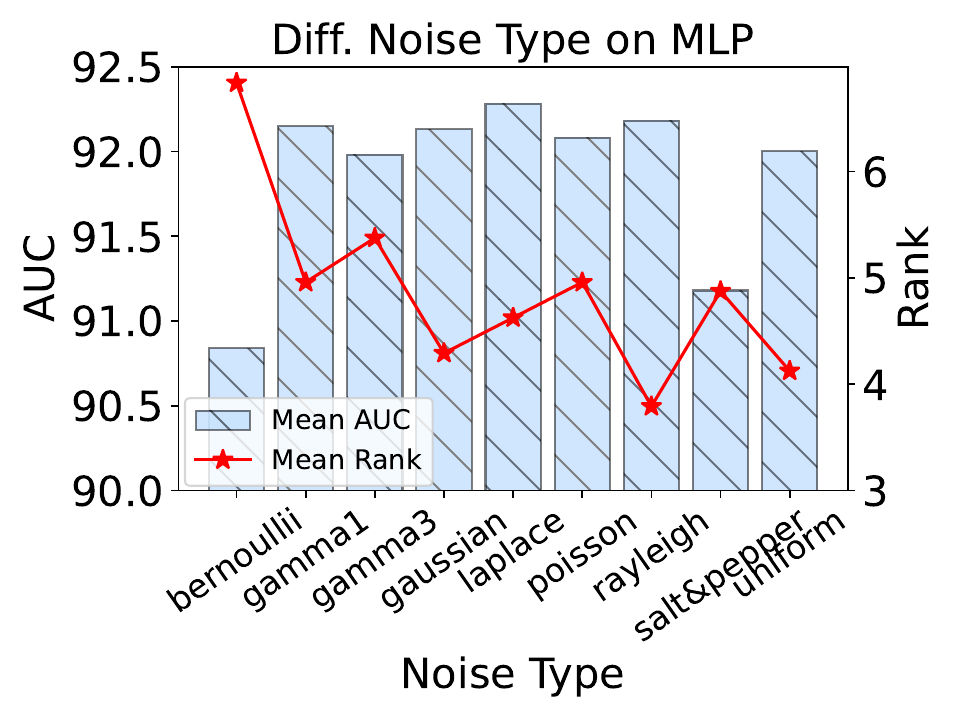}  
  \caption{AUC and Rank on MLP with Different Noise Type}
\end{subfigure}
\begin{subfigure}{.243\textwidth}
  \centering
  \includegraphics[width=\linewidth]{ 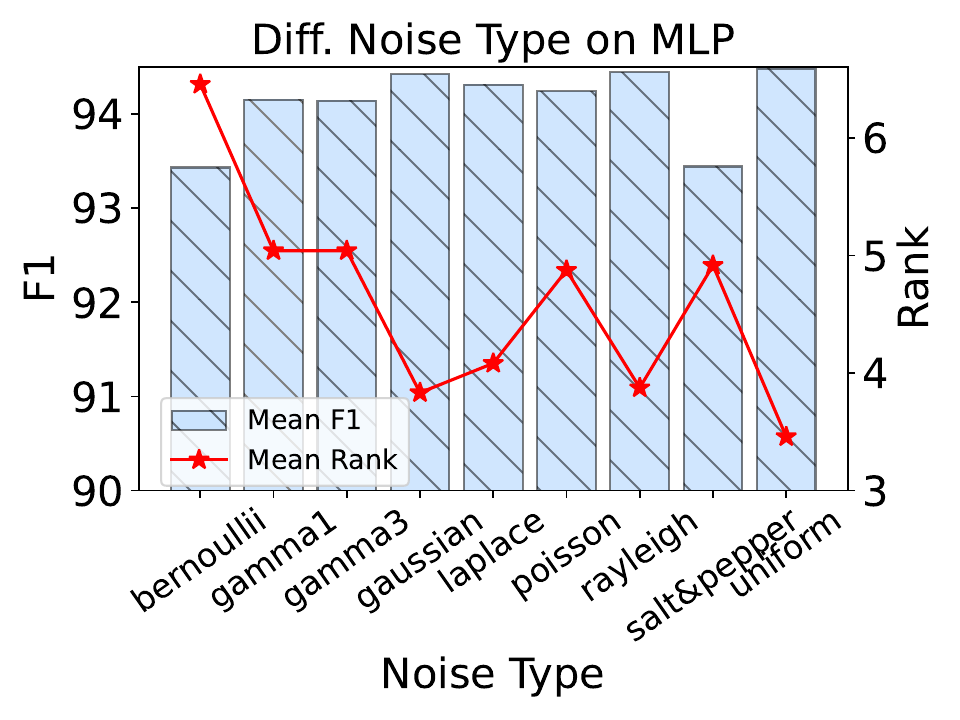}  
  \caption{F1 and Rank on MLP with Different Noise Type}
\end{subfigure}
\caption{Sensitivity of Different Noise Types. We test different noise types including Salt\&Pepper, Gaussian, Laplace, Uniform, Rayleigh, Gamma, Poisson, and Bernoulli distributions. The mean rank (the lower the better) is calculated out of $9$ different noise types. The results show that Gaussian, Rayleigh, and Uniform distributions have relatively stable performance.}
\label{fig: ablation noise type}
\end{figure*}

\subsection{Sensitivity of Different Noise Type} \label{app: type results}
We study utilizing different noise types to generate the noise. For a fair comparison, we sample all noise with $0$ mean and $\sigma$ standard deviation (noise level). We adopt several common noise distributions, including Salt\&Pepper, Gaussian, Laplace, Uniform, Rayleigh, Gamma, Poisson, and Bernoulli. For incorporating noise of a Bernoulli nature, the initial step involves the generation of a probability vector, which is derived from a uniform distribution. This vector is then utilized to produce a corresponding binary vector. The binary vector dictates the positions at which the original feature values will be altered—should the noise be present, a reversal of the feature's sign is executed. For Salt\&Pepper, similar to Bernoulli, we replace the value with the maximum or minimum value randomly in a batch instead of flipping the sign. For the other distributions, we adjust the parameter to make the generated noise having zero mean and $\sigma$ standard deviation. The detailed generation parameter is shown in Table \ref{tab: noise type}. The result is illustrated in Figure \ref{fig: ablation noise type}. Note that gamma1 and gamma3 represent gamma distribution with hyper-parameter $b=1$, and $3$, respectively. In the results, we observe that Bernoulli and Salt\&Pepper noise has a low performance score. It indicates the effectiveness of involving different noise levels in our noise generation design. Gaussian, Rayleigh, and Uniform noise have relatively stable performance with $92\%$ AUC score and $94\%$ F1 score. It reveals that our method is robust with multiple types of noise.

 \begin{table}[H]
 % {r}{0.5\textwidth}
\centering
\begin{tabular}{l|c} 
% \hline
\toprule
\textbf{Methods} & \textbf{Time (s)} \\ 
\midrule
Ours & $5.32\pm0.014$   \\ 
PLAD & $6.22\pm0.382$\\
\bottomrule
\end{tabular}
\caption{Comparative study on the time cost for noise generation. Our method takes less time to generate the noised sample in an epoch (3800 batches). In addition, our method can be implemented in a pre-generation manner to further save time.}
\label{tab: noise generation time}
% \vspace{-25pt}
\end{table}

\begin{table}[H]
\setlength\tabcolsep{0.9pt}
    \centering
\begin{tabular}{l|cc} 
% \hline
\toprule
\textbf{Methods} & \textbf{Training Time (s)} & \textbf{Inference Time (s)}\\ 
\midrule
Ours & $37.19\pm1.33$  &  $6.41\pm0.27$\\ 
PLAD (2022) & $40.54\pm0.38$ & $4.33\pm0.08$ \\
NeuTraLAD (2021) & $39.72\pm1.14$ & $6.79\pm0.24$ \\
DPAD (2024) & $26.84\pm0.24$ & $287.6\pm25.1$ \\
SCAD (2022) & $37.17\pm0.61$ & $6.46\pm0.28$ \\
\bottomrule
\end{tabular}
\caption{Comparative study of the time cost for model training and inference on KDDCUP-99 dataset. All methods use the same 5-layer backbone MLP. Our method takes less time to train than PLAD and NeuTraLAD. DPAD takes the least time for training because it is a density estimation method, but it takes unacceptable time for inference. Note that our method can reach a faster training speed by utilizing the pre-generated noisy sample.}
\label{tab: train eval time}
% \vspace{-25pt}
\end{table}

\section{Appendix I: Time Cost Analysis}
\subsection{Time Cost for Noise Generation}
Suppose the batch size is $b$ and the data dimensionality is $d$, the noise generation time complexity is $\mathcal{O}(bd)$ according to Algorithm \ref{alg: noise generation}. In contrast, other methods involving perturbation \cite{cai2022perturbation, qiu2021neural} and adversarial sample \cite{goyal2020drocc} have time complexity with $\mathcal{O}(bdW)$, where $W$ is workload related to a neural network module. Hence, our method for processing the generated negative sample is more efficient. We report the average time cost for generating the noise sample per epoch on KDDCUP-99 dataset in Table \ref{tab: noise generation time}. The time cost for our noise generation is $5.69\pm0.014$ seconds while the nagative sample processing time in PLAD is $6.22\pm0.38$ seconds. In addition, our noised sample can be generated before the training started. Hence, the generation time can be negligible.

\subsection{Time Cost for Training and Inference}
We also compare the training and inference time with other recent deep learning based methods, PLAD, NeuTraLAD, DPAD, and SCAD, on KDDCUP-99 dataset. The results are reported in Table \ref{tab: train eval time}.  All methods use the same 5-layer backbone MLP. Our method takes less time to train than PLAD and NeuTraLAD. DPAD takes the least time for training because it is a density estimation method, but it takes unacceptable time for inference. Note that our method can reach a faster training speed by utilizing the pre-generated noisy sample. Hence, our method is more efficient at the training stage.

\begin{table}[H]
    \centering
    \resizebox{0.95\linewidth}{!}{
\begin{tabular}{l|cc} 
% \hline
\toprule
\textbf{Methods} & \textbf{Hyper-param.} & \textbf{\#} \\ 
\midrule
Ours & Noise Level, Ratio, and Type  &  3\\ 
\makecell[l]{PLAD \citep{cai2022perturbation}} & $\lambda$, Types of perturbator  & 2 \\
\makecell[l]{NeuTraLAD  \citep{qiu2021neural}} & $\tau, K, m$ & 3 \\
\makecell[l]{DPAD \\ \cite{fu2024dense}} & $m, \gamma, \lambda, k$ & 4 \\
\makecell[l]{SCAD \\ \citep{shenkar2022anomaly}} & $k, u, \tau, r$ & 4 \\ 
\makecell[l]{DROCC \citep{goyal2020drocc}} & $r, \lambda, \mu, \eta$ & 4 \\
\makecell[l]{DOHSC \citep{zhang2024deep}} & $k, k', \lambda, \mu, \nu$& 5\\
\bottomrule
\end{tabular}
}
\caption{The comparison of the number of hyper-parameters. Although PLAD has fewer hyper-parameters, the choice of different types of perturbator would lead to a completely different implementation and would result in more new hyper-parameters.}
\label{tab: hyper para}
% \vspace{-25pt}
\end{table}

\section{Appendix J: Number of Hyper-parameters }
We compare the number of hyper-parameters with other recent deep learning based UAD methods. The comparison is listed in Table \ref{tab: hyper para}. Note that all hyper-parameters of our method are solely related to noise generation, and no hyper-parameter is involved in the learning objective except for the network training design (which exists in all deep learning based methods). Although PLAD has fewer hyper-parameters, the choice of different types of perturbator would lead to a completely different implementation and would result in more new hyper-parameters. Hence, our method is easy to implement.

\section{Discussions}
\paragraph{\textbf{Interpretability}}
For those deep learning-based methods \citep{golan2018deep, hendrycks2019using, hendrycks2018deep} that only output an anomaly score, it is hard to explain the forwarding non-linear process of the neural network black box. In our scheme, we can explain a little in the output. Similar to \citep{liznerski2020explainable}, our method outputs a vector with the same size as the input data where each dimension tells how much the noise is. Hence, if one dimension has a higher value, we can say the abnormality corresponds to this feature. We leave the detailed interpretability analysis in future research.

\paragraph{\textbf{Extension to other data types }}
Here, we provide insights for adapting noise evaluation to other data types like sequential, textual, and graph data. Adapting our approach to \textbf{images}, \textbf{time series}, \textbf{text}, and \textbf{graphs} is non-trivial but realizable, as each data type has distinct characteristics.

\begin{itemize}
    \item \textbf{Image data} Images exhibit local smoothness or piecewise linear patterns across pixel values, so directly adding random noise (e.g., Gaussian) to pixel values is suboptimal. A better approach is to segment images into patches, sample contextual noise for each patch, and predict the average noise magnitude per segment.

    \item \textbf{Time series} Time series data, such as voice recordings, exhibit natural dependencies between adjacent samples, and anomalies may represent meaningful sequences. Pre-processing (e.g., ADbench\cite{han2022adbench} for speech) can extract frequency features, transforming the series into a tabular format compatible with our model.

    \item \textbf{Textual data} Defining "noise" in textual data is non-trivial. A practical approach leverages recent large language models to extract latent features from text tokens. Using the CLS token or aggregation methods can yield sentence- or document-level embeddings, allowing our noise evaluation model to operate in this latent space.

    \item \textbf{Graph data} Arbitrarily adding noise to the adjacency matrix is not feasible. One can add an extra virtual node\cite{cai2023connection} to capture the latent feature of the whole graph. Then, graph-level AD with our noise evaluation can be performed.
\end{itemize}
For all data types, instead of noise on raw data, we can apply diverse noise (e.g., Gaussian) to the latent features extracted by some (pre-trained) feature extractors (e.g. encoders) and then predict the noise amplitude. However, there is a risk that the feature extractor may fail to preserve the information of anomalies, which will lead to failure in detecting anomalies. A promising direction for further research involves training the feature extractor, a decoder, and the noise detector jointly, directly in the data space rather than the feature space.

%%%%%%%%%%%%%%%%%%%%%%%%%%%%%%%%%%%%%%%%%%%%%%%%%%%%%%%%%%%%

% One potential limitation of our scheme is the noised sampling space will increase significantly as the data dimension increases. When applied to image datasets, data dimensions exceed those of tabular data. Hence, noise sampling in images is challenging. Here, we list some strategies to overcome it. One solution is to reduce the dimensionality of images. Another alternative is segmenting the image into multiple patches. For each patch, we can sample noise and then predict the average noise magnitude within that patch. In future research, we aim to explore noise prediction-based anomaly detection across diverse datasets.

\paragraph{Effectiveness on Dataset with Multiple Dense Regions}
Figure 1 is just for illustration. Real-world normal data can come from several dense regions, which can be illustrated by a larger picture combining multiple patterns like Figure 1. Our method and theory are applicable to datasets spanning multiple clusters. For instance, $\tilde{\mathcal{D}}_H$ is the union of multiple hard-anomaly regions corresponding to different clusters.

We observe from real-world examples that our noise evaluation model remains effective when the dataset spans multiple clusters. 
For instance, in ADBench Letter and Vowels datasets, a t-SNE visualization in Figure \ref{fig: dataset visual} shows that Letter is composed of 5 clusters and Vowels has 3 dense regions. In Tables 5 and 6 of Appendix F, our method can obtain a high-performance score in terms of AUC (90.87\%/99.03\%) and F1 (67.79\%/86.67\%) on these datasets.

\begin{figure}[h]
\begin{subfigure}{.45\linewidth}
  \centering
  \includegraphics[width=\linewidth]{ 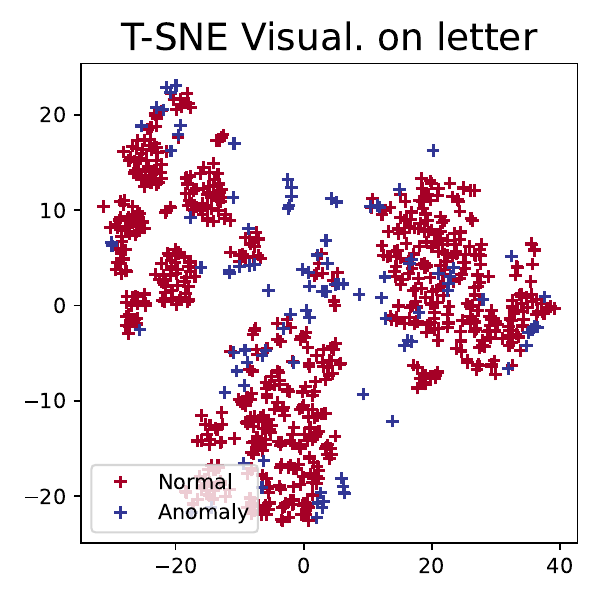}  
  % \caption{MLP Architecture}\label{fig: mlp arch}
\end{subfigure}
\quad\quad\quad
\begin{subfigure}{.45\linewidth}
  \centering
  \includegraphics[width=\linewidth]{ 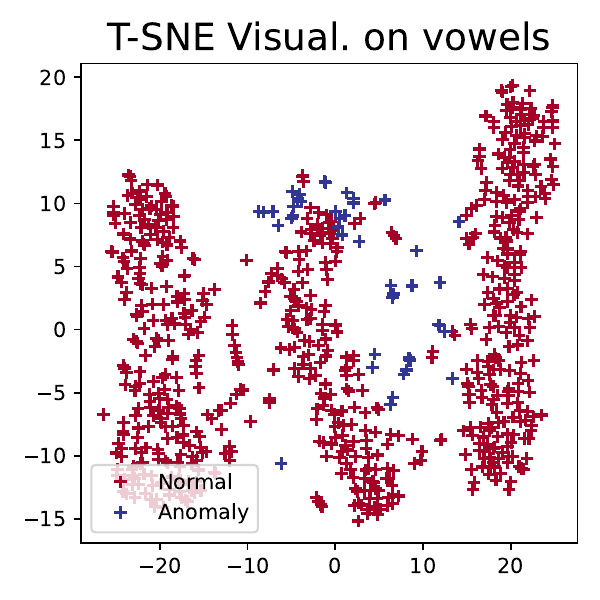}  
  % \caption{ResMLP Architecture}\label{fig: resmlp arch}
\end{subfigure}
\caption{T-SNE Visual. of Letter and Vowels}
\label{fig: dataset visual}
\end{figure}

\end{document}